\documentclass[11pt]{article} 

\pdfoutput=1

\usepackage{etoolbox}
\newtoggle{arxiv}
\toggletrue{arxiv}

\usepackage[utf8]{inputenc}
\usepackage{mathrsfs}
\usepackage{microtype}
\usepackage{booktabs} 
\usepackage{longtable,makecell}
\usepackage{amsmath, amsfonts, amssymb, graphicx}
\usepackage{mathtools,verbatim}
\usepackage{enumitem}
\usepackage{bm}
\usepackage{thm-restate}
\usepackage{natbib}
\usepackage{xspace,upgreek}

\usepackage{color-edits}

\addauthor{df}{ForestGreen}
\addauthor{ms}{orange}
\addauthor{kevin}{ForestGreen}

\usepackage[noend]{algpseudocode}
\usepackage{algorithm}

\iftoggle{arxiv}
{
	\usepackage[dvipsnames]{xcolor}
	\usepackage{amsthm}	
	\usepackage[scaled=.90]{helvet}
	
	\usepackage[vmargin=1.00in,hmargin=1.00in,centering,letterpaper]{geometry}
	\usepackage{fancyhdr, lastpage}

	\setlength{\headsep}{.10in}
	\setlength{\headheight}{15pt}
	\cfoot{}
	\lfoot{}
	\rfoot{\sc Page\ \thepage\ of\ \protect\pageref{LastPage}}

	\usepackage{subfigure}
}

\usepackage{hyperref}
\usepackage{cleveref}
\hypersetup{colorlinks,citecolor=blue}
\newtoggle{upperbound}
\togglefalse{upperbound}

\newcommand{\poly}{\mathrm{poly}}

\numberwithin{equation}{section}

\DeclarePairedDelimiter\ceil{\lceil}{\rceil}

\iftoggle{arxiv}{
 \theoremstyle{plain}
      \newtheorem{asm}{Assumption}
}{}

\creflabelformat{asm}{#2#1#3}
\newtheorem{thm}{Theorem}
\newtheorem{lem}{Lemma}[section]

\newtheorem{claim}[lem]{Claim}

\newtheorem{prop}[thm]{Proposition}

\iftoggle{arxiv}{\theoremstyle{definition}}{}
\newtheorem{defn}{Definition}[section]

\newtheorem{inst}{Instance Class}[section]

\newtheorem{rem}{Remark}[section]

\newtheorem{protocol}{Protocol}[section]
\iftoggle{arxiv}{
}{}

\renewcommand{\Pr}{\mathbb{P}}
\newcommand{\Exp}{\mathbb{E}}
\newcommand{\Var}{\mathrm{Var}}

\newcommand{\iidsim}{\overset{\mathrm{i.i.d.}}{\sim}}

\newcommand{\KL}{\mathrm{KL}}

\newcommand{\N}{\mathbb{N}}

\newcommand{\I}{\mathbb{I}}

\DeclareMathOperator*{\argmax}{arg\,max}

\renewcommand{\P}{\mathbf{P}}

\def\ddefloop#1{\ifx\ddefloop#1\else\ddef{#1}\expandafter\ddefloop\fi}
\def\ddef#1{\expandafter\def\csname bb#1\endcsname{\ensuremath{\mathbb{#1}}}}
\ddefloop ABCDEFGHIJKLMNOPQRSTUVWXYZ\ddefloop

\def\ddefloop#1{\ifx\ddefloop#1\else\ddef{#1}\expandafter\ddefloop\fi}
\def\ddef#1{\expandafter\def\csname fr#1\endcsname{\ensuremath{\mathfrak{#1}}}}
\ddefloop ABCDEFGHIJKLMNOPQRSTUVWXYZ\ddefloop
\def\ddefloop#1{\ifx\ddefloop#1\else\ddef{#1}\expandafter\ddefloop\fi}
\def\ddef#1{\expandafter\def\csname scr#1\endcsname{\ensuremath{\mathscr{#1}}}}
\ddefloop ABCDEFGHIJKLMNOPQRSTUVWXYZ\ddefloop
\def\ddefloop#1{\ifx\ddefloop#1\else\ddef{#1}\expandafter\ddefloop\fi}
\def\ddef#1{\expandafter\def\csname b#1\endcsname{\ensuremath{\mathbf{#1}}}}
\ddefloop ABCDEFGHIJKLMNOPQRSTUVWXYZ\ddefloop
\def\ddef#1{\expandafter\def\csname c#1\endcsname{\ensuremath{\mathcal{#1}}}}
\ddefloop ABCDEFGHIJKLMNOPQRSTUVWXYZ\ddefloop
\def\ddef#1{\expandafter\def\csname h#1\endcsname{\ensuremath{\widehat{#1}}}}
\ddefloop ABCDEFGHIJKLMNOPQRSTUVWXYZ\ddefloop
\def\ddef#1{\expandafter\def\csname t#1\endcsname{\ensuremath{\widetilde{#1}}}}
\ddefloop ABCDEFGHIJKLMNOPQRSTUVWXYZ\ddefloop

\def\ddefloop#1{\ifx\ddefloop#1\else\ddef{#1}\expandafter\ddefloop\fi}
\def\ddef#1{\expandafter\def\csname mat#1\endcsname{\ensuremath{\mathbf{#1}}}}
\ddefloop abcdefgijklmnopqrstuvwxyzABCDEFGHIJKLMNOPQRSTUVWXYZ\ddefloop

\newcommand{\gap}{\mathsf{gap}}

\renewcommand{\hbar}{\bar{h}}
\newcommand{\bern}{\mathrm{Bernoulli}}
\newcommand{\tsum}{{\textstyle \sum}}

\newcommand{\qcheckpith}[1][\pihat]{\mathring{Q}_h^{#1,t}}
\newcommand{\epscs}{\epsilon_{\texttt{cs}}}
\newcommand{\delcs}{\delta_{\texttt{cs}}}
\newcommand{\epsltoe}{\epsilon_{\texttt{L2E}}}
\newcommand{\algfillcomment}[1]{\hfill {\color{blue} \texttt{// #1}}}

\newcommand{\Vst}{V^\star}
\newcommand{\Vpi}{V^\pi}

\newcommand{\Qhat}{\widehat{Q}}
\newcommand{\Qst}{Q^\star}

\newcommand{\pist}{\pi^\star}
\newcommand{\pihat}{\widehat{\pi}}

\newcommand{\wpi}{w^\pi}

\newcommand{\Pst}{W}
\newcommand{\Qpi}{Q^\pi}

\newcommand{\simplex}{\bigtriangleup}

\newcommand{\Delmin}{\Delta_{\min}}

\newcommand{\Creg}{C_{\cR}}

\newcommand{\Xgood}{\cZ}
\newcommand{\Eexplore}{\cE_{\mathrm{exp}}}
\newcommand{\Eest}{\cE_{\mathrm{est}}}
\newcommand{\Esap}{\cE_{\textsc{\texttt{L2E}}}}
\newcommand{\CK}{C_K}

\newcommand{\wst}{w^\star}

\newcommand{\cOtil}{\widetilde{\cO}}

\newcommand{\Phat}{\widehat{W}}
\newcommand{\Phatst}{\widehat{W}}

\newcommand{\wtil}{\widetilde{w}}

\newcommand{\Deltil}{\widetilde{\Delta}}

\newcommand{\euler}{\textsc{Euler}\xspace}

\newcommand{\sap}{\textsc{\texttt{Learn2Explore}}}
\newcommand{\ies}{\texttt{FindExplorableSets}}
\newcommand{\collectdata}{\textsc{\texttt{EliminateActions}}}
\newcommand{\mcae}{\textsc{\texttt{Moca-SE}}}
\newcommand{\algname}{\textsc{Moca}\xspace}
\newcommand{\collectsamp}{\textsc{\texttt{CollectSamples}}}

\newcommand{\ahatst}{\widehat{a}^\star}

\newcommand{\Qpihat}{Q^{\pihat}}

\newcommand{\Qhatpihat}{\Qhat^{\pihat}}

\newcommand{\Vpihat}{V^{\pihat}}
\newcommand{\Delpihat}{\Delta^{\pihat}}

\newcommand{\frakA}{\cA}

\newcommand{\frakD}{\mathfrak{D}}

\newcommand{\Xbareps}{\mathrm{OPT}(\epsout)}

\renewcommand{\ast}{a^\star}

\newcommand{\delsamp}{\delta_{\mathrm{samp}}}
\newcommand{\ceuler}{c_{\mathrm{eu}}}
\newcommand{\pitil}{\widetilde{\pi}}

\newcommand{\false}{\texttt{false}}
\newcommand{\true}{\texttt{true}}

\newcommand{\Xlast}{\cX^{\numepochs+1}}

\newcommand{\nlast}{n^{\numepochs+1}}
\newcommand{\gamlast}{\gamma^{\numepochs+1}}

\newcommand{\Compbsolve}{\cC^{\star}}

\newcommand{\Compb}{\cC}
\newcommand{\Clow}{C_{\textsc{lot}}}
\newcommand{\epssolved}{\epsilon^{\star}}
\newcommand{\Delkl}{\Delta_{\mathrm{KL}}}

\newcommand{\finalround}{\texttt{FinalRound}}
\newcommand{\iotaepstil}{\iota_{\epstil}}

\newcommand{\numepochseps}{\ell_{\epsilon}}
\newcommand{\iotadeltil}{\iota_{\deltil}}

\newcommand{\epsexp}{\epsilon_{\mathrm{exp}}}
\newcommand{\jexp}{\iota_{\mathrm{exp}}}
\newcommand{\epsmoca}{\epsilon}
\newcommand{\delmoca}{\delta}
\newcommand{\iotaepsmoca}{\iota_{\epsmoca}}
\newcommand{\iotadelmoca}{\iota_{\delmoca}}
\newcommand{\numepochs}{\ell_{\epsmoca}}

\newcommand{\deltil}{\delta}
\newcommand{\epstil}{\epsilon}
\newcommand{\epsout}{\epsilon_{\mathrm{tol}}}
\newcommand{\epsoutm}{\epsilon_{\mathrm{tol}(m)}}
\newcommand{\epsoutmstop}{\epsilon_{\mathrm{tol}(\mstop)}}
\newcommand{\epsoutmstopb}{\epsilon_{\mathrm{tol}(\mstop-1)}}
\newcommand{\opt}{\mathrm{OPT}}
\newcommand{\delout}{\delta_{\mathrm{tol}}}
\newcommand{\deloutm}{\delta_{\mathrm{tol}(m)}}
\newcommand{\deloutmstop}{\delta_{\mathrm{tol}(\mstop)}}

\usepackage{wrapfig}

\newlength\tindent
\newcommand{\kevin}[1]{}
\allowdisplaybreaks

\makeatletter
\providecommand\theHALG@line{\thealgorithm.\arabic{ALG@line}}
\makeatother

\title{Beyond No Regret: Instance-Dependent PAC Reinforcement Learning}

\author{Andrew Wagenmaker\footnote{University of Washington, Seattle, WA. \href{mailto:ajwagen@cs.washington.edu}{ajwagen@cs.washington.edu}} \and Max Simchowitz\footnote{MIT, Cambridge, MA. \href{mailto:msimchow@mit.edu}{msimchow@mit.edu}}  \and Kevin Jamieson\footnote{University of Washington, Seattle, WA. \href{mailto:jamieson@cs.washington.edu}{jamieson@cs.washington.edu}} }
\date{June 20, 2022}

\begin{document}

\maketitle

\begin{abstract}

The theory of reinforcement learning has focused on two fundamental problems: achieving low regret, and identifying $\epsilon$-optimal policies. While a simple reduction allows one to apply a low-regret algorithm to obtain an $\epsilon$-optimal policy and achieve the worst-case optimal rate, it is unknown whether low-regret algorithms can obtain the instance-optimal rate for policy identification. We show this is not possible---there exists a fundamental tradeoff between achieving low regret and identifying an $\epsilon$-optimal policy at the instance-optimal rate.

Motivated by our negative finding, we propose a new measure of instance-dependent sample complexity for PAC tabular reinforcement learning which explicitly accounts for the attainable state visitation distributions in the underlying MDP. We then propose and analyze a novel, planning-based algorithm which attains this sample complexity---yielding a complexity which scales with the suboptimality gaps and the ``reachability'' of a state. We show our algorithm is nearly minimax optimal, and on several examples that our instance-dependent sample complexity offers significant improvements over worst-case bounds.

\end{abstract}

%
%
%
%


\section{Introduction}
Two of the most fundamental problems in Reinforcement Learning (RL) are regret minimization, and PAC (Probably Approximately Correct) policy identification. In the former setting, the goal of the agent is simply to play actions that collect sufficient reward in an online fashion, while in the latter, the goal of the agent is to explore their environment in order to identify an $\epsilon$-optimal policy with probability $1-\delta$. 

These objectives are intimately related:  for an agent to achieve low-regret they must play ``good'' policies, and therefore can solve the PAC problem as well. Indeed, in the worst case, optimal performance can be achieved by the ``online-to-batch'' reduction: running a worst-case optimal regret algorithm for $K$ episodes, and averaging its chosen policies (or choosing one at random) to make a recommendation. In this paper, we ask if online-to-batch is all there is to PAC learning. Focusing on the non-generative tabular setting, we ask
\begin{quote}
\emph{Does the online-to-batch reduction yield tight instance-dependent guarantees in non-generative, tabular PAC reinforcement learning? Or, are there other algorithmic principles and measures of sample complexity that emerge in the PAC setting but are absent when studying regret?}
\end{quote}
Mirroring recent developments in the regret setting which obtain instance-dependent regret guarantees, we approach this question from an instance-dependent perspective, and seek to develop instance-dependent PAC guarantees. 

Our focus on the non-generative setting brings to light the role of exploration in learning good policies. The majority of low-regret algorithms rely on playing actions they believe will lead to large reward (the principle of \emph{optimism}) and only explore enough to ensure they do not overcommit to suboptimal actions. While this is sufficient to balance the exploration-exploitation tradeoff and induce enough exploration to obtain low regret, as we will see, when the goal is simply \emph{exploration} and no concern is given for the online reward obtained, much more aggressive exploration can be used to efficiently traverse the MDP and learn a good policy. Hence, in addressing our question above, we aim to understand more broadly what are the most effective exploration strategies for traversing an unknown MDP when the goal is to learn a good policy.

\subsection{Our Contributions}

We demonstrate the importance of non-optimistic planning via three main contributions:
\begin{itemize}[leftmargin=*]
	\item \emph{New measure of instance-dependent complexity.} We propose a novel, fully instance-dependent measure of complexity for MDPs, the \emph{gap-visitation complexity}:
	\begin{align*}
	\Compb(\cM,\epsilon) & := \sum_{h=1}^H \inf_{\pi} \max_{s,a} \min \left \{ \frac{1}{ \wpi_h(s,a) \Deltil_h(s,a)^2},  \frac{\Pst_h(s)^2}{\wpi_h(s,a) \epsilon^2}  \right \} + \frac{H^2 | \opt(\epsilon)|}{\epsilon^2}
	\end{align*}
	where here $\wpi_h(s,a)$ is the probability of visiting $(s,a)$ at step $h$ under policy $\pi$, $\Deltil_h(s,a)$ is a measure of the suboptimality of choosing action $a$ at state $s$ and step $h$, $\Pst_h(s)$ is the \emph{maximum reachability} of state $s$ at step $h$, and $\opt(\epsilon)$ is the set of all ``near-optimal'' state-action tuples. We show that $\Compb(\cM,\epsilon)$ is no larger than the minimax optimal PAC rate, and that in some cases, $\Compb(\cM,\epsilon)$ is equivalent to the instance-optimal complexity. 
	\item \emph{A novel planning-based algorithm.} We propose and analyze a computationally efficient planning-based algorithm, \algname, which returns an $\epsilon$-optimal policy with probability at least $1-\delta$ after $\cOtil(\Compb(\cM,\epsilon) \cdot \log 1/\delta)$ episodes, for finite $\delta > 0$ and $\epsilon > 0$. 
	Rather than relying on optimism to guarantee exploration, it employs an aggressive exploration strategy which seeks to reach states of interest as quickly as possible, coupling this with a Monte Carlo estimator and action-elimination procedure to identify suboptimal actions. 
	\item \emph{Insufficiency of online-to-batch.} We show, through several explicit instances, that low-regret algorithms cannot achieve our proposed measure of complexity, and indeed can do arbitrarily worse. This shows that optimistic planning does not suffice to attain sharp instance-dependent PAC guarantees in tabular reinforcement learning. 
\end{itemize}
	
\begin{wrapfigure}{R}{0.4\textwidth}
\vspace{-1em}
  \begin{center}
    \includegraphics[width=0.37\textwidth]{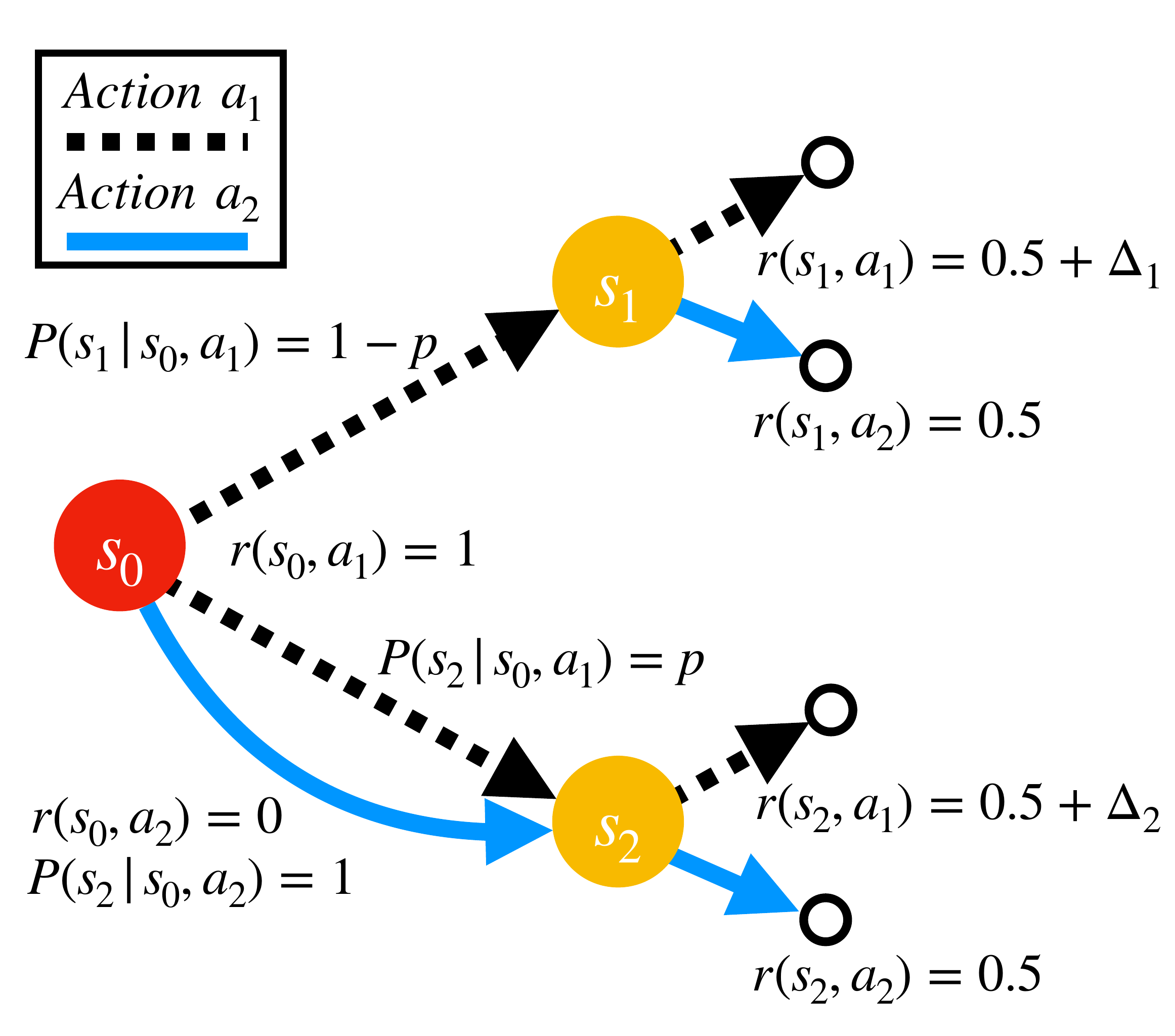}
  \end{center}
  \vspace{-15pt}
  \caption{A motivating example}
  \label{fig:mdp_ex1}
\end{wrapfigure}

\paragraph{A Motivating Example.} 
Consider the MDP in \Cref{fig:mdp_ex1}. In state $s_0$, action $a_1$ is optimal and transitions to state $s_1$ with probability $1-p$ and state $s_2$ with probability $p$. Action $a_2$ is suboptimal and transitions to state $s_2$ with probability 1. 
To learn a good policy, we need to identify the optimal action in both $s_1$ and $s_2$. An optimistic or low-regret algorithm will primarily play $a_1$ in $s_0$, as this action is optimal, and it will therefore only reach $s_2$ approximately $\cO(p K)$ times. It follows that a low-regret algorithm will take at least $\Omega(\frac{1}{p \Delta_2^2})$ episodes to learn the optimal action in $s_2$. In contrast, we could instead play $a_2$ in $s_0$, collecting less reward but learning the optimal action in $s_2$ in only $\Omega(\frac{1}{\Delta_2^2})$ episodes. For small $p$, this could be arbitrarily better. The following result makes this formal, illustrating that for identifying good policies in MDPs, existing low-regret and optimistic approaches can be highly suboptimal, and more intentional exploration procedures are needed. 
\vspace{-.25em}
\begin{prop}[Informal]\label{prop:two_state_ex} 
On the example in \Cref{fig:mdp_ex1}, any low-regret algorithm must run for at least $K \ge \Omega \big ( \frac{\log 1/\delta}{ \Delta_1^2} +  \frac{\log 1/\delta}{p \Delta_2^2}  \big )$ episodes to identify the optimal policy, while \algname will terminate and output the optimal policy after only $K \le \cO \big (  \frac{\log 1/\delta}{\Delta_1^2} + \frac{\log 1/\delta}{ \Delta_2^2} \big )$ episodes.
\end{prop}
\vspace{-.5em}

We stress that our goal in this work is \emph{not} to match the $\delta \to 0$ scaling of the optimal instance-dependent lower bound for $(\epsilon,\delta)$-PAC, but rather to obtain an instance-dependent complexity that captures the \emph{finite-time} difficulty of learning an $\epsilon$-optimal policy, and scales with an intuitive notion of MDP explorability, as in the example above. Even in the much simpler bandits setting, hitting the instance-optimal rate usually requires algorithms that ``track'' the optimal allocation, which can typically only be accomplished in the aforementioned $\delta \to 0$ limit, making such algorithms impractical in practice \citep{garivier2016optimal}. 
In contrast to this approach, we focus on the non-asymptotic regime, avoiding mixing-time and tracking arguments, and seeking to instead obtain ``practical'' instance-dependence.


\subsection{Organization} 
The remainder of this paper is organized as follows. First, in \Cref{sec:related} we review the related work on PAC RL. \Cref{sec:prelim} then introduces our notation and the basic problem setting we are working in. \Cref{sec:results} presents our new notion of complexity, the \emph{gap-visitation complexity}, and states our main results. In \Cref{sec:lowregret_breaks}, building on the example above, we introduce a particular class of MDP instances which shows that low-regret algorithms are provably suboptimal for PAC RL. \Cref{sec:alg_proof} provides an overview of our algorithm, \algname, and a proof sketch of our main theorem. Finally, we conclude in \Cref{sec:conclusion} with several interesting directions for future work. In the interest of space, detailed proofs of all our results are deferred to the appendix.

\section{Related Work}\label{sec:related}
The literature on PAC RL is vast and dates back at least two decades \citep{kearns2002near,kakade2003sample}. We cannot do it justice here so we aim to review just the most relevant works. In particular, as we focus on the tabular setting in this work, we omit discussion of similar works in reinforcement learning with function approximation (e.g. \cite{jin2020provably}).
In what follows, all claimed sample complexities hide constants and logarithmic factors. In addition, we state only the leading order $\epsilon^{-2}$ term---many works also have lower order $\poly(S,A,H) \epsilon^{-1}$ terms.

\paragraph{Minimax $(\epsilon,\delta)$-PAC Bounds.}
The vast majority of work has focused on minimax sample complexities that hold for \emph{any} MDP with arbitrary probability transition kernels and bounded rewards \citep{lattimore2012pac,dann2015sample,azar2017minimax,dann2017unifying}. 
The current state of the art in the stationary setting (i.e., $P_h(s'|s,a)=P(s'|s,a), \forall h \in [H]$) is \cite{dann2019policy}, which outputs an $\epsilon$-optimal policy with probability at least $1-\delta$ after at most $SA H^2 \epsilon^{-2}\log(1/\delta)$ episodes. This is known to be worst-case optimal \citep{dann2015sample}. In the non-stationary setting, \cite{menard2020fast} achieves a complexity of $SA H^3 \epsilon^{-2} \log(1/\delta)$.

\paragraph{PAC Bounds via Online-to-Batch Conversion.} 
As noted in the introduction, a PAC guarantee can be obtained from any low-regret algorithm using an online-to-batch conversion. For example, if an algorithm has a regret guarantee which, after $K$ episodes, scales as $\cO(\sqrt{C K})$, by randomly drawing a policy from the set of all policies played, via a simple application of Markov's inequality, one can guarantee that this policy will be $\frac{\sqrt{C}}{\delta \sqrt{K}}$-optimal with probability $1-\delta$. It follows that setting $K \ge \frac{C}{\epsilon^2 \delta^2}$ we are able to learn an $\epsilon$-optimal policy.\footnote{The reader will notice that the scaling in $\delta$ is suboptimal, scaling as $1/\delta^2$ instead of the familiar $\log 1/\delta$. To obtain a $\log 1/\delta$ scaling, instead of returning a \emph{single} policy, one could return a \emph{uniform distribution} over the policies returned by the regret-minimizing algorithm. At the start of each episode, a single policy would be drawn from this distribution and played for the duration of the episode. As a standard regret guarantee gives that $\sum_{k=1}^K (V_0^{\pi_k} - \Vst_0) \le \cO(\sqrt{K \log 1/\delta})$, choosing $K \ge \epsilon^{-2} \log 1/\delta$ implies the expected suboptimality of this distribution is no more than $\epsilon$ with probability $1-\delta$. However, the variance of this method is still large and so, since the standard PAC setting requires that a single policy be returned, we state subsequent online-to-batch results as scaling in $1/\delta^2$.} See \citep{jin2018q,menard2020fast} for a more in-depth discussion of this approach.

\paragraph{Gap-Dependent Regret Bounds for Episodic MDPs, and their implications for PAC.} Turning away from minimax-bounds to instant dependent analyses, optimistic planning algorithms have been shown to obtain gap-dependent regret bounds that, in many regimes, scale as $\log(K) \sum_{s,a,h} \frac{1}{\Delta_h(s,a)}$ \citep{simchowitz2019non,xu2021fine,dann2021beyond}, ignoring horizon and logarithmic factors. Here $\Delta_h(s,a)$ is the $Q$-value sub-optimality gaps under the optimal policy $\pist$ defined as $\Delta_h(s,a) := \max_{a' \in \cA} Q^{\pist}_h(s,a') - Q^{\pist}_h(s,a)$.
Using the online-to-batch conversion, we can obtain a PAC guarantee scaling as $\sum_{s,a,h} \frac{1}{\Delta_{h}(s,a) \cdot\epsilon} \cdot \frac{1}{\delta^2}$.\footnote{In the worst case, these bounds also incur a dependence on $S/\Delta_{\min}$, the inverse of the minimum nonzero gap $\Delta_{\min} := \min_{s,a,h}\{\Delta_h(s,a): \Delta_h(s,a) \ne 0\}$, scaled by the number of states $S$ \citep{simchowitz2019non}, or, with a more sophisticated algorithm, \cite{xu2021fine}, scaled by the number of states with non-unique optimal actions.} In a similar vein, \cite{ok2018exploration} propose an algorithm that has instance-optimal regret, though it is not computationally efficient and they only achieve the optimal rate in the asymptotic $T \rightarrow \infty$ regime.

\paragraph{Horizon-Free Instance Dependent Bounds.} 
A parallel line of work seeks regret bounds which replace dependence on the horizon $H$ with more refined quantities. The algorithm of \cite{zanette2019tighter}, \euler, yields regret of $\sqrt{SAK\min\{\mathbb{Q}_\star  H, \,\mathcal{G}^2 \}} $ (ignoring lower order terms), where $\mathbb{Q}_\star= \max_{s,a,h} \Var[R(s,a)] + \Var_{s^+ \sim P(\cdot| s,a)}[  V^{\star}(s') ]$ is a measure of reward and value variance, and $\mathcal{G}$ is a deterministic upper bound on the cumulative reward in an episode. Translated to the PAC setting, this implies a sample complexity (again suppressing lower order terms) of $\min\{\mathbb{Q}_\star H, \,\mathcal{G}^2\} SA \epsilon^{-2} $. In the special case where $\mathcal{G} = 1$, subsequent works sharpen polynomial dependence on $H$ in the lower-order term to polylogarithmic in both the PAC \citep{wang2020long}\footnote{This work suffers a worse $\epsilon^{-3}$ guarantee for PAC.} and regret \citep{zhang2020reinforcement} settings, thereby nearly eliminating the dependence on $H$ altogether. In this work, we do not focus on the horizon factor $H$, and hence these works, while compelling, are somewhat orthogonal.

\paragraph{Towards Instance-Dependent PAC Learning.}
To date, only several works have derived instance-dependent PAC bounds in the non-generative setting. The aforementioned instance-dependent regret guarantees can be seen as a first step in this direction, albeit with a suboptimal $1/\delta^2$ scaling.
\cite{jonsson2020planning} obtains a complexity that scales as the $Q$-value gap for the first time step, but that is exponential in $H$. 
Very recently, \cite{marjani2021navigating} studied the problem of best-policy identification, and proposed an algorithm which has an instance-dependent sample complexity. However, their results are purely asymptotic $(\delta \rightarrow 0)$, while we are concerned with the setting of finite $\epsilon > 0$ and $\delta > 0$. We discuss \citep{marjani2021navigating} in more detail in \Cref{sec:interpret_results}. In the special case of linear dynamical systems and smooth rewards, a setting which encompasses the Linear Quadratic Regulator problem, \cite{wagenmaker2021task} establish a finite-time, instance-dependent lower bound and matching upper bound for $\epsilon$-optimal policy identification. To our knowledge, this is the only work to obtain an instance-optimal $(\epsilon,\delta)$-PAC result, but their analysis does not apply to tabular MDPs.

\paragraph{Generative Model Setting.} 
In the generative model setting, the agent can query any $s$ and $a$ and observe the next state and reward. This setting is much simpler, entirely obviating the need for intentional exploration, and more favorable results are therefore obtainable.
A number of impactful analysis techniques have been developed for this setting with corresponding minimax bounds \citep{azar2013minimax,sidford2018near,agarwal2020model,li2020breaking}. Recently, several instance-dependent results have been shown in the generative model setting \citep{zanette2019generative,khamaru2020temporal,khamaru2021instance}. Most relevant is the work of \cite{zanette2019generative} which proposes the \textsc{Bespoke} algorithm and achieves a sample complexity of $\sum_{s,a} \frac{\log(1/\delta)}{\max\{ \epsilon^2, \Delta(s,a)^2 \}} $, ignoring horizon dependence.
A major shortcoming of this result is that their complexity will always scale at least as $S \epsilon^{-2}$, since for every state there exists an action $a$ such that $\Delta(s,a) = 0$. \cite{marjani2020best} study best policy identification in the $\delta \rightarrow 0$ regime. While they obtain an instance-dependent complexity, it is not clear they hit the instance-optimal rate.

\paragraph{Lower Bounds.}
We are unaware of any instance-dependent lower bound for $(\epsilon,\delta)$-PAC for MDPs.
Indeed, we are not even aware of an instance-dependent lower bound for $(\epsilon,\delta)$-PAC for contextual bandits ($H=1$). 
On the other hand, it is straightforward to obtain lower bounds for exact best policy optimization \citep{ok2018exploration,marjani2020best,marjani2021navigating}.
However, the best-policy identification case is frequently trivial because the sample complexity necessarily becomes vacuous as a state becomes harder and harder to access. Furthermore, the known lower bounds in this setting are relatively uninterpretable solutions to non-convex optimization problems.

\section{Preliminaries}\label{sec:prelim}
\paragraph{Notation.} 
We let $[N] = \{ 1,2,\ldots,N \}$. $\simplex(\mathcal{X})$ denotes the set of probability distributions over a set $\mathcal{X}$. 
$\Exp_{\pi} [ \cdot ]$ denotes the expectation over the trajectories induced by policy $\pi$ and $\Pr_{\pi} [ \cdot]$ denotes the probability measure induced by $\pi$. 
We let $\gtrsim$ refer to inequality up to absolute constants, and let $\cO(\cdot)$ hide absolute constants, and $\cOtil(\cdot)$ hide absolute constants as well as $\poly\log$ terms. 
In general, we use $\log$ to denote the base 2 logarithm.

\paragraph{Markov Decision Processes.}
We study finite-horizon, time inhomogeneous Markov Decision Processes (MDPs) given by the tuple $\cM = (\cS, \cA, H, \{ P_h \}_{h=1}^H, P_0, \{ R_h \}_{h=1}^H )$. Here $\cS$ is the set of states ($S := |\cS|$), $\cA$ the set of actions ($A := |\cA|$), $H$ the horizon, $P_h : \cS \times \cA \rightarrow \bigtriangleup(\cS)$ the transition kernel at step $h$, $P_0 \in \bigtriangleup(\cS)$ the initial state distribution, and $R_h : \cS \times \cA \rightarrow \bigtriangleup([0,1])$ the reward distribution, with $r_h(s,a) = \Exp[R_h(s,a)]$.
We assume that $\{ P_h \}_{h=1}^H, P_0,$ and $\{ R_h \}_{h=1}^H$ are all initially unknown to the learner.


An \emph{episode} is a trajectory $\{ (s_h,a_h,R_h) \}_{h=1}^H$ where $s_1 \sim P_0$, $s_{h+1} \sim P_h(\cdot | s_h,a_h)$, and $R_h \sim R_h(s_h,a_h)$. After $H$ steps, the MDP restarts and the process repeats. A \emph{policy} $\pi$ is a mapping from states to actions: $\pi : \cS \times [H] \rightarrow \simplex(\cA)$. 
$\pi_h(a|s)$ denotes the probability that $\pi$ chooses $a$ at $(s,h)$. If for all $(s,h)$, $\pi_h(a|s) = 1$ for some $a$, we say $\pi$ is a \emph{deterministic policy} and denote $\pi_h(s)$ the action it chooses at $(s,h)$. Otherwise we say $\pi$ is a \emph{stochastic policy}.

Given a policy $\pi$, the $Q$-value function, $\Qpi : \cS \times \cA \times [H] \rightarrow [0,H]$, denotes the expected reward obtained by playing action $a$ in state $s$ at time $h$, and then playing $\pi$ for all subsequent time. Formally, it is defined as
\begin{align*}
\Qpi_h(s,a) := \Exp_{\pi} \left [ \sum_{h' = h}^H R_{h'}(s_{h'},a_{h'}) | s_h = s, a_h = a \right ] .
\end{align*}
We also define the value function, $\Vpi : \cS \times [H] \rightarrow [0,H]$, as $\Vpi_h(s) := \Exp_{a \sim \pi_h(s)}[\Qpi_h(s,a)]$. The $Q$-function satisfies the Bellman equation:
\begin{align*}
\Qpi_h(s,a) = r_h(s,a) + \sum_{s'} P_h(s'|s,a) \Vpi_{h+1}(s').
\end{align*}
We let $\Vpi_{H+1}(s) = 0$ and $\Qpi_{H+1}(s,a) = 0$. We define the optimal $Q$-function as $\Qst_h(s,a) := \sup_\pi \Qpi_h(s,a)$,  $\Vst_h(s) := \sup_\pi \Vpi_h(s)$, and let $\pist$ denote an optimal policy. 
$\Vpi_0 := \sum_s P_0(s) \Vpi_1(s) $ denotes the \emph{value} of a policy, the expected reward it will obtain, and $\Vst_0 := \sup_\pi \Vpi_0$.

\paragraph{Optimal Actions and Effective Gap.}  
Critical to our analysis is the concept of a \emph{suboptimality gap}. In particular, we will define the suboptimality gap as:
\begin{align*}
\Delta_h(s,a) := \Vst_h(s) - \Qst_h(s,a).
\end{align*}
In words, $\Delta_h(s,a)$ denotes the suboptimality of taking action $a$ in $(s,h)$, and then playing the optimal policy henceforth. We also let $\Delta^\pi_h(s,a) := \max_{a'} \Qpi_h(s,a') - \Qpi_h(s,a)$. 

We say $a$ is optimal at $(s,h)$ if $\Delta_h(s,a) = 0$ (at least one such action is guaranteed to exist). We say $a$ is the unique optimal action at $(s,h)$ if $\Delta_h(s,a) = 0$, but $\Delta_h(s,a') > 0$ for all other $a \ne a'$, and say $a$ is a non-unique optimal action if there exists another $a'$ for which $\Delta_h(s,a) = \Delta_h(s,a') = 0$. We say $\mathcal{M}$ has unique optimal actions if, for all $(s,h)$, there is a unique  optimal action $a$. 

We now construct an effective gap $\Deltil_h(s,a)$ which coincides with $\Delta_h(s,a)$ for suboptimal actions, but is possibly non-zero if the optimal action is unique. Formally, at a particular $(s,h)$, we denote the minimum non-zero gap as
\begin{align*}
\Delmin(s,h) := \min_{a : \Delta_h(s,a) > 0} \Delta_h(s,a).
\end{align*}
The effective gap is then defined as follows:
\begin{align*}
\Deltil_h(s,a) := \begin{cases}  \Delta_h(s,a) & \Delta_h(s,a) > 0 \\
\Delmin(s,h) & a \text{ is the unique action at $s,h$ for which } \Delta_h(s,a) = 0 \\
0 & a \text{ is a non-unique action at $s,h$ for which } \Delta_h(s,a) = 0.
\end{cases}
\end{align*}

Finally, we introduce the idea of a state-action visitation distribution. We define
\begin{align*}
\wpi_h(s,a) := \Pr_{\pi} [ s_h = s, a_h = a], \quad \wpi_h(s) := \Pr_{\pi} [ s_h = s].
\end{align*}
Note that $\wpi_h(s,a) = \pi_h(a|s) \wpi_h(s)$. We denote the \emph{maximum reachability} of a state $s$ at time $h$ by:
\begin{align*}
W_h(s) := \sup_\pi \wpi_h(s).
\end{align*}
In words, $\Pst_h(s)$ is the maximum probability with which we could hope to reach $s$ at time $h$.

\paragraph{Special Cases: Bandits and Contextual Bandits.}
Two important special cases of the tabular MDP setting are the multi-armed bandit and contextual bandit problems. Both settings are of horizon $H = 1$. In the multi-armed bandit setting, there is a single state, and at every timestep the learner must choose an action (arm) and observes the reward for that action. The value of a (deterministic) policy is then measured simply by the expected reward obtained by the single action that policy takes. As the setting has only a single state and horizon of 1, we simplify notation and let $\Delta(a)$ denote the gap associated with action $a$.

The contextual bandit setting is a slight generalization of the multi-armed bandit where now we do allow for multiple states. In this setting, a state is sampled from $s \sim P_0$, the learner chooses an action to play, receives a reward, and the process repeats. As the state is sampled from $P_0$, the learner has no control over which state they visit. We do not assume any similarity between the different states---every state can be thought of as an independent multi-armed bandit.

Due to the simplicity of these settings---both are absent of any ``dynamics''---they therefore prove to be useful benchmarks on which to evaluate the optimality of our results.

\paragraph{PAC Reinforcement Learning Problem.}
In this work we study PAC RL. Formally, in PAC RL, the goal is to, with probability $1-\delta$, identify a policy $\pihat$ such that 
\begin{align}\label{eq:pac_def}
 \Vst_0 - \Vpihat_0 \le \epsilon 
\end{align}
using as few episodes as possible. We say that a policy satisfying \eqref{eq:pac_def} is \emph{$\epsilon$-optimal} and that an algorithm which returns a policy satisfying \eqref{eq:pac_def} with probability at least $1-\delta$ is $(\epsilon,\delta)$-PAC. Note that our goal is to find a single policy not a distribution over policies\footnote{That is, we want to find some policy $\pihat$ such that $\Vst_0 - V_0^{\pihat} \le \epsilon$, not a distribution over policies $\lambda \in \simplex(\Pi)$ such that $\Vst_0 - \sum_{\pi \in \Pi} \lambda_\pi V_0^\pi \le \epsilon$. Note that returning a single policy is the standard goal of PAC RL found in the literature.}.

\section{Instance-Dependent PAC Policy Identification}\label{sec:results}
Before stating our main result, we introduce our new notion of sample complexity for MDPs.

\begin{defn}[Gap-Visitation Complexity]\label{defn:gap_visitation_complexity}
For a given MDP $\cM$, we define the \emph{gap-visitation complexity} as:
\begin{align*}
\Compb(\cM,\epsilon) & := \sum_{h=1}^H \inf_{\pi} \max_{s,a} \min \left \{ \frac{1}{\wpi_h(s,a) \Deltil_h(s,a)^2},  \frac{\Pst_h(s)^2}{\wpi_h(s,a) \epsilon^2}  \right \} + \frac{H^2 | \opt(\epsilon) |}{\epsilon^2}.
\end{align*}
where the infimum is over all policies, both deterministic and stochastic, and:
\begin{align*}
\opt(\epsilon) & := \big \{ (s,a,h) \ : \ \epsilon \ge \Pst_h(s)   \Deltil_h(s,a)/3 \big \}.
\end{align*}
In the special case when $\cM$ has unique optimal actions, we define the \emph{best-policy gap-visitation complexity} as:
\begin{align*}
\Compbsolve(\cM) & := \sum_{h=1}^H \inf_{\pi} \max_{s,a } \frac{1}{\wpi_h(s,a) \Delta_h(s,a)^2}.
\end{align*}
\end{defn}
\noindent Since $\wpi_h(s,a) = \pi_h(a|s) \wpi_h(s)$, as long as $\wpi_h(s) > 0$ for some $\pi$, we can always choose our policy such that all actions are supported and $\wpi_h(s,a) > 0$ for all $a$\footnote{Here, we adopt the convention that, in the trivial case $\Pst_h(s) = 0$ (and thus $\wpi_h(s,a) = 0$), $\frac{\Pst_h(s)^2}{\wpi_h(s,a)\epsilon^2}$ evaluates to $0$.}. Recall that we have defined $\Deltil_h(s,a)$ so that $\Deltil_h(s,a) > 0$ for all $a$ as long as $(s,h)$ has a unique optimal action. This implies that as $\epsilon \rightarrow 0$, if $\cM$ has unique optimal actions, $|\opt(\epsilon)| \rightarrow 0$. Given this new notion of sample complexity, we are now ready to state our main result.

\begin{thm}\label{thm:complexity}
There exists an $(\epsilon,\delta$)-PAC algorithm, \algname, which, with probability at least $1-\delta$, terminates after running for at most
\begin{align*}
 \Compb(\cM,\epsilon)  \cdot H^2 c_\epsilon \log \tfrac{1}{\delta}  + \tfrac{\Clow(\epsilon)}{\epsilon} 
\end{align*}
episodes and returns an $\epsilon$-optimal policy, for lower-order term $\Clow(\epsilon) = \poly(S,A,H, \log \tfrac{1}{\epsilon}, \log \tfrac{1}{\delta})$ and $c_\epsilon = \poly\log(SAH/\epsilon)$. 
Furthermore, if $\epsilon < \epssolved := \min \{ \min_{s,a,h} \Pst_h(s) \Delta_h(s,a)/3,$ $ 2 H^2 S \min_{s,h} \Pst_h(s) \}$ and $\cM$ has unique optimal actions, \algname terminates after at most
\begin{align*}
 \Compbsolve(\cM)  \cdot H^2 c_{\epssolved} \log \tfrac{1}{\delta}  + \tfrac{\Clow(\epssolved)}{\epssolved} 
\end{align*}
episodes and returns $\pist$, the optimal policy, with probability $1-\delta$.
\end{thm}


%
%


\noindent In addition, \algname is computationally efficient with computational cost scaling polynomially in problem parameters. In \Cref{sec:alg_proof}, we provide a sketch of the proof of \Cref{thm:complexity} and state the definition \algname. The full proof is deferred to \Cref{sec:detailed_proof}.

\subsection{Interpreting the Complexity}\label{sec:interpret_results}
Intuitively, the first term in the gap-visitation complexity quantifies how quickly we can eliminate all actions at least $\epsilon/\Pst_h(s)$-suboptimal for all $s$ and $h$, given that we must explore in our particular MDP. For a given $s$ and $h$, if we play policy $\pi$ for $K$ episodes, we will reach $(s,h)$ on average $\wpi_h(s) K$ times. Thus, if we imagine that there is a bandit at $(s,h)$, to eliminate action $a$ will require that we run for at least $\frac{1}{\wpi_h(s,a) \Delta_h(s,a)^2}$ episodes. The following result makes this rigorous---up to $H$ factors, a complexity of $\cO(\Compbsolve(\cM) \cdot \log 1/\delta)$, which \algname achieves, cannot be improved on in general for best-policy identification.

\begin{prop}\label{prop:bpi_lb1}
Fix some $S > 1, A > 1, H > 1$, $\hbar \in [H]$, transition kernels $\{ P_h \}_{h = 1}^{\hbar - 1}$, and gaps $\{ \gap(s,a) \}_{s \in [S], a \in [A - 1]} \subseteq (0,1/2)^{SA}$. Then there exists some MDP $\cM$ with $S$ states, $A$ actions, horizon $H$, transition kernel $ P_h $ for $h \le \hbar - 1$, and gaps
\begin{align*}
\Delta_{\hbar}(s,a) = \gap(s,a), \quad \forall s \in \cS, a \in \cA, a \neq \pist_{\hbar}(s), \qquad \Delta_{h}(s,a) \ge 1, \quad \forall s \in \cS, a \in \cA, h \neq \hbar,
\end{align*}
such that any $(0,\delta)$-PAC algorithm with stopping time $K_\delta$ requires:
\begin{align*}
\Exp_{\cM}[K_\delta] \gtrsim  \inf_\pi \max_{s,a} \frac{1}{\wpi_{\hbar}(s,a) \Delta_{\hbar}(s,a)^2} \cdot \log \tfrac{1}{2.4\delta} .
\end{align*}
\end{prop}

In this instance, as $\Delta_h(s,a) \ge 1$ for $h \neq \hbar$, assuming $\{ P_h \}_{h=1}^{\hbar-1}$ is chosen such that $\Pst_h(s)$ is not too small for each $s$ and $h \le \hbar$, we will have that $\Compbsolve(\cM) = \cO(\inf_\pi \max_{s,a} \tfrac{1}{\wpi_{\hbar}(s,a) \Delta_{\hbar}(s,a)^2})$, so \Cref{prop:bpi_lb1} implies that we must have $\Exp_{\cM}[K_\delta] \ge \Omega(\Compbsolve(\cM) \cdot \log 1/\delta)$, matching the upper bound given in \Cref{thm:complexity} up to $H$ factors.

The second term in $\cC(\cM,\epsilon)$, $H^2 | \opt(\epsilon) |/\epsilon^2$, captures the complexity of ensuring that, after eliminating $\epsilon/\Pst_h(s)$-suboptimal actions, sufficient exploration is performed to guarantee the returned policy is $\epsilon$-optimal. While this will be no worse than $H^3 SA/\epsilon^2$, it could be much better, if in our MDP the number of $(s,a,h)$ with $\Deltil_h(s,a) \lesssim \epsilon/\Pst_h(s)$ is small (note that in the case when $\cM$ has unique optimal actions, since $\Deltil_h(s,a) \ge \Delmin(s,h)$ by definition for all $(s,a,h)$, $\opt(\epsilon)$ will only contain states for which the minimum \emph{non-zero} gap is less than $\epsilon/\Pst_h(s)$).
We next obtain the following bounds on $\cC(\cM,\epsilon)$, providing an interpretation of $\cC(\cM,\epsilon)$ in terms of the maximum reachability, and illustrating $\cC(\cM,\epsilon)$ is no larger than the minimax optimal complexity. This implies \algname is nearly worst-case optimal, matching the lower bound of $\Omega(\frac{SAH^2}{\epsilon^2} \cdot \log 1/\delta)$ from \cite{dann2015sample} up to $H$ and log factors\footnote{This lower bound is for the stationary setting. As noted in \cite{menard2020fast}, one would expect a lower bound of $\Omega(\frac{SAH^3}{\epsilon^2} \cdot \log 1/\delta)$ in the non-stationary setting, implying \algname is $H^2$ off the lower bound.}. 

\begin{prop}\label{cor:complexity2}
The following bounds hold: 
\begin{enumerate}
\item $\cC(\cM,\epsilon) \le \frac{H^3 SA}{\epsilon^2}$
\item $\cC(\cM,\epsilon) \le \tsum_{h=1}^H  \tsum_{s,a} \min  \{ \tfrac{1}{\Pst_h(s) \Deltil_h(s,a)^2}, \tfrac{\Pst_h(s)}{\epsilon^2}  \} + \tfrac{H^2 | \opt(\epsilon)|}{\epsilon^2}$
\item $\cC(\cM,\epsilon) \le \tsum_{h=1}^H  \tsum_{s,a} \frac{1}{\epsilon \max \{ \Deltil_h(s,a),\epsilon \}} + \tfrac{H^2 | \opt(\epsilon)|}{\epsilon^2}$.
\end{enumerate}
\end{prop}

\noindent In the special case of multi-armed and contextual bandits, the gap-visitation complexity simplifies considerably.

\begin{prop}\label{prop:complexity_bandit}
If $\cM$ is a multi-armed bandit, then
\begin{align*}
\Compb(\cM,\epsilon) = \sum_a \min \left \{ \frac{1}{\Deltil(a)^2}, \frac{1}{\epsilon^2} \right \}, \quad \Compbsolve(\cM) =  \sum_{a : \Delta(a) > 0} \frac{1}{\Delta(a)^2}.
\end{align*}
Furthermore, if $\cM$ is a contextual bandit, then
\begin{align*}
 \Compbsolve(\cM) = \max_s \frac{1}{\Pst(s)} \sum_{a } \frac{1}{\Delta(s,a)^2} .
\end{align*}
\end{prop}

\noindent The values given here are known to be the optimal problem-dependent constants for both best arm identification and $(\epsilon,\delta)$-PAC for multi-armed bandits \citep{kaufmann2016complexity,degenne2019pure}. 
To our knowledge, the lower bound for best-policy identification in contextual bandits has never been formally stated, yet it is obvious it will take the form of $\Compbsolve(\cM)$ given here. It follows that in the special cases of multi-armed bandits and contextual bandits, \algname is instance-optimal, up to logarithmic factors and lower-order terms.

Several additional interpretations of the gap-visitation complexity are given in \Cref{sec:non_unique_actions}. The above results show that the gap-visitation complexity cleanly interpolates between the worst-case optimal rate for $(\epsilon,\delta)$-PAC, and, in certain MDPs, the instance-optimal rate for best-policy identification. In between these extremes, it captures an intuitive sense of instance-dependence. As we will show in the following section, this instance-dependence can offer significant improvements over worst-case optimal approaches.

\begin{rem}[Comparison to \cite{marjani2021navigating}]
Our notion of best-policy gap-visitation complexity is closely related to the measure of complexity introduced in \cite{marjani2021navigating}, though they study the infinite-horizon, discounted case. 
Notably, however, their analysis only considers best-policy identification ($\epsilon = 0$) and is purely asymptotic ($\delta \rightarrow 0$), while ours holds for $\delta > 0$ and $\epsilon > 0$. Further, our best-policy gap-visitation complexity offers a non-trivial improvement over their complexity, scaling as $(\min_s \wpi_h(s,a) \Delmin(s,h)^2)^{-1}$ instead of $( \min_s \wpi(s,\pist(s)) \cdot \min_s \Delmin(s)^2)^{-1}$ which \cite{marjani2021navigating} obtains.
\end{rem}

\begin{rem}[Dependence on $\log 1/\delta$]\label{rem:log_delta}
While the leading term in the sample complexity of \algname only scales as $\log 1/\delta$, the lower order term scales as a suboptimal $\log^3 1/\delta$. These additional factors of $\log 1/\delta$ are due to the regret-minimization algorithm used in the exploration procedure we employ. We show in \Cref{rem:logterm} that it can be improved to $\log 1/\delta \cdot \log\log 1/\delta$ and leave completely removing the suboptimal $\delta$ scaling for future work.
\end{rem}

\begin{rem}[Improving $H$ Dependence]
As noted above, \algname attains a worst-case $H$ dependence that is a factor of $H^2$ worse than the lower bound.
Our analysis relies on Hoeffding's inequality to argue about the concentration of our estimate of $\Qpihat_h(s,a)$. Rather than depending on the variance of the next-state value function, our confidence interval therefore depends on $H^2$, an upper bound on the variance. If desired, we could instead employ an empirical Bernstein-style inequality \citep{maurer2009empirical}, which would allow us to replace this $H^2$ scaling with the variance of the reward obtained from playing $a$ at $(s,h)$ and then playing $\pihat$. We believe that this modification may allow us to refine the $H$ dependence of \algname.
As the focus of this work is obtaining an instance-dependent complexity, we leave the details of this for future work.
\end{rem}



\section{Low-Regret Algorithms are Suboptimal for PAC}\label{sec:lowregret_breaks}

Using our instance-dependence complexity, we next show that running a low-regret algorithm and applying an online-to-batch conversion can be very suboptimal for PAC RL. We first define a low-regret algorithm and our learning protocol:


\begin{defn}[Low-Regret Algorithm]\label{def:low_regret}
We say an algorithm $\cR$ is a \emph{low-regret algorithm} if it has expected regret bounded as $\mathrm{Regret}(K) = \sum_{k=1}^K \Exp_{\cR}[\Vst_0 - V_0^{\pi_k}] \le C_1 K^\alpha + C_2$,
for some constants $C_1,C_2$, $\alpha \in (0,1)$, and where $\pi_k$ is the policy $\cR$ plays at episode $k$. 
\end{defn}


\begin{protocol}[Low-Regret to PAC]\label{prot:lr_to_pac}
We consider the following procedure:
\begin{enumerate}
\item Learner runs low-regret algorithm $\cR$ satisfying \Cref{def:low_regret} for $K$ episodes, collects data $\frakD_{\cR}(K)$.
\item Using $\frakD_{\cR}(K)$ any way it wishes, the learner proposes a (possibly stochastic) policy $\pihat$.
\end{enumerate}
\end{protocol}

Note that the setting considered in \Cref{prop:two_state_ex} is precisely that considered here. We now present an additional instance class where any learner following \Cref{prot:lr_to_pac} with a low regret algorithm $\cR$ is provably suboptimal.

\begin{wrapfigure}[22]{L}{0.38\textwidth}
\vspace{-2em}
  \begin{center}
    \includegraphics[width=0.38\textwidth]{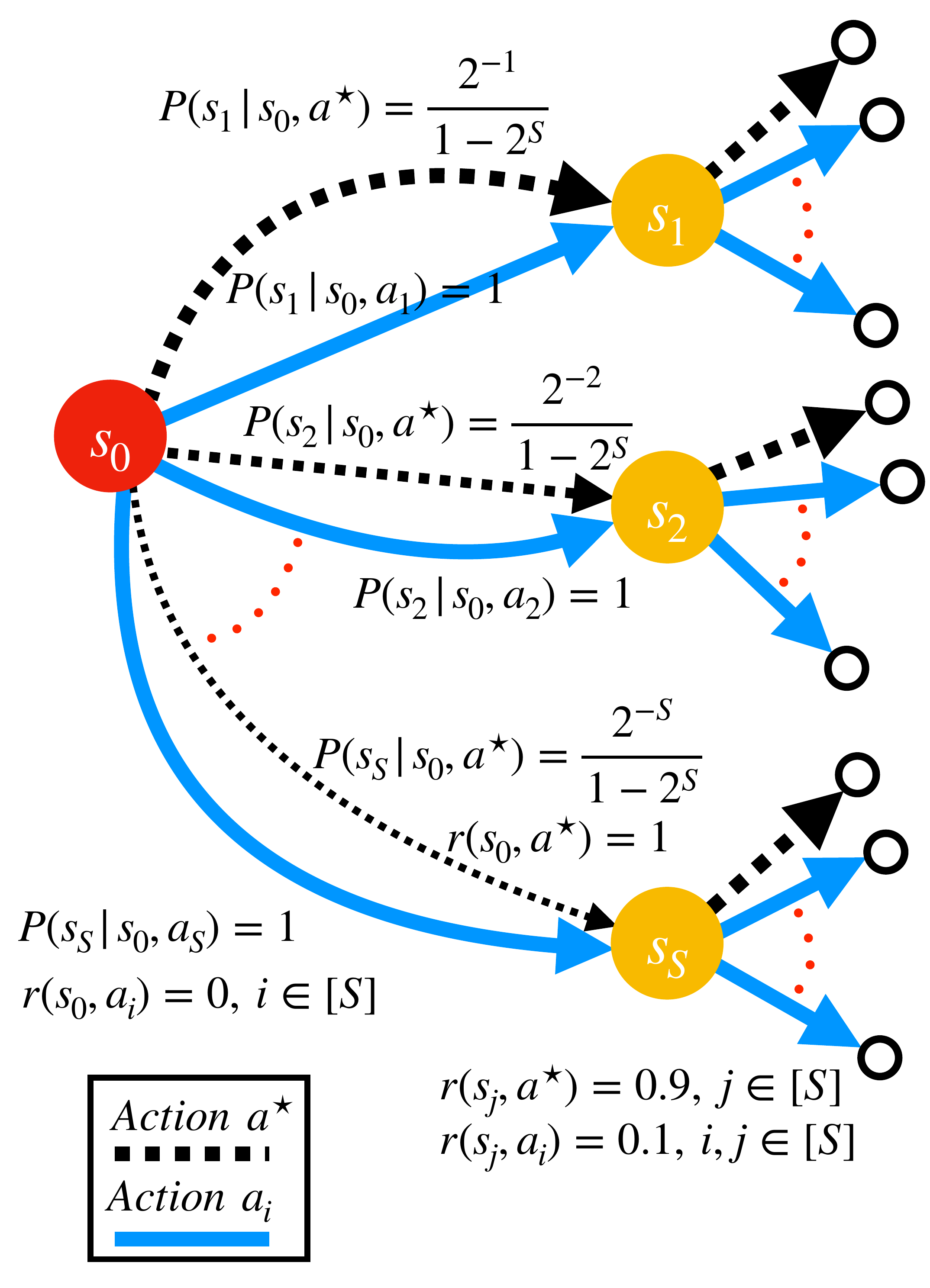}
  \end{center}
  \vspace{-10pt}
  \caption{MDP from \Cref{ex:epsilon}}
  \label{fig:ex2}
\end{wrapfigure}

\begin{inst}\label{ex:epsilon} Given a number of states $S \in \N$, consider an MDP with horizon $H= 2$, $S$ states, and $S+1$ actions, defined as in \Cref{fig:ex2}. 
\end{inst}

Similar to the example considered in \Cref{prop:two_state_ex}, here $\ast$ is the optimal action in every state, yet in state $s_0$, taking action $a_i$ is much more informative. The following result shows that this structure results in poor performance for low-regret algorithms.


\begin{prop}[Informal]\label{prop:epsilon_ex} For the MDP in \Cref{ex:epsilon} with $S$ states and small enough $\epsilon$,
to find an $\epsilon$-optimal policy with probability $1-\delta$ any learner executing \Cref{prot:lr_to_pac} with a low-regret algorithm satisfying \Cref{def:low_regret} must collect at least
$\Omega  (\frac{S \log 1/\delta}{\epsilon}  )$ 
episodes. In contrast, on this example $\Compbsolve(\cM) = \cO(S^2)$ and $\epssolved = 1/3$, so,  for $ \epsilon \le 1/3$, with probability $1-\delta$, \algname terminates and output $\pist$ in $\cOtil(\poly(S))$ episodes.
\end{prop}

In particular, this example shows that there is an exponential separation between low-regret algorithms and \algname. For exponentially small $\epsilon$, learning the optimal policy following \Cref{prot:lr_to_pac} takes $\widetilde{\Omega}(2^S)$ samples, yet \algname finds the optimal policy in $\cOtil(\poly(S))$ samples. 


\Cref{prop:epsilon_ex}, as well as \Cref{prop:two_state_ex}, imply that the true complexity of finding a good policy is often much smaller than the complexity of finding a good policy \emph{given that we explore to minimize regret}. As noted, the key piece in this example, and the example of \Cref{prop:two_state_ex}, is that the optimal action in the initial state is very uninformative---if we want to learn the optimal action in a \emph{subsequent} state, we should not take the optimal action in the initial state, but should instead take an action that leads us to the subsequent state with high probability. Nearly all existing works rely on algorithms which play policies which \emph{converge to a good policy}. For instance-dependent PAC RL, instead of \emph{playing} good policies, our examples show that an algorithm ought to explore efficiently, possibly taking very suboptimal actions in the process, and ultimately \emph{recommending} a good policy. This shortcoming of greedy algorithms motivates our design of \algname, where we seek to incorporate this insight.

While it is known that low-regret algorithms are minimax optimal for PAC RL, these instances show that running a low-regret algorithm and then an online-to-batch procedure is suboptimal by an arbitrarily large factor for PAC RL. We conclude that minimax optimality is far from being the complete story for PAC RL, and that if our goal is to simply identify a good policy, we can do much better than running a low-regret algorithm.  

\begin{rem}[Performance of Optimistic Algorithms]
Optimistic algorithms that rely on standard bonuses will also achieve low regret. This implies that recent works specifically targeting PAC bounds such as \citep{dann2019policy,menard2020fast}, which rely on optimism, will also fail to hit the optimal instance-dependent rate, or a rate of $\cO(\Compb(\cM,\epsilon))$.
In addition, even works such as \cite{xu2021fine} which do not explicitly rely on the principle of optimism and do not have known $\cO(T^\alpha)$-style regret bounds can also be shown to fail on our examples as they only take actions which may be optimal. 
\end{rem}


\section{Algorithm and Proof Sketch}\label{sec:alg_proof}

We turn now to the definition of our algorithm, \algname, and sketch out the proof of \Cref{thm:complexity}. We first provide some intuition for \algname in \Cref{sec:alg_intuition} before stating the algorithm and giving the proof sketch in \Cref{sec:alg_outline}. A detailed proof is given in \Cref{sec:detailed_proof}.

\subsection{Algorithm Intuition}\label{sec:alg_intuition}
At a high level, \algname operates by treating every state as an individual bandit, and running an action elimination-style algorithm at each state \citep{even2006action}. Unlike low-regret algorithms, \algname aggressively directs its exploration to reach uncertain states as quickly as possible. The sequential structure of an MDP introduces several unique challenges, upon which we expand below. 

\paragraph{Compounding Errors.}
In a standard bandit, from the perspective of the learner, the value of a particular action is determined solely by the environment. However, in an MDP, the value of an action $a$ at state $s$ and time $h$ depends not only on the environment, but also on the actions the learner chooses to play in subsequent steps. If we run some policy $\pihat$ after reaching $(s,h)$, though we may be able to identify the optimal action to play at $(s,h)$ \emph{given that we then play $\pihat$}, if $\pihat$ is suboptimal, this action may also be suboptimal. The following result, a direct consequence of the celebrated performance-difference lemma \citep{kakade2003sample}, is a key piece in our analysis, allowing us to effectively handle the compounding nature of errors, and may be of independent interest.

\begin{prop}\label{lem:local_to_global_subopt}
Assume that for each $h$ and $s$, $\pihat$ plays an action which satisfies $\max_a Q_{h}^{\pihat}(s,a) - \Qpihat_{h}(s,\pihat_{h}(s)) \le \epsilon_{h}(s)$.
Then the suboptimality of $\pihat$ is bounded as:
\begin{align*}
\Vst_0 - \Vpihat_0 \le \sum_{h=1}^H \sup_\pi \sum_s \wpi_h(s) \epsilon_{h}(s).
\end{align*}
\end{prop}
In particular, if $\epsilon_h(s) \le \epsilon / H $ for all $s$ and $h$, then we guarantee $\Vst_0 - \Vpihat_0 \le \epsilon$. \Cref{lem:local_to_global_subopt} in fact holds with $w_h^\pi(s)$ replaced by $\wst_h(s)$, the visitation probability under the optimal policy. We choose to work instead with the (looser) bound stated in \Cref{lem:local_to_global_subopt} as we do not in general know $\pist$ and, as we will see, can more easily control the visitations under this ``worst-case'' policy. 

Intuitively, \Cref{lem:local_to_global_subopt} says that it is sufficient to learn an action in each state that performs well \emph{as compared to the best action one could take given that $\pihat$ is played in subsequent steps}. This motivates the basic premise of our algorithm. We proceed backwards, first learning near-optimal actions in every $(s,H)$, which gives us $\pihat_H$. We then continue on to level $H-1$ where, after playing an action $a$, we play $\pihat_H$. This gives us an unbiased estimate of $\Qpihat_{H-1}(s,a)$, and allows us to determine actions that are near-optimal at stage $H-1$ if we play $\pihat$ at stage $H$. We repeat this process backwards: at stage $h$, after playing action $a$, we play $\{ \pihat_{h'} \}_{h' = h+1}^H$, yielding an unbiased estimate of $\Qpihat_{h}(s,a)$, the ``reward'' of action $a$, and allowing us to eliminate actions that are suboptimal, given that we play $\pihat$ in subsequent steps.

Our approach relies on a \emph{Monte Carlo} estimate of the value of a particular action $a$ at a given $(s,h)$. Rather than attempting to compute this value using knowledge of the MDP, or relying on a bootstrapped estimator, we simply play the policy and observe the reward obtained. As the rewards are bounded, concentration applies, allowing us to efficiently estimate $\Qpihat_{h}(s,a)$, and turning the learning problem at a given $(s,h)$ into nothing more than a bandit problem. We note that the Monte Carlo technique has previously proven useful for attaining refined gap-dependent guarantees in the regret setting \citep{xu2021fine}.

\paragraph{Balancing Suboptimality and Reachability.} 
To perform the above procedure efficiently, we must guarantee that we can reach every $(s,h)$ enough times to eliminate suboptimal actions. Bear in mind the weighting of each suboptimality, $\epsilon_h(s)$, in \Cref{lem:local_to_global_subopt}: for a given $(s,h)$, knowing an $\epsilon_h(s)$-optimal action in $(s,h)$ will only add at most $\sup_\pi \wpi_h(s) \epsilon_h(s) =: \Pst_h(s) \epsilon_h(s)$ to the total suboptimality. Thus, we only need to learn good actions in each state in proportion to how easily that state may be reached.

In particular, if we play the policy achieving $\wpi_h(s) = \Pst_h(s)$ for $K$ episodes, we will reach $(s,h)$ $\Pst_h(s) K$ times on average. By standard bandit sample complexities, we would expect it to take on order $\frac{A}{\epsilon_h(s)^2}$ samples to learn an $\epsilon_h(s)$-optimal action at $(s,h)$, so it follows that the total number of episodes we would need to run would be $K \gtrsim \frac{A}{\Pst_h(s) \epsilon_h(s)^2}$. However, if we set $\epsilon_h(s) \sim \beta \epsilon/\Pst_h(s)$, which will ensure that the suboptimality of our policy is proportional to $\epsilon$ and does not scale with the reachability, we will only require $K \gtrsim \frac{A \Pst_h(s)}{ \beta^2 \epsilon^2}$. We see then that the difficulty of reaching a state to explore it is balanced by the fact that such a state does not contribute significantly to the total suboptimality.

\paragraph{Navigating the MDP by Grouping States.} 
Naively performing the above strategy could result in a sample complexity very suboptimal in its dependence on $S$. Indeed, to ensure our final policy is $\epsilon$-suboptimal, we would need to choose $\beta \sim (SH)^{-1}$, since in this case we can only bound the suboptimality term from \Cref{lem:local_to_global_subopt} as
$$  \sum_{h=1}^H \sup_\pi \sum_s \wpi_h(s) \epsilon_{h}(s) \lesssim SH \beta \epsilon.$$
This would give us a sample complexity scaling as a suboptimal $S^3$. To overcome this, we propose an exploration procedure which \emph{groups states}---instead of exploring each state individually, in a given rollout it seeks to reach any number of states which are ``nearby'', in the sense that a single policy may reach any of them with similar probability. 

To make this practical, we take inspiration from the algorithm of \cite{zhang2020nearly}---designed for the so-called ``reward-free'' learning setting \citep{jin2020reward}, where the agent seeks merely to learn policies which traverse all reachable states---which is itself inspired by the classical \textsc{Rmax} algorithm \citep{brafman2002r}. 
We modify the true reward function, giving a reward of ``1'' to any $(s,a,h)$ pair we wish to visit, and otherwise setting the reward to ``0''. We then run a (variance-sensitive) regret minimizing algorithm, \textsc{Euler} \citep{zanette2019tighter}, on this modified reward function to generate a set of policies that can effectively traverse the MDP to visit the desired states. Critically, we show that the complexity of generating these policies amounts to a lower-order term---it is easier to learn to explore an MDP than to learn a good policy on it. Furthermore, grouping states allows us to obtain the optimal worst-case dependence on $S$ and $A$.

\subsection{Detailed Algorithm Description and Proof Sketch}\label{sec:alg_outline}
We next outline how \algname implements the above intuition and provide a proof sketch of \Cref{thm:complexity}. We first describe our core navigation procedure, \sap, in \Cref{sec:sap_overview}, then outline the main algorithm structure in \Cref{sec:mcae_overview} and \Cref{sec:moca_overview}, and finally detail the helper functions employed by \mcae\xspace in \Cref{sec:helper_functions}.

\subsubsection{\sap\xspace Overview}\label{sec:sap_overview}
\begin{algorithm}[h]
\begin{algorithmic}[1]
\Function{\sap}{active set $\cX \subseteq \cS \times \cA$, step $h$, confidence $\delta$, sampling confidence $\delsamp$, tolerance $\epsltoe$}
	\If{$|\cX| = 0$}
		\textbf{return} $\{(\emptyset,\emptyset,0,0)\}_{j=1}^{\lceil \log(1/\epsltoe) \rceil}$
	\EndIf
	\For{$j=1,\ldots,\lceil \log(1/\epsltoe) \rceil$}
		\State $K_j \leftarrow K_j(\delta/\lceil \log(1/\epsltoe) \rceil, \delsamp) $ as defined in \eqref{eq:Kjval}, $M_j \leftarrow | \cX |$, $N_j \leftarrow  K_j/(4 |\cX| \cdot 2^j)$
		\State $\cX_j, \Pi_j \leftarrow$ \ies($\cX,h,\delta,K_j,N_j$)
		\State $\cX \leftarrow \cX \backslash \cX_j$
	\EndFor
	\State \textbf{return} $\{(\cX_j,\Pi_j,N_j,M_j)\}_{j=1}^{\lceil \log(1/\epsltoe) \rceil}$
\EndFunction
\\
\Function{\ies}{active set $\cX \subseteq \cS \times \cA $, step $h$, confidence $\delta$, epochs to run $K$, samples to collect $N$}
	\State Set $r_h^1(s,a) \leftarrow 1$ for $(s,a) \in \cX$ and 0 otherwise, $N(s,a,h) \leftarrow 0$, $\cY \leftarrow \emptyset$, $\Pi \leftarrow \emptyset$, $j \leftarrow 1$ 
	\For{$k=1,2,\ldots,K$}
		\Statex \hspace{2.4em} { \color{blue} \texttt{// \euler is as defined in \cite{zanette2019tighter}}}
		\State Run \euler on reward function $r_h^j$, get trajectory $\{ (s_h^k,a_h^k,h) \}_{h=1}^H$ and policy $\pi_k$
		\State $N(s_h^k,a_h^k) \leftarrow N(s_h^k,a_h^k) + 1$, $\Pi \leftarrow \Pi \cup \pi_k$
		\If{$N(s_h^k,a_h^k) \ge N$, $(s_h^k,a_h^k) \in \cX$, and $(s_h^k,a_h^k) \not\in \cY$}\label{line:restart_euler}
			\State $\cY \leftarrow \cY \cup (s_h^k,a_h^k)$
			\State $r_h^{j+1}(s,a) \leftarrow 1$ for $(s,a) \in \cX \backslash \cY$ and 0 otherwise
			\State $j \leftarrow j + 1$
			\State Restart \euler
		\EndIf
		
	\EndFor
	\State \textbf{return} $\cY,\Pi$
\EndFunction
\end{algorithmic}
\caption{\sap}
\label{alg:partition}
\end{algorithm}

\noindent \sap\xspace implements the navigation procedure described in \Cref{sec:alg_intuition}. In particular, it takes as input a set $\cX \subseteq \cS \times \cA$, and returns a partition $\{\cX_j \}_j, \cX_j \subseteq \cX$, set of policies $\{ \Pi_j \}_j$, and values $\{ N_j \}_j$. These sets satisfy the following property.
\begin{thm}[Performance of \sap, informal]\label{thm:partitioning_works_informal}
With high probability, the partition $\{\cX_j \}_j$ returned by \sap\xspace satisfies
\begin{align*}
\sup_\pi \sum_{(s,a) \in \cX_j} \wpi_h(s,a) \le 2^{-j+1},
\end{align*}
Moreover, the policy classes $\Pi_j$ are such that, by executing a single trajectory of each $\pi \in \Pi_j$ once, we visit every $(s,a) \in \cX_j$  at least $\frac{1}{2} N_j$ times, where 
\begin{align}\label{eq:Kjval}
N_j = \cO \bigg ( \frac{2^{-j} | \Pi_j |}{| \cX \backslash \cup_{j' = 1}^{j-1} \cX_{j'} |} \bigg ), \quad | \Pi_j | = K_j( \tfrac{\delta}{\lceil \log(1/\epsltoe) \rceil}, \delsamp) = \cO(2^j S^3 A^2 H^4 \log^3 1/\delta)
\end{align}
Furthermore, if \sap\xspace is run with tolerance $\epsltoe$, it will terminate after running for at most $\poly(S,A,H, \log 1/\delta, \log 1/\epsltoe) \cdot \frac{1}{\epsltoe}$ episodes.
\end{thm}
In other words, the sets $\cX_j$ are groupings of ``nearby'' states that are increasingly difficult to reach, and the sets $\Pi_j$ give a policy cover which navigates to each $(s,a) \in \cX_j$. In addition, as $N_j = \cO(2^{-j} | \Pi_j | / | \cX \backslash \cup_{j' = 1}^{j-1} \cX_{j'} |)$, if we wish to collect $n$ samples from each $(s,a) \in \cX_j$, it will only require running for
\begin{align*}
 \cO\left (| \Pi_j | \cdot \frac{n}{2^{-j} | \Pi_j | / | \cX \backslash \cup_{j' = 1}^{j-1} \cX_{j'} |} \right ) = \cO \left ( 2^j  | \cX \backslash \cup_{j' = 1}^{j-1} \cX_{j'} | \cdot n \right ) \le \cO \left ( 2^j SA n \right )
\end{align*}
episodes. Thus, if we choose $n$ so that it is proportional to the reachability of $\cX_j$---for example, $n \sim 2^{-j}/\epsilon^2$---the total number of episodes that must be run to collect $n$ samples is no more than $\cO(\frac{SA}{\epsilon^2})$ (this can tightened to a term behaving in some cases as $\cO(\frac{|\cX_j|}{\epsilon^2})$). As we noted in \Cref{sec:alg_intuition}, it suffices to collect samples from every state in proportion with its reachability, which, combined with this fact, allows our exploration to be performed efficiently. \sap\xspace is the backbone of our sample collection procedure and is called both in \Cref{line:get_phat} of \mcae\xspace as well as in \collectsamp. We provide the full statement of \Cref{thm:partitioning_works_informal} in \Cref{sec:learn2explore}.

\subsubsection{\mcae\xspace Overview}\label{sec:mcae_overview}

\begin{algorithm}[h]
\begin{algorithmic}[1]
  	\State{}\textbf{input: } tolerance $\epsmoca$, confidence $\delmoca$, final round flag \finalround
	\State\textbf{initialize} $\epsexp \gets \frac{\epsmoca}{2H^2 S}$, $\Xgood_{h} \leftarrow \emptyset$, $\jexp = \ceil{\log \frac{1}{\epsexp}}$
	\For{each $(s,h)$}\label{line:W_loop} \algfillcomment{loop over all $s,h$ to learn maximum reachability}
		\State $\{ (\cX_{j}^{sh},\Pi_{j}^{sh},N_{j}^{sh}) \}_{j=1}^{\jexp} \leftarrow $ \sap$(\{ (s,a) \}, h, \tfrac{\delmoca}{SH}, \tfrac{1}{2}, \epsexp )$  for arbitrary $a \in \cA$ \label{line:get_phat}
		\If{$\cX_{j}^{sh} = \{ (s,a) \}$ for $j \in [\jexp]$}
			$\Phat_h(s) \leftarrow \frac{N_j^{sh}}{2 |\Pi_j^{sh}|} = \frac{1}{16 \cdot 2^j}$, $\Xgood_{h} \leftarrow \Xgood_h \cup  \{ s \}$ 
		\EndIf 
	\EndFor
	\State\textbf{set} $\iotaepsmoca \leftarrow \lceil \log \frac{64}{H^2 S \epsilon} \rceil $, $\iotadelmoca \leftarrow \log \tfrac{SAH \iotaepsmoca (\numepochs+1)}{\delmoca}$, $\numepochs \leftarrow \lceil \log \tfrac{H}{\epsmoca} \rceil $,  $\pihat_h(s) \leftarrow$ arbitrary action, $\frakA_h^0(s) \leftarrow \cA$.
	\For{$h = H,H-1,\ldots,1$}\label{line:main_loop} \hfill {\color{blue} \texttt{// loop over horizon}}
		\For{$i = 1,2,\ldots,\iotaepsmoca$} \hfill	{\color{blue} \texttt{// loop over estimated maximum reachability}}	
		\State $\Xgood_{hi} \leftarrow \{ s \in \Xgood_h \ : \ \Phat_h(s) \in [2^{-i},2^{-i+1}] \}$
		\For{$\ell = 1,\ldots, \numepochs$}\label{line:ell_loop} \hfill {\color{blue} \texttt{// loop over tolerance $\epsmoca_\ell$}}
			\State $\epsmoca_\ell \leftarrow H 2^{-\ell}$, $\Xgood_{hi}^\ell \leftarrow \{ (s,a) \ : \ s \in \Xgood_{hi}, a \in \frakA^{\ell-1}_h(s), |\frakA^{\ell-1}_h(s)| > 1 \}$ 
			\State $n_{ij}^\ell \leftarrow \tfrac{2^{18} H^2 \iotadelmoca}{2^{2i} \epsmoca_\ell^2}$, $\gamma_{ij}^\ell \leftarrow \tfrac{2^i \epsmoca_\ell}{2^8}$ for $j = 1,\ldots, \iotaepsmoca$ 
			\State $\frakD_{hi}^\ell, \{ \cX_{hij}^\ell \}_{j=1}^{\iotaepsmoca} \leftarrow$ \collectsamp($\Xgood_{hi}^\ell,\{ n_{ij}^\ell \}_{j=1}^{\iotaepsmoca},h,\pihat,\tfrac{\delmoca}{H \iotaepsmoca \numepochs},\tfrac{\epsexp}{32}$) \label{line:collectsamp1}
			\State $\{ \frakA^{\ell}_h(s) \}_{s \in \Xgood_{hi}} \leftarrow$ \collectdata$( \Xgood_{hi}^\ell, \{ \cX_{hij}^\ell \}_{j=1}^{\iotaepsmoca}, \frakD_{hi}^\ell, \{ \frakA^{\ell-1}_h(s) \}_{s \in \Xgood_{hi}}, h, \{\gamma_{ij}^\ell \}_{j=1}^{\iotaepsmoca})$ \label{line:collect_explore} \vspace{-1em}
		\EndFor
	\EndFor
	\If{\finalround\xspace is \texttt{true}}\label{line:fr_true} \hfill {\color{blue} \texttt{// ensure $\pihat$ $\epsilon$-optimal}}
	\State $\Xgood^{\numepochs+1}_h \leftarrow \{ (s,a) \ : \ s \in \Xgood_h, a \in \frakA^{\numepochs}_h(s), | \frakA^{\numepochs}_h(s) | > 1 \}$ 
	\State $\nlast_j \leftarrow \tfrac{64H^4 \iotadelmoca \iotaepsmoca^2 2^{2(-j+1)}}{\epsmoca^2}$, $\gamlast_j \leftarrow \frac{\epsmoca}{4 H \iotaepsmoca 2^{-j+1}}$ for $j = 1,\ldots,\iotaepsmoca$ \label{line:ngam2_vals}  
	\State $\frakD_h^{\numepochs+1}, \{ \Xlast_{hj} \}_{j=1}^{\iotaepsmoca} \leftarrow$ \collectsamp($\Xgood_h^{\numepochs+1}, \{\nlast_j \}_{j=1}^{\iotaepsmoca},h,\pihat,\tfrac{\delmoca}{H},\tfrac{\epsexp}{32}$) \label{line:sap_final}
	\State $\{ \frakA^{\numepochs+1}_h(s) \}_{s \in \Xgood_h^{\numepochs+1}} \leftarrow$ \collectdata$(\Xgood_h^{\numepochs+1},  \{ \Xlast_{hj} \}_{j=1}^{\iotaepsmoca}, \frakD_h^{\numepochs+1},\{ \cA_h^{\numepochs}(s) \}_{s \in \Xgood_h^{\numepochs+1}}, h, \{\gamlast_j \}_{j=1}^{\iotaepsmoca})$ \label{line:collect_final} \vspace{-1em}
	\Else
	\State $\frakA^{\numepochs+1}_h(s) \leftarrow \frakA^{\numepochs}_h(s)$ for all $s \in \Xgood_h$
	\EndIf
	\State Set $\pihat_h(s)$ to any action in $\frakA^{\numepochs+1}_h(s)$ for all $s \in \Xgood_h$ 
	\EndFor
  	\State \textbf{return} $\pihat$, $\max_{s,h} |\frakA_h^{\numepochs+1}(s) |$
\end{algorithmic}
\caption{Monte Carlo Action Elimination - Single Epoch (\mcae($\epsilon$, $\delta$, \finalround))}
\label{alg:mcae2}
\end{algorithm}

Given this description of \sap, we are ready to describe the \mcae\xspace (single-epoch $\algname$) procedure. Assume that we run \mcae\xspace with tolerance $\epsilon $ and confidence $\delta  $. We begin by calling \sap\xspace on \Cref{line:get_phat}, which allows us to form an estimate of $W_h(s)$, the maximum reachability of $(s,h)$. This in turn allows us to determine which states are efficiently reachable. We let $\Xgood_h$ denote the set of all such efficiently reachable states at stage $h$: $\Pst_h(s) \ge \frac{\epsilon}{2 H^2 S}, \forall s \in \Xgood_h$. All other states have little effect on the performance of any policy and can henceforth be ignored. The following claim shows that our estimate of $\Pst_h(s)$ is in fact accurate for $s \in \Xgood_h$.
\newpage
\begin{claim}[Informal]
If running \mcae, with high probability $\Phatst_h(s) \le \Pst_h(s) \le 32 \Phatst_h(s)$ for all $s \in \Xgood_h$. 
\end{claim}

We then proceed to our main loop over $h$ in \Cref{line:main_loop}. For a fixed $h$, we loop over $i$ and form the partition $\Xgood_{hi}$ which contains all $s \in \Xgood_h$ with $\Phatst_h(s) \sim 2^{-i}$. Given $\Xgood_{hi}$, we next loop over $\ell$, and for each $\ell$ aim to eliminate actions from $\Xgood_{hi}$ that are more than $\epsilon_\ell = H 2^{-\ell}$-suboptimal. 
We define $\Xgood_{hi}^\ell \subseteq \cS \times \cA$ as the set of $(s,a)$ for $s \in \Xgood_{hi}$, and $a$ we have not yet determined are $\epsilon_{\ell-1}/\Pst_h(s)$-suboptimal. 
To collect a sufficient number of samples from each $(s,a) \in \Xgood_{hi}^\ell$ in order to eliminate suboptimal actions,
we run \collectsamp\xspace on $\Xgood_{hi}^\ell$ and seek to collect $n_{ij}^\ell = \cO(H^2 / (2^{2i} \epstil_\ell^2)) = \cO(H^2 \Pst_h(s)^2/\epstil_\ell^2)$ from each $(s,a) \in \Xgood_{hi}^\ell$. 

Note that  every $(s,a) \in \Xgood_{hi}^\ell$ has similar maximum reachability, $\Pst_h(s) \sim 2^{-i}$, determined by index $i$. Nevertheless, as outlined in \Cref{sec:alg_intuition}, to obtain the proper scaling in $S$, we may still need to group states in a way that allows nearby states to be explored effectively. Calling \sap\xspace in \collectsamp\xspace does just this, efficiently traversing the MDP to guarantee enough samples are collected from all states in tandem.

After running \collectsamp, we run \collectdata\xspace to eliminate suboptimal actions, yielding a set of candidate $\epstil_\ell/\Pst_h(s)$-suboptimal actions for each $(s,h)$, denoted $\frakA_h^\ell(s)$. The following result shows that this procedure does indeed winnow out sufficiently suboptimal actions.
\begin{lem}[Informal]\label{lem:action_subopt_informal}
With high probability, any $a \in \frakA_h^\ell(s)$ satisfies $\Delta_h(s,a) \le \frac{3\epstil_\ell}{2 \Pst_h(s)}$.
\end{lem}
The guarantee follows by verifying that our exploration collects enough samples to ensure the confidence intervals on $\Qpihat_h(s,a)$ have width $\cO(\epstil_\ell/\Pst_h(s))$. 
Furthermore, using properties of \sap\xspace given in \Cref{thm:partitioning_works_informal}, we can bound the sample complexity of this procedure, which yields a dominant term reminiscent of our sample complexity measure, $\Compb(\cM,\epsilon)$ in \Cref{defn:gap_visitation_complexity}.

\begin{lem}[Informal]\label{lem:complexity1}
With high probability, for a given value of $h$ and $i$, the inner loop over $\ell$ on \Cref{line:ell_loop} will execute for at most
\begin{align*}
\cOtil \left ( H^2 \inf_\pi \max_{s \in \Xgood_{hi}} \max_a \min \left \{ \frac{1}{\wpi_h(s,a) \Deltil_h(s,a)^2}, \frac{\Pst_h(s)^2}{\wpi_h(s,a) \epstil^2} \right \} \right )
\end{align*}
episodes.
\end{lem}
\begin{proof}[Proof Sketch of \Cref{lem:complexity1}]
 In order to collect at least $n_{ij}^\ell$ samples from each $(s,a) \in \cX_{hij}^\ell$, \Cref{thm:partitioning_works_informal} ensures that it suffices to run for $\cO(|\Pi_j| n_{ij}^\ell/N_j) \approx \cO(2^j | \cX_{hij}^\ell| n_{ij}^\ell)$ episodes; thus, we can collect $n_{ij}^\ell$ samples from each $(s,a) \in \Xgood_{hi}^\ell$ with only $\cO(\sum_{j} 2^j | \cX_{hij}^\ell| n_{ij}^\ell)$ episodes. 

\Cref{thm:partitioning_works_informal} also shows that $\cX_{hij}^\ell$ satisfies 
\begin{align*}
\sup_\pi \min_{(s,a) \in \cX_{hij}^\ell} | \cX_{hij}^\ell| \wpi_h(s,a)\le  \sup_\pi \sum_{(s,a) \in \cX_{hij}^\ell} \wpi_h(s,a) \le 2^{-j+1},
\end{align*} 
which upper bounds the $\cO(\sum_{j} 2^j | \cX_{hij}^\ell| n_{ij}^\ell)$ episodes required by $\cO ( \sum_j \inf_\pi \max_{(s,a) \in \cX_{hij}^\ell} n_{ij}^\ell/\wpi_h(s,a))$. As all $(s,a), (s',a') \in \cX_{hij}^\ell$ satisfy $\Pst_h(s) \approx \Pst_h(s')$ by construction, $n_{ij}^\ell = \cO(H^2 \Pst_h(s)^2/\epstil_\ell^2)$ for any $(s,a) \in \cX_{hij}^\ell$, so the sample complexity reduces to $\cO (H^2 \sum_j \inf_\pi \max_{(s,a) \in \cX_{hij}^\ell} \Pst_h(s)^2/(\wpi_h(s,a) \epsilon_\ell^2))$. Finally, since actions in stage $\ell$ are only active if their gap is less than $\Pst_h(s)/\epsilon_\ell$, we obtain \Cref{lem:complexity1}.
 \end{proof}

\paragraph{The \finalround\xspace flag.} 
Single-epoch $\algname$ is called multiple times by our main algorithm (\Cref{alg:mcae3_meta}), each with geometrically decreasing tolerance $\epsilon$. For all but the smallest such $\epsilon$, \mcae\xspace is run with $\finalround\xspace = \false$, and terminates after the previously described loop over $h,i,\ell$ terminates.
The last call to $\mcae$ constitutes the ``final round'', where we set  $\finalround\xspace = \true$; this calls \collectsamp\xspace and \collectdata\xspace one more time for each $h$.

While the  loop with the $ \finalround\xspace = \false$ is able to eliminate suboptimal actions, it does not shrink the action set enough to guarantee that the returned policy is $\epstil$-optimal. In particular, while each $(s,h)$ pair upon entering this final-round loop is sub-optimal by at most $\epsilon_h(s) = \cO(\epsilon / \Pst_h(s))$, \Cref{lem:local_to_global_subopt} suggests that we actually need $\epsilon_h(s) \le \cO(\epsilon / H \cdot \sup_\pi \sum_{s' \in \cX} \wpi_h(s') )$. To remedy this, $ \finalround\xspace = \true$ invokes a final step to ensure the latter bound holds. Critically, while in the previous step we only sampled $(s,a)$ in proportion with $\Pst_h(s)^2$, the individual maximum reachability of that state, in this step we sample each $(s,a)$ in proportion with the \emph{reachability of the partition containing $(s,a)$. This subtlety is indispensable for attaining our instance-dependent sample complexity.}

In other words, after forming our set $\Xgood_h^{\numepochs+1} $ of active states and actions corresponding to the minimal error-resolution index $\ell = \ell_{\epsilon}$ (from the previous argument, this will only contain states we have not determined the optimal action for and actions that satisfy $\Delta_h(s,a) \le \frac{3\epstil}{2 \Pst_h(s)}$) and partitioning it into $\{ \cX_{hj}^{\numepochs + 1} \}_j$ by calling \sap, we seek to collect $\cO(H^4 2^{-2j}/\epstil^2)$ from every $(s,a) \in \cX_{hj}^{\numepochs+1}$. By \Cref{thm:partitioning_works_informal}, $\cX_{hj}^{\numepochs+1}$ satisfies $\sup_{\pi} \sum_{(s,a) \in \cX_{hj}^{\numepochs+1}} \wpi_h(s,a) \le 2^{-j+1}$, so sampling $(s,a)$ $\cO(H^4 2^{-2j}/\epstil^2)$  times means we sample it in proportion to its group reachability squared. As before, we can cleanly bound the suboptimality of actions remaining after this step, as well as the number of samples used by this procedure.

\begin{lem}[Informal]\label{lem:fr_good_actions}
If $s \in \Xgood_h$, then any $a \in \frakA_h^{\numepochs+1}(s)$ satisfies $\Delpihat_h(s,a) \le \cO(\frac{ \epstil}{H \cdot 2^{-j(s)+1} })$, where $j(s)$ is the largest value of $j$ such that there exists $a'$ with $(s,a') \in \Xlast_{hj}$.
\end{lem}

\begin{lem}[Informal]\label{lem:complexity2}
If \mcae\xspace is run with \finalround\xspace = \true, the procedure within the if statement on \Cref{line:fr_true} terminates in a number of episodes bounded by
\begin{align*}
\cOtil \left ( \frac{H^4}{\epstil^2} | \Xgood_h^{\numepochs+1} | \right ).
\end{align*}
\end{lem}
Critically, as noted above, $\Xgood_h^{\numepochs+1} $ will only contain near-optimal actions and unsolved states, so its cardinality could be much less than $SA$. Finally, a simple calculation combining \Cref{lem:fr_good_actions} and \Cref{lem:local_to_global_subopt} gives the following result.

\begin{lem}[Informal]\label{lem:mcae_correct2}
With high probability, if \mcae\xspace is run with \finalround\xspace = \true, it will return a policy $\pihat$ which is $\epstil$-optimal.
\end{lem}

\subsubsection{Putting everything together: \algname and proving \Cref{thm:complexity}}\label{sec:moca_overview}

\begin{algorithm}[h]
\begin{algorithmic}[1]
  	\State{}\textbf{input: } tolerance $\epsout$, confidence $\delout$
	\State $\frakA_h^0(s) \leftarrow \cA$ for all $s,h$
	\For{$m = 1,\ldots, \lceil \log(H/\epsout) \rceil - 1$}
		\State $\epsoutm \leftarrow H 2^{-m}, \deloutm \leftarrow \tfrac{\delout}{36 m^2}$
		\State $\pihat^m, \texttt{MaxOpt} \leftarrow$ \mcae$(\epsoutm,\deloutm, \texttt{false})$
		\If{$\texttt{MaxOpt} = 1$}\label{line:meta_early_term}
			\State \textbf{return} $\pihat^m$ 
		\EndIf
	\EndFor
	\State $\pihat, \texttt{MaxOpt} \leftarrow$ \mcae$(\epsout,\tfrac{\delout}{36 \lceil \log(H/\epsout) \rceil^2}, \texttt{true})$ \label{line:mcae_final_call}
	\State \textbf{return} $\pihat$
\end{algorithmic}
\caption{\textbf{MO}nte \textbf{C}arlo \textbf{A}ction Elimination (\algname)}
\label{alg:mcae3_meta}
\end{algorithm}

We turn now to our main algorithm, \algname. \algname takes as input a tolerance $\epsout$ and confidence $\delout$. Were our goal simply to find an $\epsout$-optimal policy, from the above argument we could call \mcae\xspace with tolerance $\epsout$ and \finalround\xspace = \true. However, if $\epsout$ is small enough that \mcae\xspace identifies the optimal action in \emph{every} state, this may result in overexploring---since once we have identified the optimal action in every state we can terminate and output the optimal policy. To remedy this, we instead call \mcae\xspace with exponentially decreasing tolerance and \finalround\xspace = \false. If it returns a set of actions for every $s,h$ with $|\frakA_h(s)| = 1$, we can guarantee we have identified the optimal policy, and simply terminate without overexploring. Note also in this stage, since \finalround\xspace = \false, we do not pay for the $\cOtil ( \frac{H^4}{\epstil^2} | \Xgood_h^{\numepochs+1} |  )$ term. If this condition is never met, we simply call \mcae\xspace a final time at the end with \finalround\xspace = \true\xspace to ensure the policy we return is $\epsout$-optimal. 

\Cref{thm:complexity} follows directly from this argument. In particular, the correctness of \Cref{thm:complexity}---that \algname returns an $\epsout$-optimal policy---follows from \Cref{lem:mcae_correct2}, and the sample complexity bound follows from summing the complexity bounds of \Cref{lem:complexity1} and \Cref{lem:complexity2} over all iterations.

\subsubsection{Helper Function Descriptions}\label{sec:helper_functions}

\begin{algorithm}[h]
\begin{algorithmic}[1]
	\Function{\collectsamp}{active set $\cX$, allocation $\{ n_j \}_{j=1}^{\lceil \log 1/\epscs \rceil}$, step $h$, policy $\pihat$, tolerance $\delcs$, precision $\epscs$}
	\State $\{ (\cX_j, \Pi_j, N_j) \}_{j=1}^{\lceil \log 1/\epscs \rceil} \leftarrow$ \sap($\cX, h, \delcs, \tfrac{\delcs}{\lceil \log 1/\epscs \rceil \max_j n_j}, \epscs$), $\frakD \leftarrow \emptyset$ \label{line:get_part}
		\For{$j = 1,\ldots,\lceil \log 1/\epscs \rceil$}
    		\For{$\pi \in \Pi_{j}$}
    			\State Run $\pi$ for $T = \lceil  2 n_{j}/N_j \rceil$ times up to level $h$, then play $\pihat$
			\State Collect reward rollouts $\frakD \leftarrow \frakD \cup \{ s_h^t, a_h^t, \qcheckpith(s_h^t,a_h^t) := \sum_{h'=h}^H R_{h'}^t \}_{t=1}^T$
  		\EndFor
		\EndFor
	\State \textbf{return} $\frakD$, $\{ \cX_j \}_{j=1}^{\lceil \log 1/\epscs \rceil}$
	\EndFunction
	\\

	\Function{\collectdata}{active set $\cX$, partition $\{ \cX_j \}_{j=1}^k$, dataset $\frakD$, active actions $\{ \frakA_h(s) \}_{s \in \Xgood}$, level $h$, thresholds $\{\gamma_j \}_{j=1}^k$} 
	\For{$(s,a) \in \cX$}
		\State $N_h(s,a) \leftarrow  \sum_{(s_h^t, a_h^t, \qcheckpith(s_h^t,a_h^t)) \in \frakD} \I \{ (s_h^t,a_h^t) = (s,a) \} $
		\State $\Qhatpihat_{h}(s,a) \leftarrow \frac{1}{N_h(s,a)}  \sum_{(s_h^t, a_h^t, \qcheckpith(s_h^t,a_h^t)) \in \frakD} \I \{ (s_h^t,a_h^t) = (s,a) \} \cdot \qcheckpith(s_h^t,a_h^t)  $ \label{line:collect_Qhat}
	\EndFor
	\For{$j = 1,\ldots,k$}
		\For{$s$ s.t. $\exists a$ with $(s,a) \in \cX_j$}
			\State $j(s) \leftarrow \argmax_{j'} j' \ \text{s.t.} \ \exists a', (s,a') \in \cX_{j'}$ \label{line:set_js}
			\State $\frakA_h(s) \leftarrow \{ a \in \frakA_h(s) \ : \  \max_{a' \in \frakA_h(s)} \Qhatpihat_{h}(s,a') - \Qhatpihat_{h}(s,a) \le \gamma_{j(s)} \}$  \label{line:collect_Aell}
		\EndFor
	\EndFor
	\State \textbf{return} $\{ \frakA_h(s) \}_{s \in \Xgood}$
	\EndFunction
\end{algorithmic}
\caption{\algname Helper Functions}
\label{alg:collect_samples}
\end{algorithm}

\paragraph{Description of \collectsamp.}
\collectsamp\xspace takes as input a set $\cX \subseteq \cS \times \cA$, an allocation $\{ n_j\}_j$, a timestep $h$, and a policy $\pihat$. In short, \collectsamp\xspace first calls \sap\xspace on $\cX$ to obtain a partition $\{ \cX_j \}_j$, and then reruns the policies returned by \sap\xspace enough times to ensure that every $(s,a) \in \cX_j$ is reached at least $n_j$ times at timestep $h$. After reaching $(s,a,h)$, $\pihat$ is played, to obtain a Monte Carlo estimate $\qcheckpith(s,a) $ of $\Qpihat_h(s,a)$. \collectsamp\xspace then returns the data collected and the partition returned by \sap.

\paragraph{Description of \collectdata.}
\collectdata\xspace takes as input a set $\cX \subseteq \cS \times \cA$, a partition of this set $\{\cX_j \}_j$, a dataset $\frakD$ generated by \collectsamp, a set of active actions $\{ \frakA_h(s) \}_{s}$, a timestep $h$, and a threshold $\{ \gamma_j \}_j$. For each $(s,a) \in \cX$, it forms an estimate of $\Qpihat_h(s,a)$ from the rollouts in $\frakD$. Given these estimates, for $s$ such that there exists $a$ with $(s,a) \in \cX_j$, it removes actions from $\frakA_h(s)$ that are more than $\gamma_{j(s)}$-suboptimal.


\section{Conclusion}\label{sec:conclusion}
In this work, we proposed a new instance-dependent measure of complexity for PAC RL, the gap-visitation complexity, showed that our algorithm, \algname, hits this complexity, and, through several examples, showed that running a low-regret procedure cannot be instance-optimal for PAC RL. Our work opens several interesting directions for future work.

\begin{itemize}
\item While the gap-visitation complexity takes into account the maximum reachability of a given state, it does not take into account how easily a given state may be reached by a near-optimal policy. One could imagine an MDP where some state, $s$, is easily reached by a suboptimal policy but is never visited by near-optimal policies. In this case, a PAC algorithm need not learn a good action in this state to return an $\epsilon$-optimal policy, yet \algname currently would do so. We believe that this idea---weighting states during exploration not by their maximum visitation but by their visitation from near-optimal policies---could be incorporated into our current framework, but leave the details of this to future work.
\item Neither this work nor \cite{marjani2021navigating} hit the true instance-optimal lower bound which, as shown in \cite{marjani2021navigating}, is the solution to a non-convex optimization problem even for best-policy identification. The previous point suggests that $\Compb(\cM,\epsilon)$ is not in general the instance-dependent lower bound, though \Cref{prop:bpi_lb1} and \Cref{prop:complexity_bandit} show that in certain cases it does match the instance-dependent lower bound.
Relating $\Compb(\cM,\epsilon)$ to the true lower bound in general and developing algorithms that hit the lower bound would both be interesting directions for future work.
\item By running an algorithm that achieves gap-dependent logarithmic regret (such as \cite{simchowitz2019non}) and performing an online-to-batch conversion, one can obtain a PAC sample complexity of
\begin{align}\label{eq:gap_eps_pac}
\cO \bigg ( \sum_{s,a,h : \Delta_h(s,a) > 0} \frac{1}{\Delta_h(s,a) \epsilon} \cdot \frac{1}{\delta^2} \bigg ).
\end{align}
While \Cref{cor:complexity2} shows that \algname achieves a similar complexity, albeit with a $\log 1/\delta$ scaling, it must also pay for the $\frac{|\opt(\epsilon)|}{\epsilon^2}$ term, which could dominate the $\frac{1}{\Delta_h(s,a)\epsilon}$ term. We believe removing this term (or showing it is necessary) and obtaining a sample complexity of the form \eqref{eq:gap_eps_pac} but that scales instead with $\log 1/\delta$ is an important step in understanding the true complexity of PAC reinforcement learning.
\end{itemize}

\subsection*{Acknowledgements}

The work of AW is supported by an NSF GFRP Fellowship DGE-1762114. MS is generously supported by an Open Philanthropy AI Fellowship. The work of KJ is funded in part by the AFRL and NSF TRIPODS 2023166.

\bibliographystyle{icml2021}
\bibliography{bibliography.bib}

\begin{thebibliography}{40}
\providecommand{\natexlab}[1]{#1}
\providecommand{\url}[1]{\texttt{#1}}
\expandafter\ifx\csname urlstyle\endcsname\relax
  \providecommand{\doi}[1]{doi: #1}\else
  \providecommand{\doi}{doi: \begingroup \urlstyle{rm}\Url}\fi

\bibitem[Agarwal et~al.(2020)Agarwal, Kakade, and Yang]{agarwal2020model}
Agarwal, A., Kakade, S., and Yang, L.~F.
\newblock Model-based reinforcement learning with a generative model is minimax
  optimal.
\newblock In \emph{Conference on Learning Theory}, pp.\  67--83. PMLR, 2020.

\bibitem[Azar et~al.(2013)Azar, Munos, and Kappen]{azar2013minimax}
Azar, M.~G., Munos, R., and Kappen, H.~J.
\newblock Minimax pac bounds on the sample complexity of reinforcement learning
  with a generative model.
\newblock \emph{Machine learning}, 91\penalty0 (3):\penalty0 325--349, 2013.

\bibitem[Azar et~al.(2017)Azar, Osband, and Munos]{azar2017minimax}
Azar, M.~G., Osband, I., and Munos, R.
\newblock Minimax regret bounds for reinforcement learning.
\newblock In \emph{International Conference on Machine Learning}, pp.\
  263--272. PMLR, 2017.

\bibitem[Brafman \& Tennenholtz(2002)Brafman and Tennenholtz]{brafman2002r}
Brafman, R.~I. and Tennenholtz, M.
\newblock R-max-a general polynomial time algorithm for near-optimal
  reinforcement learning.
\newblock \emph{Journal of Machine Learning Research}, 3\penalty0
  (Oct):\penalty0 213--231, 2002.

\bibitem[Dann \& Brunskill(2015)Dann and Brunskill]{dann2015sample}
Dann, C. and Brunskill, E.
\newblock Sample complexity of episodic fixed-horizon reinforcement learning.
\newblock \emph{arXiv preprint arXiv:1510.08906}, 2015.

\bibitem[Dann et~al.(2017)Dann, Lattimore, and Brunskill]{dann2017unifying}
Dann, C., Lattimore, T., and Brunskill, E.
\newblock Unifying pac and regret: Uniform pac bounds for episodic
  reinforcement learning.
\newblock \emph{arXiv preprint arXiv:1703.07710}, 2017.

\bibitem[Dann et~al.(2019)Dann, Li, Wei, and Brunskill]{dann2019policy}
Dann, C., Li, L., Wei, W., and Brunskill, E.
\newblock Policy certificates: Towards accountable reinforcement learning.
\newblock In \emph{International Conference on Machine Learning}, pp.\
  1507--1516. PMLR, 2019.

\bibitem[Dann et~al.(2021)Dann, Marinov, Mohri, and Zimmert]{dann2021beyond}
Dann, C., Marinov, T.~V., Mohri, M., and Zimmert, J.
\newblock Beyond value-function gaps: Improved instance-dependent regret bounds
  for episodic reinforcement learning.
\newblock \emph{Advances in Neural Information Processing Systems}, 34, 2021.

\bibitem[Degenne \& Koolen(2019)Degenne and Koolen]{degenne2019pure}
Degenne, R. and Koolen, W.~M.
\newblock Pure exploration with multiple correct answers.
\newblock \emph{arXiv preprint arXiv:1902.03475}, 2019.

\bibitem[Even-Dar et~al.(2006)Even-Dar, Mannor, Mansour, and
  Mahadevan]{even2006action}
Even-Dar, E., Mannor, S., Mansour, Y., and Mahadevan, S.
\newblock Action elimination and stopping conditions for the multi-armed bandit
  and reinforcement learning problems.
\newblock \emph{Journal of machine learning research}, 7\penalty0 (6), 2006.

\bibitem[Freedman(1975)]{freedman1975tail}
Freedman, D.~A.
\newblock On tail probabilities for martingales.
\newblock \emph{the Annals of Probability}, pp.\  100--118, 1975.

\bibitem[Garivier \& Kaufmann(2016)Garivier and Kaufmann]{garivier2016optimal}
Garivier, A. and Kaufmann, E.
\newblock Optimal best arm identification with fixed confidence.
\newblock In \emph{Conference on Learning Theory}, pp.\  998--1027. PMLR, 2016.

\bibitem[Jin et~al.(2018)Jin, Allen-Zhu, Bubeck, and Jordan]{jin2018q}
Jin, C., Allen-Zhu, Z., Bubeck, S., and Jordan, M.~I.
\newblock Is q-learning provably efficient?
\newblock In \emph{Proceedings of the 32nd International Conference on Neural
  Information Processing Systems}, pp.\  4868--4878, 2018.

\bibitem[Jin et~al.(2020{\natexlab{a}})Jin, Krishnamurthy, Simchowitz, and
  Yu]{jin2020reward}
Jin, C., Krishnamurthy, A., Simchowitz, M., and Yu, T.
\newblock Reward-free exploration for reinforcement learning.
\newblock In \emph{International Conference on Machine Learning}, pp.\
  4870--4879. PMLR, 2020{\natexlab{a}}.

\bibitem[Jin et~al.(2020{\natexlab{b}})Jin, Yang, Wang, and
  Jordan]{jin2020provably}
Jin, C., Yang, Z., Wang, Z., and Jordan, M.~I.
\newblock Provably efficient reinforcement learning with linear function
  approximation.
\newblock In \emph{Conference on Learning Theory}, pp.\  2137--2143. PMLR,
  2020{\natexlab{b}}.

\bibitem[Jonsson et~al.(2020)Jonsson, Kaufmann, M{\'e}nard, Domingues, Leurent,
  and Valko]{jonsson2020planning}
Jonsson, A., Kaufmann, E., M{\'e}nard, P., Domingues, O.~D., Leurent, E., and
  Valko, M.
\newblock Planning in markov decision processes with gap-dependent sample
  complexity.
\newblock \emph{arXiv preprint arXiv:2006.05879}, 2020.

\bibitem[Kakade(2003)]{kakade2003sample}
Kakade, S.~M.
\newblock \emph{On the sample complexity of reinforcement learning}.
\newblock PhD thesis, UCL (University College London), 2003.

\bibitem[Kaufmann et~al.(2016)Kaufmann, Capp{\'e}, and
  Garivier]{kaufmann2016complexity}
Kaufmann, E., Capp{\'e}, O., and Garivier, A.
\newblock On the complexity of best-arm identification in multi-armed bandit
  models.
\newblock \emph{The Journal of Machine Learning Research}, 17\penalty0
  (1):\penalty0 1--42, 2016.

\bibitem[Kearns \& Singh(2002)Kearns and Singh]{kearns2002near}
Kearns, M. and Singh, S.
\newblock Near-optimal reinforcement learning in polynomial time.
\newblock \emph{Machine learning}, 49\penalty0 (2):\penalty0 209--232, 2002.

\bibitem[Khamaru et~al.(2020)Khamaru, Pananjady, Ruan, Wainwright, and
  Jordan]{khamaru2020temporal}
Khamaru, K., Pananjady, A., Ruan, F., Wainwright, M.~J., and Jordan, M.~I.
\newblock Is temporal difference learning optimal? an instance-dependent
  analysis.
\newblock \emph{arXiv preprint arXiv:2003.07337}, 2020.

\bibitem[Khamaru et~al.(2021)Khamaru, Xia, Wainwright, and
  Jordan]{khamaru2021instance}
Khamaru, K., Xia, E., Wainwright, M.~J., and Jordan, M.~I.
\newblock Instance-optimality in optimal value estimation: Adaptivity via
  variance-reduced q-learning.
\newblock \emph{arXiv preprint arXiv:2106.14352}, 2021.

\bibitem[Lattimore \& Hutter(2012)Lattimore and Hutter]{lattimore2012pac}
Lattimore, T. and Hutter, M.
\newblock Pac bounds for discounted mdps.
\newblock In \emph{International Conference on Algorithmic Learning Theory},
  pp.\  320--334. Springer, 2012.

\bibitem[Li et~al.(2020)Li, Wei, Chi, Gu, and Chen]{li2020breaking}
Li, G., Wei, Y., Chi, Y., Gu, Y., and Chen, Y.
\newblock Breaking the sample size barrier in model-based reinforcement
  learning with a generative model.
\newblock \emph{Advances in Neural Information Processing Systems}, 33, 2020.

\bibitem[Marjani \& Proutiere(2020)Marjani and Proutiere]{marjani2020best}
Marjani, A.~A. and Proutiere, A.
\newblock Best policy identification in discounted mdps: Problem-specific
  sample complexity.
\newblock \emph{arXiv preprint arXiv:2009.13405}, 2020.

\bibitem[Marjani et~al.(2021)Marjani, Garivier, and
  Proutiere]{marjani2021navigating}
Marjani, A.~A., Garivier, A., and Proutiere, A.
\newblock Navigating to the best policy in markov decision processes.
\newblock \emph{arXiv preprint arXiv:2106.02847}, 2021.

\bibitem[Maurer \& Pontil(2009)Maurer and Pontil]{maurer2009empirical}
Maurer, A. and Pontil, M.
\newblock Empirical bernstein bounds and sample variance penalization.
\newblock \emph{arXiv preprint arXiv:0907.3740}, 2009.

\bibitem[M{\'e}nard et~al.(2020)M{\'e}nard, Domingues, Jonsson, Kaufmann,
  Leurent, and Valko]{menard2020fast}
M{\'e}nard, P., Domingues, O.~D., Jonsson, A., Kaufmann, E., Leurent, E., and
  Valko, M.
\newblock Fast active learning for pure exploration in reinforcement learning.
\newblock \emph{arXiv preprint arXiv:2007.13442}, 2020.

\bibitem[Ok et~al.(2018)Ok, Proutiere, and Tranos]{ok2018exploration}
Ok, J., Proutiere, A., and Tranos, D.
\newblock Exploration in structured reinforcement learning.
\newblock \emph{arXiv preprint arXiv:1806.00775}, 2018.

\bibitem[Puterman(2014)]{puterman2014markov}
Puterman, M.~L.
\newblock \emph{Markov decision processes: discrete stochastic dynamic
  programming}.
\newblock John Wiley \& Sons, 2014.

\bibitem[Sidford et~al.(2018)Sidford, Wang, Wu, Yang, and Ye]{sidford2018near}
Sidford, A., Wang, M., Wu, X., Yang, L.~F., and Ye, Y.
\newblock Near-optimal time and sample complexities for solving discounted
  markov decision process with a generative model.
\newblock \emph{arXiv preprint arXiv:1806.01492}, 2018.

\bibitem[Simchowitz \& Jamieson(2019)Simchowitz and
  Jamieson]{simchowitz2019non}
Simchowitz, M. and Jamieson, K.
\newblock Non-asymptotic gap-dependent regret bounds for tabular mdps.
\newblock \emph{arXiv preprint arXiv:1905.03814}, 2019.

\bibitem[Tsybakov(2009)]{tsybakov2009introduction}
Tsybakov, A.~B.
\newblock Introduction to nonparametric estimation., 2009.

\bibitem[Wagenmaker et~al.(2021)Wagenmaker, Simchowitz, and
  Jamieson]{wagenmaker2021task}
Wagenmaker, A., Simchowitz, M., and Jamieson, K.
\newblock Task-optimal exploration in linear dynamical systems.
\newblock \emph{arXiv preprint arXiv:2102.05214}, 2021.

\bibitem[Wang et~al.(2020)Wang, Du, Yang, and Kakade]{wang2020long}
Wang, R., Du, S.~S., Yang, L.~F., and Kakade, S.~M.
\newblock Is long horizon reinforcement learning more difficult than short
  horizon reinforcement learning?
\newblock \emph{arXiv preprint arXiv:2005.00527}, 2020.

\bibitem[Xu et~al.(2021)Xu, Ma, and Du]{xu2021fine}
Xu, H., Ma, T., and Du, S.~S.
\newblock Fine-grained gap-dependent bounds for tabular mdps via adaptive
  multi-step bootstrap.
\newblock \emph{arXiv preprint arXiv:2102.04692}, 2021.

\bibitem[Zanette \& Brunskill(2019)Zanette and Brunskill]{zanette2019tighter}
Zanette, A. and Brunskill, E.
\newblock Tighter problem-dependent regret bounds in reinforcement learning
  without domain knowledge using value function bounds.
\newblock In \emph{International Conference on Machine Learning}, pp.\
  7304--7312. PMLR, 2019.

\bibitem[Zanette et~al.(2019)Zanette, Kochenderfer, and
  Brunskill]{zanette2019generative}
Zanette, A., Kochenderfer, M.~J., and Brunskill, E.
\newblock Almost horizon-free structure-aware best policy identification with a
  generative model.
\newblock In \emph{Advances in Neural Information Processing Systems},
  volume~32. Curran Associates, Inc., 2019.

\bibitem[Zhang et~al.(2020{\natexlab{a}})Zhang, Du, and Ji]{zhang2020nearly}
Zhang, Z., Du, S.~S., and Ji, X.
\newblock Nearly minimax optimal reward-free reinforcement learning.
\newblock \emph{arXiv preprint arXiv:2010.05901}, 2020{\natexlab{a}}.

\bibitem[Zhang et~al.(2020{\natexlab{b}})Zhang, Ji, and
  Du]{zhang2020reinforcement}
Zhang, Z., Ji, X., and Du, S.~S.
\newblock Is reinforcement learning more difficult than bandits? a near-optimal
  algorithm escaping the curse of horizon.
\newblock \emph{arXiv preprint arXiv:2009.13503}, 2020{\natexlab{b}}.

\bibitem[Zimin \& Neu(2013)Zimin and Neu]{zimin2013online}
Zimin, A. and Neu, G.
\newblock Online learning in episodic markovian decision processes by relative
  entropy policy search.
\newblock In \emph{Neural Information Processing Systems 26}, 2013.

\end{thebibliography}

\newpage

\appendix

\newpage

\section{Interpreting the Gap-Visitation Complexity}\label{sec:non_unique_actions}

\begin{prop}\label{prop:relate_complexities}
The gap-visitation complexity, $\Compb(\cM,\epsilon)$, satisfies
\begin{align}
\Compb(\cM,\epsilon) & = \sum_{h=1}^H \inf_\pi \max_s \frac{1}{\wpi_h(s)} \sum_{a}  \min \left \{ \frac{1}{ \Deltil_h(s,a)^2},  \frac{\Pst_h(s)^2}{\epsilon^2}  \right \} + \frac{H^2 | \opt(\epsilon)|}{\epsilon^2}. \nonumber
\end{align}
Furthermore, when $\cM$ has unique optimal actions, the best-policy gap-visitation complexity, $\Compbsolve(\cM)$, satisfies
\begin{align*}
\Compbsolve(\cM) & = \sum_{h=1}^H \inf_\pi \max_s \frac{1}{\wpi_h(s)} \sum_{a: \Delta_h(s,a) > 0} \frac{1}{\Delta_h(s,a)^2} .
\end{align*}
\end{prop}

\begin{proof}
Consider the optimization
\begin{align*}
\min_{\lambda \in \simplex(X)} \max_{x \in X}  a_x/\lambda_x.
\end{align*}
It is easy to see that
\begin{align*}
\sum_{x \in X} a_x = \min_{\lambda \in \simplex(X)} \max_{x \in X} a_x/\lambda_x
\end{align*}
and the optimal $\lambda$ is 
\begin{align*}
\lambda^*_x = \frac{a_x}{\sum_{x' \in X} a_{x'}}.
\end{align*}
For any policy $\pi$, we will have that $\sum_a \pi_h(a|s)=1$, and $\pi_h(a|s)$ must be a valid distribution over $a$. This implies that $\wpi_h(s,a) = \wpi_h(s) \pi_h(a|s)$. Now fix $\pi$ for steps $h' = 1,\ldots,h-1$, then it follows that
\begin{align*}
& \inf_{\pi_h} \max_{s,a}  \min \left \{ \frac{1}{\wpi_h(s,a) \Deltil_h(s,a)^2},  \frac{\Pst_h(s)^2}{\wpi_h(s,a) \epsilon^2}  \right \} = \inf_{\pi_h} \max_s \frac{1}{\wpi_h(s)}  \max_a \frac{1}{\pi_h(a|s)} \min \left \{ \frac{1}{ \Deltil_h(s,a)^2},  \frac{\Pst_h(s)^2}{ \epsilon^2}  \right \}.
\end{align*}
Now for a given $s$, we can use that $\wpi_h(s)$ is independent of $\pi_h$ and apply our above calculation to get that
\begin{align*}
\inf_{\pi_h} \frac{1}{\wpi_h(s)} \max_a \frac{1}{\pi_h(a|s)} \min \left \{ \frac{1}{ \Deltil_h(s,a)^2},  \frac{\Pst_h(s)^2}{ \epsilon^2}  \right \} = \frac{1}{\wpi_h(s)} \sum_a \min \left \{ \frac{1}{ \Deltil_h(s,a)^2},  \frac{\Pst_h(s)^2}{ \epsilon^2}  \right \}.
\end{align*}
As the maximum over $s$ is over a finite set and $\pi_h(\cdot | s)$ can be chosen independently of $\pi_h(\cdot | s')$ for any $s \neq s'$, we have that
\begin{align*}
 \inf_{\pi_h} \max_s \frac{1}{\wpi_h(s)}  \max_a \frac{1}{\pi_h(a|s)} \min \left \{ \frac{1}{ \Deltil_h(s,a)^2},  \frac{\Pst_h(s)^2}{ \epsilon^2}  \right \}= \max_s \frac{1}{\wpi_h(s)} \sum_a \min \left \{ \frac{1}{ \Deltil_h(s,a)^2},  \frac{\Pst_h(s)^2}{ \epsilon^2}  \right \}.
\end{align*}
Since taking an inf over $\pi$ is equivalent to taking an inf over $\{ \pi_{h'} \}_{h'=1}^{h-1}$ and $\pi_h$, we can take the inf of this over $\{ \pi_{h'} \}_{h'=1}^{h-1}$ to get
\begin{align*}
& \inf_\pi \max_{s,a}  \min \left \{ \frac{1}{\wpi_h(s,a) \Deltil_h(s,a)^2},  \frac{\Pst_h(s)^2}{\wpi_h(s,a) \epsilon^2}  \right \}  = \inf_\pi \max_s \frac{1}{\wpi_h(s)} \sum_a \min \left \{ \frac{1}{ \Deltil_h(s,a)^2},  \frac{\Pst_h(s)^2}{ \epsilon^2}  \right \}.
\end{align*}
The same line of reasoning can be used to obtain the expression for $\Compbsolve(\cM)$. 
\end{proof}

\begin{prop}
We can bound
\begin{align*}
\Compb(\cM,\epsilon) & \le \sum_{h=1}^H \inf_{\pi} \max_{s,a}   \frac{4}{\wpi_h(s,a) \Deltil_h^\epsilon(s,a)^2 + \tfrac{\epsilon^2}{SA}}
\end{align*}
where
\begin{align*}
\Deltil_h^{\epsilon}(s,a) := \begin{cases}  \Deltil_h(s,a) & \tfrac{\epsilon}{\Pst_h(s)} < \tfrac{ \Deltil_h(s,a)}{3} \\
\epsilon/H & \tfrac{\epsilon}{\Pst_h(s)} \ge \tfrac{  \Deltil_h(s,a)}{3} \end{cases} . 
\end{align*}
\end{prop}
\begin{proof}
Let $\opt_h(\epsilon) = \{ (s,a)  \ : \ \Deltil_h(s,a) \Pst_h(s)/3 \le \epsilon \}$ so that $\opt(\epsilon) = \cup_h \opt_h(\epsilon)$. We can always bound $|\opt_h(\epsilon)| \le SA$, and furthermore,
\begin{align*}
\frac{H^2 |\opt_h(\epsilon)|}{\epsilon^2} & = \min \left \{ \frac{H^2}{1/|\opt_h(\epsilon)| \cdot \epsilon^2}, \frac{H^2SA}{\epsilon^2} \right \} \\
& \overset{(a)}{=} \inf_{\lambda \in \simplex(\opt_h(\epsilon))} \max_{(s,a) \in \opt_h(\epsilon)} \min \left \{ \frac{H^2}{\lambda_{sa} \epsilon^2}, \frac{H^2SA}{\epsilon^2} \right \} \\
& \le \inf_{\pi} \max_{(s,a) \in \opt_h(\epsilon)} \min \left \{ \frac{H^2}{\wpi_h(s,a) \epsilon^2}, \frac{H^2SA}{\epsilon^2} \right \} \\
& \overset{(b)}{\le} \inf_{\pi} \max_{(s,a) \in \opt_h(\epsilon)} \frac{2H^2}{\wpi_h(s,a) \epsilon^2 + \frac{\epsilon^2}{SA}}
\end{align*}
where $(a)$ follows since the optimal distribution will simply place a mass of $1/|\opt_h(\epsilon)|$ on each $(s,a) \in \opt_h(\epsilon)$, and $(b)$ follows since $\min \{ \frac{1}{a}, \frac{1}{b} \} = \frac{1}{\max \{ a,b \}} \le \frac{1}{a/2 + b/2}$. 

Consider the distribution $\pi'$ which is a mixture of distribution $\pi$ 1/2 of the time, and the distribution $\pi^{sh}$ $1/(2SA)$ of the time, where $\pi^{sh}$ is the distribution which achieves $w^{\pi^{sh}}_h(s) = \Pst_h(s)$. In other words, we will have $w^{\pi'}_h(s,a) \ge \wpi_h(s,a)/2 + \Pst_h(s)/(2SA)$. Given this, we can bound
\begin{align*}
& \inf_\pi \max_{s,a} \min  \left \{ \frac{1}{\wpi_h(s,a) \Deltil_h(s,a)^2}, \frac{\Pst_h(s)^2}{\wpi_h(s,a) \epsilon^2} \right \} \\
& \qquad \le \inf_\pi \max_{s,a} \min \left \{ \frac{2}{\wpi_h(s,a) \Deltil_h(s,a)^2 + \Pst_h(s) \Deltil_h(s,a)^2/SA}, \frac{2 \Pst_h(s)^2}{\wpi_h(s,a) \epsilon^2 + \Pst_h(s) \epsilon^2/SA} \right \} \\
& \qquad \le \inf_\pi \bigg [ \max_{(s,a) \in \opt_h(\epsilon)^c} \frac{2}{\wpi_h(s,a) \Deltil_h(s,a)^2 + \Pst_h(s) \Deltil_h(s,a)^2/SA} + \max_{(s,a) \in \opt_h(\epsilon)} \frac{2}{\wpi_h(s,a) \epsilon^2 + \epsilon^2/SA} \bigg ] .
\end{align*}
If $(s,a) \in \opt_h(\epsilon)^c$, then $\Deltil_h(s,a) \Pst_h(s) > 3\epsilon$, so $\Pst_h(s) \Deltil_h(s,a)^2 \ge 3 \Deltil_h(s,a) \epsilon \ge \epsilon^2$. Thus, we can bound the above as
\begin{align*}
 \le \inf_\pi \bigg [ \max_{(s,a) \in \opt_h(\epsilon)^c} \frac{2}{\wpi_h(s,a) \Deltil_h(s,a)^2 + \epsilon^2/SA} + \max_{(s,a) \in \opt_h(\epsilon)} \frac{2}{\wpi_h(s,a) \epsilon^2 + \epsilon^2/SA} \bigg ] .
\end{align*}
The result then follows combining this with the bound on $\frac{H^2 |\opt_h(\epsilon)|}{\epsilon^2}$ given above, and using the definition of $\Deltil_h^\epsilon(s,a)$. 
\end{proof}

\begin{prop}
We can bound
\begin{align*}
\cC(\cM,\epsilon) \le \sum_{s,a,h} \frac{1}{\epsilon \max \{\Deltil_h(s,a),\epsilon \}} + \frac{H^2 | \opt(\epsilon)|}{\epsilon^2}  . 
\end{align*}
\end{prop}
\begin{proof}
This follows from \Cref{prop:relate_complexities} and noting that
\begin{align*}
 \min \left \{ \frac{1}{\Pst_h(s) \Deltil_h(s,a)^2}, \frac{ \Pst_h(s)}{\epsilon^2} \right \} & \le  \min \left \{ \frac{1}{\sqrt{\Pst_h(s)} \Deltil_h(s,a)}, \frac{ \sqrt{\Pst_h(s)}}{\epsilon} \right \} \cdot \frac{\sqrt{\Pst_h(s)}}{\epsilon} \\
&  \le  \min \left \{ \frac{1}{\Deltil_h(s,a) \epsilon}, \frac{ 1}{\epsilon^2} \right \}.
\end{align*}
\end{proof}

\begin{proof}[Proof of \Cref{cor:complexity2}]
Let $\pi^{sh}$ denote the policy that achieves $w_h^{\pi^{sh}}(s) = \Pst_h(s)$. Consider the state visitation distribution:
\begin{align*}
w_h'(s) =  \frac{ \sum_{s'} w_h^{\pi^{s'h}}(s) \cdot \sum_{a}  \min \left \{ \frac{1}{ \Pst_h(s') \Deltil_h(s',a)^2},  \frac{\Pst_h(s')}{\epsilon^2}  \right \} } {\sum_{s',a} \min \left \{ \frac{1}{\Pst_h(s') \Deltil_h(s',a)^2}, \frac{ \Pst_h(s')}{\epsilon^2} \right \} }.
\end{align*}
Since the set of state visitations realizable on a given MDP is convex and for any realizable state distribution there exists a policy with that state distribution by \Cref{prop:state_act_vis}, and since $w_h'$ is a convex combination of state visitation distributions, it follows that there exists some policy $\pitil$ such that $w_h'(s) = w_h^{\pitil}(s)$. Furthermore, by definition, 
\begin{align*}
w_h^{\pitil}(s) \ge \frac{  w_h^{\pi^{sh}}(s) \cdot \sum_{a}  \min \left \{ \frac{1}{\Pst_h(s) \Deltil_h(s,a)^2},  \frac{\Pst_h(s)}{\epsilon^2}  \right \} } {\sum_{s',a} \min \left \{ \frac{1}{ \Pst_h(s') \Deltil_h(s',a)^2}, \frac{ \Pst_h(s')}{\epsilon^2} \right \} } =  \Pst_h(s) \cdot \frac{\sum_{a}  \min \left \{ \frac{1}{ \Pst_h(s) \Deltil_h(s,a)^2},  \frac{\Pst_h(s)}{\epsilon^2}  \right \} } {\sum_{s',a} \min \left \{ \frac{1}{\Pst_h(s') \Deltil_h(s',a)^2}, \frac{ \Pst_h(s')}{\epsilon^2} \right \} }.
\end{align*}
Thus, since $\pitil$ is a feasible policy, using the expression for $\Compb(\cM,\epsilon)$ given in \Cref{prop:relate_complexities}, it follows that
\begin{align*}
\Compb(\cM,\epsilon) & = \sum_{h=1}^H \inf_\pi \max_s \frac{1}{\wpi_h(s)} \sum_{a}  \min \left \{ \frac{1}{ \Deltil_h(s,a)^2},  \frac{\Pst_h(s)^2}{\epsilon^2}   \right \} + \frac{H^2 | \opt(\epsilon)|}{\epsilon^2}\\
& \le \sum_{h=1}^H \sum_{s,a} \min \left \{ \frac{1}{\Pst_h(s) \Deltil_h(s,a)^2}, \frac{ \Pst_h(s)}{\epsilon^2} \right \} + \frac{H^2 | \opt(\epsilon)|}{\epsilon^2}.
\end{align*}

To obtain the first bound, we use the second bound to get
\begin{align*}
\Compb(\cM,\epsilon) \le \sum_{s,a,h} \frac{H^2 \Pst_h(s)}{\epsilon^2} \le \frac{H^3 SA}{\epsilon^2}
\end{align*}
and use that $| \Xbareps | \le SAH$. 
\end{proof}

\begin{proof}[Proof of \Cref{prop:complexity_bandit}]
This follows directly from \Cref{prop:relate_complexities}.
\end{proof}

\section{MDP Technical Results}
\begin{proof}[Proof of \Cref{lem:local_to_global_subopt}]
This result follows from the Performance-Difference Lemma. We give the full proof for completeness. The following is the standard proof of the Performance-Difference Lemma:
\begin{align*}
\Vst_0 - V_0^\pi & = \Exp_{\pist,s_0 \sim P_0} \left [ \sum_{h=1}^H r_h(s_h,a_h) \right ] - \Vpi_0 \\
& =  \Exp_{\pist,s_0 \sim P_0} \left [ \sum_{h=1}^H r_h(s_h,a_h) + \Vpi_h(s_h) - \Vpi_h(s_h) \right ] - \Vpi_0 \\ 
& =  \Exp_{\pist,s_0 \sim P_0} \left [ \sum_{h=1}^H r_h(s_h,a_h) + \Vpi_{h+1}(s_{h+1}) - \Vpi_h(s_h) \right ]  \\ 
& =  \Exp_{\pist,s_0 \sim P_0} \left [ \sum_{h=1}^H r_h(s_h,a_h) + \Exp[\Vpi_{h+1}(s') | s_h,a_h] - \Vpi_h(s_h) \right ]  \\
& =  \Exp_{\pist,s_0 \sim P_0} \left [ \sum_{h=1}^H \Qpi_h(s_h,a_h) - \Vpi_h(s_h) \right ]  \\
& = \sum_{h=1}^H \sum_{s,a} w_h^{\pist}(s,a) ( \Qpi_h(s,a) - \Vpi_h(s)) .
\end{align*}
In the case when $\pi$ is deterministic, we have $\Vpi_h(s) = \Qpi_h(s,\pi_h(s))$. Furthermore, we can upper bound the above by
\begin{align*}
& \sum_{h=1}^H \sum_{s,a} w_h^{\pist}(s,a) ( \max_{a'} \Qpi_h(s,a') - \Qpi_h(s,\pi_h(s)))  = \sum_{h=1}^H \sum_{s} w_h^{\pist}(s) \epsilon_h(s).
\end{align*}
The result follows by upper bounding the visitation under $\pist$ by the visitation under the worst-case policy. 
\end{proof}

\begin{lem}\label{lem:gap_diff}
Assume that
\begin{align*}
\sup_\pi \sum_{s'} \wpi_h(s') (\Vst_h(s') - \Vpihat_h(s')) \le \epsilon \quad \text{and} \quad \sup_\pi \sum_{s'} \wpi_{h+1}(s') ( \Vst_{h+1}(s') - \Vpihat_{h+1}(s')) \le \epsilon.
\end{align*}
Then, for any $s$,
\begin{align*}
| \Delta_h(s,a) - \Delpihat_h(s,a) | \le  \epsilon /\Pst_h(s).
\end{align*}
\end{lem}
\begin{proof}
By definition,
\begin{align*}
|\Delta_h(s,a)  - \Delpihat_h(s,a)| & = | \Vst_h(s) - \Qst_h(s,a) - (\max_{a'} \Qpihat_h(s,a') - \Qpihat_h(s,a)) | \\
& \le \max \{ | \Vst_h(s) - \max_{a'} \Qpihat_h(s,a')|, | \Qpihat_h(s,a) - \Qst_h(s,a)| \}.
\end{align*}
where the last inequality follows since
\begin{align*}
\Vst_h(s) - \Qst_h(s,a) - (\max_{a'} \Qpihat_h(s,a') - \Qpihat_h(s,a)) \le \Vst_h(s) - \max_{a'} \Qpihat_h(s,a')
\end{align*}
and 
\begin{align*}
-(\Vst_h(s) - \Qst_h(s,a) - (\max_{a'} \Qpihat_h(s,a') - \Qpihat_h(s,a))) \le \Qst_h(s,a) - \Qpihat_h(s,a).
\end{align*}
Now,
\begin{align*}
\Vst_h(s) - \max_{a'} \Qpihat_h(s,a') & = \Vst_h(s) - \Qpihat_h(s,\pihat_h(s)) + \Qpihat_h(s,\pihat_h(s)) - \max_{a'} \Qpihat_h(s,a') \\
& \le \Vst_h(s) - \Vpihat_h(s)
\end{align*}
where the inequality follows since, by definition, $\Vpihat_h(s) = \Qpihat_h(s,\pihat_h(s))$ and $\Qpihat_h(s,\pihat_h(s)) - \max_{a'} \Qpihat_h(s,a') \le 0$. By assumption,
\begin{align*}
\sup_\pi \sum_{s'} \wpi_h(s') (\Vst_h(s') - \Vpihat_h(s')) \le \epsilon
\end{align*}
and furthermore, for any $s$,
\begin{align*}
\sup_\pi \sum_{s'} \wpi_h(s') (\Vst_h(s') - \Vpihat_h(s'))  \ge \Pst_h(s) (\Vst_h(s) - \Vpihat_h(s))
\end{align*}
so it follows that $| \Vst_h(s) - \Vpihat_h(s) | \le \epsilon / \Pst_h(s)$. By definition,
\begin{align*}
\Qst_h(s,a) - \Qpihat_h(s,a) & = \sum_{s'} P_h(s'|s,a) ( \Vst_{h+1}(s') - \Vpihat_{h+1}(s')) 
\end{align*}
so
\begin{align*}
\Pst_h(s) ( \Qst_h(s,a) - \Qpihat_h(s,a)) & = \sum_{s'} P_h(s'|s,a) \Pst_h(s) ( \Vst_{h+1}(s') - \Vpihat_{h+1}(s')) \\
& \le \sup_\pi \sum_{s'} \wpi_{h+1}(s') ( \Vst_{h+1}(s') - \Vpihat_{h+1}(s'))
\end{align*}
where the inequality follows since $\Vst_{h+1}(s') \ge \Vpihat_{h+1}(s')$, and since
\begin{align*}
P_h(s'|s,a) \Pst_h(s) = \Pr[s_{h+1} = s' | s_h = s, a_h = a] \Pr_{\pi}[s_h = s] = \Pr_{\pi'}[s_{h+1} = s', s_h = s] \le \Pr_{\pi'}[s_{h+1}]
\end{align*}
where $\pi$ denotes the policy achieving $\Pr_{\pi}[s_h = s] = \Pst_h(s)$ and $\pi'$ plays $\pi$ up to $h$ and then $\pi_h'(s) = a$. Thus, if $ \sup_\pi \sum_{s'} \wpi_{h+1}(s') ( \Vst_{h+1}(s') - \Vpihat_{h+1}(s')) \le \epsilon$, rearranging the inequalities gives the result.

\end{proof}

We are aware of several works which obtain the following result for non-episodic MDPs \citep{zimin2013online,puterman2014markov}, but present the result for episodic MDPs for completeness.

\begin{prop}\label{prop:state_act_vis}
Fix some MDP $\cM$. Then:
\begin{enumerate}
\item The set of valid state-action visitation distributions on $\cM$ is convex.
\item For any valid state-action visitation distribution on $\cM$, there exists some policy which realizes it.
\end{enumerate}
\end{prop}
\begin{proof}
The set of valid state-action visitation distributions, $\cW$, is defined as
\begin{align*}
\cW := \Big \{ w \in [0,1]^{SAH} \ : \ \exists \pi & \in \Pi \text{ s.t. } w_h(s,a)  = \pi_h(a|s) \cdot \sum_{s',a'} P_{h-1}(s|s',a') w_{h-1}(s',a'), \forall h \ge 1, \\
w_0(s,a) & = \pi_0(a|s) P_0(s), \quad \sum_{s,a} w_h(s,a) = 1,  \forall h \ge 0 \Big \}
\end{align*}
where here $\Pi = \simplex(\cA)^{SH}$.

Fix some state-action visitation distributions $w, w' \in \cW$, and let $\pi$ and $\pi'$ denote their correponding policies as above. Furthermore, denote $w_h(s) = \sum_a w_h(s,a)$ (and similarly for $w'$). Our goal is to show that for any $t \in [0,1]$, $\wtil = (1-t) w + t w' \in \cW$. First, we show that there exists some policy $\pitil$ such that 
\begin{align*}
(1-t) w_0(s,a) + t w_0'(s,a) = \pitil_0(a|s) P_0(s) .
\end{align*}
Note that we can take $\pitil_0(a|s) = (1-t) \pi_0(a|s) + t \pi_0'(a|s)$, since
\begin{align*}
( (1-t) \pi_0(a|s) + t \pi_0'(a|s)) P_0(s) = (1-t) w_0(s,a) + t w_0'(s,a).
\end{align*}

By construction, for any $h \ge 1$,
\begin{align*}
\wtil_h(s) = \sum_a \wtil_h(s,a) = (1-t) \sum_a w_h(s,a) + t \sum_a w_h'(s,a) = (1-t) w_h(s) + t w_h'(s). 
\end{align*}
Furthermore, since $w$ is a valid state-action distribution,
\begin{align*}
w_h(s) = \sum_{s',a'} P_{h-1}(s|s',a') w_{h-1}(s',a')
\end{align*}
and similarly for $w'$. Let $\pitil_h(a|s) = \wtil_h(s,a) / \wtil_h(s)$ (where we define $0/0 = 0$), and note that this is a valid distribution since by definition $\sum_{a} \wtil_h(s,a) = \wtil_h(s)$. Then,
\begin{align*}
\wtil_h(s,a) & = \pitil_h(a|s) \wtil_h(s) \\
& = \pitil_h(a|s) ((1-t) w_h(s) + t w_h'(s)) \\
& = \pitil_h(a|s) \sum_{s',a'} P_{h-1}(s|s',a') ((1-t) w_{h-1}(s',a') + t w_{h-1}'(s',a')) \\
& = \pitil_h(a|s) \sum_{s',a'} P_{h-1}(s|s',a') \wtil_{h-1}(s',a')
\end{align*}
where the last equality follows by the definition of $\wtil_{h-1}$. The other constraints are trivial to verity, so $\wtil \in \cW$. This proves the first result.

For the second result, take some $w \in \cW$, and let $\pi_h(a|s)  = w_h(s,a) / w_h(s)$. By definition this is a valid distribution. Furthermore, it trivially holds that $\wpi_0(s,a) = w_0(s,a)$. Assume that $\wpi_{h-1}(s,a) = w_{h-1}(s,a)$ for all $(s,a)$. By definition and the inductive hypothesis,
\begin{align*}
\wpi_h(s,a) & = \pi_h(a|s) \sum_{s',a'} P_{h-1}(s|s',a') \wpi_{h-1}(s',a') \\
& = \pi_h(a|s) \sum_{s',a'} P_{h-1}(s|s',a') w_{h-1}(s,a) \\
& = \pi_h(a|s) w_h(s) \\
& = w_h(s,a),
\end{align*}
which proves the second result.
\end{proof}


\section{Proof of \Cref{thm:complexity}}\label{sec:detailed_proof}

In this section we give a formal proof of \Cref{thm:complexity}. 

\paragraph{Notation.}
Throughout the proof, we let $\epsout$ denote the tolerance and $\delout$ the confidence given as an input to \algname, and $\epstil = \epsoutm$ and $\deltil = \deloutm$ the tolerance and confidence given as an input to \mcae\xspace at epoch $m$ of \algname, respectively. For convenience, we will also define $\epstil_0 = H$. For a single call of \mcae, we will use the following notation:
\begin{itemize}
\item For a given $h$, $i$, and $\ell$, consider the call to \collectsamp\xspace on \Cref{line:collectsamp1}, and let $\{ \cX_{hij}^\ell \}_{j = 1}^{\iotaepstil}$ denote the partition returned by calling \sap\xspace on \Cref{line:get_part} of \collectsamp. Similarly, let $\{ \Pi_{hij}^\ell \}_{j = 1}^{\iotaepstil}$ and $\{ N_{hij}^\ell \}_{j = 1}^{\iotaepstil}$ denote the policies and minimum number of samples returned by \sap, respectively. 
\item For a given $h$, consider the call to \collectsamp\xspace on \Cref{line:sap_final}, and let $\{ \cX_{hj}^{\numepochs}\}_{j=1}^{\iotaepstil}$ denote the partition returned by calling \sap\xspace on \Cref{line:get_part} of \collectsamp. As before, let $\{ \Pi_{hj}^{\numepochs+1} \}_{j = 1}^{\iotaepstil}$ and $\{ N_{hj}^{\numepochs+1} \}_{j = 1}^{\iotaepstil}$ denote the policies and minimum number of samples.
\end{itemize}

\paragraph{Good Events.}
We next define the good events, which we will assume hold throughout the remainder of the proof.

First, let $\Eexplore$ be the event on which, for all calls to \mcae\xspace simultaneously:
\begin{itemize}
\item For every $h = 1,\ldots,H$, $i = 1, \ldots, \iotaepstil$, $\ell = 1,\ldots,\numepochs$, we collect at least $n_{i1}^\ell$ samples from each $(s,a) \in \Xgood_{hi}^\ell$. Furthermore, $\cup_{j=1}^{\iotaepstil} \cX_{hij}^\ell = \Xgood_{hi}^\ell$ and $\cX_{hij}^\ell$ satisfy
\begin{align*}
\sup_\pi \sum_{(s,a) \in \cX_{hij}^\ell} \wpi_h(s,a) \le 2^{-j+1}.
\end{align*}
\item For every $h = 1,\ldots,H$, if \mcae\xspace is run with \finalround\xspace = \true, then we collect at least $\nlast_j$ samples from each $(s,a) \in \Xlast_{hj}$. Furthermore, $\cup_{j=1}^{\iotaepstil} \Xlast_{hj} = \Xgood_h^{\numepochs+1}$ and $\Xlast_{hj}$ satisfies
\begin{align*}
\sup_\pi \sum_{(s,a) \in \Xlast_{hj}} \wpi_h(s,a) \le 2^{-j+1}.
\end{align*} 
\item $\Phatst_h(s) \le \Pst_h(s) \le 32 \Phatst_h(s)$ for all $s \in \Xgood_h$.
\item Following \Cref{line:main_loop} of \mcae, $\Xgood_h$ satisfies, for all $h$,
\begin{align*}
\sup_\pi  \max_{s \in \Xgood_h^c} \wpi_h(s) \le \frac{\epsilon}{2H^2S}.
\end{align*}
\end{itemize}
Next, let $\Eest$ be the event on which, for all calls to \mcae,
\begin{align*}
& | \Qhatpihat_{h,\ell}(s,a) - \Qpihat_h(s,a)| \le \sqrt{\frac{H^2 \iotadeltil}{N_h^{hi\ell}(s,a)}}, \quad \forall (s,a) \in \Xgood_{hi}^\ell, \forall h \in [H],i \in [\iotaepstil],\ell \in [\numepochs] \\
& | \Qhatpihat_{h,\numepochs+1}(s,a) - \Qpihat_h(s,a)| \le \sqrt{\frac{H^2 \iotadeltil}{N_h^{h(\numepochs+1)}(s,a)}}, \quad \forall (s,a) \in \Xgood_{h}^{\numepochs}, \forall h \in [H]
\end{align*}
where $\Qhatpihat_{h,\ell}(s,a)$ is the estimate of $\Qpihat_h(s,a)$ formed on \Cref{line:collect_Qhat} of \collectdata, $N_h^{hi\ell}(s,a)$ is the number of samples collected from $(s,a,h)$ at iteration $(h,i,\ell)$, and $\Qhatpihat_{h,\numepochs+1}(s,a)$ and $N_h^{h(\numepochs+1)}(s,a)$ are the analogous quantities for the sampling done if \finalround\xspace = \true. 

We can think of $\Eexplore$ as the event on which we \emph{explore} successfully---we reach every state the desired number of times---and $\Eest$ the event on which we \emph{estimate} correctly---our Monte Carlo estimates of $\Qpihat_h(s,a)$ concentrate. The following lemma shows that these events hold with high probability.

\begin{lem}\label{lem:good_events_hold}
If we run \algname, $\Pr[\Eexplore \cap \Eest] \ge 1-\delout$.
\end{lem}
\begin{proof}[Proof Sketch]
That $\Eest$ holds is simply a consequence of Hoeffding's inequality since $\Qpihat_h(s,a)$ will be in $[0,H]$ almost surely. That $\Eexplore$ holds is a direct consequence of the correctness of our exploration procedure, as described in \Cref{sec:learn2explore}. We give the full proof of this result in \Cref{sec:good_event_holds}.
\end{proof}

\subsection{Correctness of \mcae.}\label{sec:correctness}
We next establish that the policy returned by \mcae\xspace run with tolerance $\epstil$ and \finalround\xspace = \true\xspace is $\epstil$-optimal. To this end, we first show that any action in the active set, $\frakA_h^\ell(s)$, will satisfy a certain suboptimality bound.

\begin{lem}[Formal Statement of \Cref{lem:action_subopt_informal} and \Cref{lem:fr_good_actions}] \label{lem:action_subopt}
On the event $\Eest \cap \Eexplore$, if \mcae\xspace is run with tolerance $\epstil$, for any $h \in [H]$ and $\ell \in [\numepochs+1]$, if $| \frakA_h^\ell(s)| = 1$, then for $a \in \frakA_h^\ell(s)$,
\begin{align*}
\max_{a'} \Qpihat_h(s,a') - \Qpihat_h(s,a) = 0.
\end{align*}
Furthermore, if $| \frakA_h^\ell(s)| > 1$, $\ell \le \numepochs$, and $s \in \Xgood_{hi}$ for some $i$, then any $a \in \frakA_h^\ell(s)$ satisfies
\begin{align*}
\Delta_h(s,a) \le \frac{3\epstil_\ell}{2 \Pst_h(s)}.
\end{align*}
Finally, if $| \frakA_h^{\numepochs+1}(s)| > 1$ and $s \in \Xgood_h$, then any $a \in \frakA_h^{\numepochs+1}(s)$ satisfies
\begin{align*}
\Delpihat_h(s,a) \le \frac{ \epstil}{2 H \iotaepstil \cdot 2^{-j(s)+1} }
\end{align*}
where $j(s) = \argmax_j j \ \text{s.t.} \ \exists a', (s,a') \in \Xlast_{hj}$.
\end{lem}
\begin{proof}
We first claim that the optimal action with respect to $\pihat$ must always be active.

\begin{claim}\label{lem:optactionsin}
On the event $\Eest \cap \Eexplore$, for any $h$, $s$, and $\ell \in [\numepochs+1]$, we will have that $\ahatst_h(s) \in \frakA_h^\ell(s)$ where $\ahatst_h(s) = \argmax_a \Qpihat_h(s,a)$.
\end{claim}

We prove this claim in \Cref{sec:good_event_holds}. By construction, we will always have that $| \frakA_h^{\ell}(s)| \ge 1$. If $|\frakA_h^{\ell}(s)|=1$, from \Cref{lem:optactionsin} it follows that $\frakA_h^{\ell}(s) = \{ \ahatst_h(s) \}$, and thus $\max_{a'} \Qpihat_h(s,a') - \Qpihat_h(s,a) = 0$.

Assume then that $|\frakA_h^{\ell}(s)| > 1$, $\ell \le \numepochs$, and $s \in \Xgood_{hi}$. The result is trivial when $\ell = 0$, since in this case $\epstil_\ell = H$, and we will always have $\Delta_h(s,a) \le H, \Pst_h(s) \le 1$. On the event $\Eexplore$, for all $i \in [\iotaepstil]$ we will collect at least $n_{i1}^\ell = 2^{18} \cdot 2^{-2i} H^2 \iotadeltil/\epstil_\ell^2$ samples from $(s,a)$ for each $a \in \frakA_h^\ell(s)$, and on $\Eest$ we will then have that 
\begin{align*}
| \Qhatpihat_{h,\ell}(s,a) - \Qpihat_h(s,a)| \le \sqrt{\frac{H^2 \iotadeltil}{n_{i1}^\ell}} = 2^i \epstil_\ell /2^9.
\end{align*}
Thus, for any $a \in \frakA_h^\ell(s)$, we have
\begin{align*}
\max_{a' \in \frakA_h^{\ell}(s)} \Qhatpihat_{h,\ell}(s,a') - \Qhatpihat_{h,\ell}(s,a) & \ge \max_{a' \in \frakA_h^{\ell}(s)} \Qpihat_{h}(s,a') - \Qpihat_{h}(s,a) - 2 \cdot 2^i \epstil_\ell /2^9 \\
& =  \max_{a'} \Qpihat_{h}(s,a') - \Qpihat_{h}(s,a) - 2 \cdot 2^i \epstil_\ell /2^9
\end{align*}
where the equality follows since $\ahatst_h(s) \in \frakA_h^{\ell}(s)$. It follows that if
\begin{align*}
\Delpihat_h(s,a) = \max_{a'} \Qpihat_{h}(s,a') - \Qpihat_{h}(s,a) \ge 4 \cdot 2^i \epstil_\ell /2^9
\end{align*}
then
\begin{align*}
\max_{a' \in \frakA_h^{\ell}(s)} \Qhatpihat_{h,\ell}(s,a') - \Qhatpihat_{h,\ell}(s,a)  \ge 2 \cdot 2^i \epstil_\ell /2^9.
\end{align*}
so the exit condition on \Cref{line:collect_Aell} for \collectdata\xspace is met for our choice of $\gamma_{ij}^\ell = 2^i \epstil_\ell /2^8$ (note that in this case, since $\gamma_{ij}^\ell$ is the same for all $\ell$, \Cref{line:set_js} has no effect), and therefore $a \not\in \frakA_h^{\ell+1}(s)$. Thus, any $a \in \frakA_h^{\ell+1}(s)$ must satisfy
\begin{align*}
\Delpihat_h(s,a) \le 2^i \epstil_\ell /2^7.
\end{align*}
By construction, we will have that $\Phat_h(s) \in [2^{-i},2^{-i+1}]$ and on $\Eexplore$, $\Phatst_h(s) \le \Pst_h(s) \le 32 \Phatst_h(s)$. Thus, we can upper bound
\begin{align*}
\Delpihat_h(s,a) \le  2^i \epstil_\ell /2^7 \le  \frac{2 \epstil_\ell}{\Phat_h(s) 2^7} \le  \frac{32 \cdot 2 \epstil_\ell}{\Pst_h(s) 2^7} = \frac{\epstil_\ell}{2 \Pst_h(s)}.
\end{align*}
Finally, the following claim, proved in \Cref{sec:good_event_holds}, allows us to relate $\Delpihat_h(s,a)$ to $\Delta_h(s,a)$:
\begin{claim}\label{lem:gap_lower_bound}
On the event $\Eest \cap \Eexplore$, for any $(s,a,h)$, we will have $| \Delpihat_h(s,a) - \Delta_h(s,a) | \le  \epstil / \Pst_h(s)$.
\end{claim}

Applying \Cref{lem:gap_lower_bound}, we can lower bound $\Delpihat_h(s,a) \ge \Delta_h(s,a) - \epstil/\Pst_h(s) \ge \Delta_h(s,a) - \epstil_\ell/\Pst_h(s)$. Rearranging this gives the second conclusion.

The argument for the third conclusion is similar to the preceding argument. However, we now have the extra subtlety that for $a \neq a'$ with $a,a' \in \frakA_h^{\numepochs}(s)$, we may collect a different number of samples from $(s,a)$ and $(s,a')$ since it's possible that $(s,a) \in \Xlast_{hj}$ and $(s,a') \in \Xlast_{hj'}$ for $j \neq j'$. Denote 
\begin{align*}
j(s) = \argmax_j j \quad \text{s.t.} \quad \exists a, (s,a) \in \Xlast_{hj}.
\end{align*}
Note that, on $\Eexplore$, we are guaranteed that there exists some $j \in [\iotaepstil]$ such that $(s,a) \in \Xlast_{hj}$ so $j(s)$ is always well-defined. We can repeat the above argument, but now we can only guarantee that
\begin{align*}
| \Qhatpihat_{h,\numepochs+1}(s,a) - \Qpihat_h(s,a)| \le \sqrt{\frac{H^2 \iotadeltil}{n_{j(s)}^{\numepochs+1}}} = \frac{\epsilon}{8 H \iotaepstil 2^{-j(s)+1}}.
\end{align*}
since we can only guarantee we collect $n_{j(s)}^{\numepochs+1}$ samples from each $(s,a), a \in \frakA_h^{\numepochs}(s)$. It again follows that if
\begin{align*}
\Delpihat_h(s,a) \ge 4 \cdot \frac{\epsilon}{8 H \iotaepstil 2^{-j(s)+1}}
\end{align*}
then 
\begin{align*}
\max_{a' \in \frakA_h^{\numepochs}(s)} \Qhatpihat_{h,\numepochs+1}(s,a') - \Qhatpihat_{h,\numepochs+1}(s,a) \ge 2 \cdot \frac{\epsilon}{8 H \iotaepstil 2^{-j(s)+1}}.
\end{align*}
As this is precisely the elimination criteria used in \collectdata, it follows that $a$ will be eliminated. Thus, all $a \in \frakA_h^{\numepochs+1}(s)$ must satisfy
\begin{align*}
\Delpihat_h(s,a) \le 4 \cdot \frac{\epsilon}{8 H \iotaepstil 2^{-j(s)+1}}
\end{align*}
which gives the third conclusion.

\end{proof}

\Cref{lem:action_subopt} and the definition of $\Eexplore$ then let us prove that \algname returns an $\epstil$-optimal policy. 

\newcommand{\Xtil}{\widetilde{\cX}}
\begin{lem}[Formal Statement of \Cref{lem:mcae_correct2}]\label{lem:correct}
On the event $\Eest \cap \Eexplore$, if \mcae\xspace is run with tolerance $\epstil$ and \finalround\xspace = \texttt{true}, then the policy $\pihat$ returned by \mcae\xspace is $\epstil$-suboptimal.
\end{lem}
\begin{proof}
 \Cref{lem:local_to_global_subopt} gives that, if $\pihat$ satisfies $\max_a \Qpihat_h(s,a) - \Qpihat_h(s,\pihat_h(s)) \le \epstil_h(s)$ for all $h$ and $s$, then $\pihat$ is at most
\begin{align}\label{eq:policy_subopt}
\sum_{h=1}^H \sup_\pi \sum_s \wpi_h(s) \epstil_h(s) 
\end{align}
suboptimal. When running \Cref{alg:mcae2}, for a particular $h$ every state $s$ can be classified in one of three ways:
\begin{itemize}
\item $s \not\in \Xgood_h$: In this case, on $\Eexplore$ we will have $\sup_\pi \wpi_h(s) \le \epstil/(2H^2S)$ and $\epstil_h(s) \le H$.
\item $s \in \Xgood_h$ and $| \frakA_h^{\numepochs+1}(s)| = 1$: In this case, by \Cref{lem:action_subopt}, since $\pihat$ only takes actions that are in $\frakA_h^{\numepochs+1}(s)$, we will have $\epstil_h(s) = \max_{a} \Qpihat_h(s,a) - \Qpihat_h(s,\pihat_h(s)) = 0$.
\item $s \in \Xgood_h$, $| \frakA_h^{\numepochs+1}(s)| > 1$:  Then we can apply \Cref{lem:action_subopt} to get 
\begin{align*}
\epstil_h(s) = \max_{a'} \Qpihat_h(s,a') - \Qpihat_h(s,\pihat_h(s)) \le \frac{\epstil}{2H\iotaepstil \cdot 2^{-j(s)+1}}
\end{align*}
\end{itemize}
Let $\Xtil_j = \{ s \ : \ j(s) = j \}$ and note that $\{ s \in \Xgood_h \ : \ | \frakA_h^{\numepochs+1}(s) | > 1 \} \subseteq \cup_{j=1}^{\iotaepstil} \Xtil_j$ since, on $\Eexplore$, for every $s$ satisfying $s \in \Xgood_h, | \frakA_h^{\numepochs+1}(s) | > 1$, we will have $(s,a) \in \Xgood_h^{\numepochs+1}$ for some $a$, so we must have that $(s,a) \in \Xlast_{hj}$ for some $j \in [\iotaepstil]$. Furthermore, by definition of $j(s)$, if $s \in \Xtil_j$, then $(s,a) \in \Xlast_{hj}$ for some $a$. Then, plugging all of this into \Cref{eq:policy_subopt}, on $\Eexplore$, 
\begin{align*}
\sum_{h=1}^H \sup_\pi \sum_s \wpi_h(s) \epstil_h(s) & \le \sum_{h=1}^H \sup_\pi \sum_{j=1}^{\iotaepstil} \sum_{s \in \Xtil_j} \wpi_h(s) \epstil_h(s) + H \sum_{h=1}^H \sup_\pi \sum_{s \in \Xgood_h^c} \wpi_h(s) \\
& \le \frac{\epsilon}{2H \iotaepstil} \sum_{h=1}^H \sup_\pi\sum_{j=1}^{\iotaepstil} \sum_{s \in \Xtil_j} \wpi_h(s) 2^{j(s)-1} + H \sum_{h=1}^H \sup_\pi \sum_{s \in \Xgood_h^c} \wpi_h(s)  \\
& \overset{(a)}{\le} \frac{\epsilon}{2H \iotaepstil} \sum_{h=1}^H \sum_{j=1}^{\iotaepstil}  2^{j-1} \sup_\pi \sum_{(s,a) \in \Xlast_{hj}} \wpi_h(s,a) + H \sum_{h=1}^H \sup_\pi \sum_{s \in \Xgood_h^c} \wpi_h(s) \\
& \le \frac{\epsilon}{2H \iotaepstil} \sum_{h=1}^H \sum_{j=1}^{\iotaepstil}  2^{j-1} 2^{-j+1} + H \sum_{h=1}^H  \sum_{s \in \Xgood_h^c} \frac{\epsilon}{2H^2 S} \\
& \le \epsilon
\end{align*}
where $(a)$ holds since for $s \in \Xtil_j$, $j(s) = j$, and since we can always choose $\pi$ so that $\pi_h(s) = a$ so $\wpi_h(s,a) = \wpi_h(s)$. It follows that $\pihat$ is at most $\epsilon$-suboptimal. 
\end{proof}

\subsection{Sample Complexity}
We turn now to establishing a bound on the sample complexity of \algname. We first bound the complexity of a \emph{single} call to \collectsamp.

\begin{lem}\label{lem:collectdata_complexity}
\collectsamp($\Xgood_{hi}^\ell,\{ n_{ij}^\ell \}_{j=1}^{\iotaepstil},h,\pihat,\tfrac{\deltil}{H \iotaepstil \numepochs},\tfrac{\epsexp}{32}$) terminates in at most
\begin{align*}
\frac{c H^2 \iotadeltil \iotaepstil}{\epstil_\ell^2} \sum_{j=1}^{\iotaepstil} 2^{j} \sum_{(s,a) \in \cX_{hij}^\ell } \Pst_h(s)^2 + \frac{\poly(S,A,H, \log 1/\deltil, \log 1/\epstil)}{\epstil} 
\end{align*}
episodes and \collectsamp($\Xgood_h^{\numepochs+1}, \{\nlast_j \}_{j=1}^{\iotaepstil},h,\pihat,\tfrac{\deltil}{H},\tfrac{\epsexp}{32}$)  terminates in at most
\begin{align*}
\frac{c H^4 \iotadeltil \iotaepstil^2}{\epstil^2}  | \Xgood_h^{\numepochs+1}|  + \frac{\poly(S,A,H, \log 1/\deltil, \log 1/\epstil)}{\epstil} 
\end{align*}
episodes.
\end{lem}
\begin{proof}
Recall that $\epsexp = \frac{\epsilon}{2H^2 S}$. The complexity of \collectsamp($\Xgood_{hi}^\ell,n_i^\ell,h,\pihat,\tfrac{\deltil}{H \iotaepstil \numepochs},\tfrac{\epstil}{64H^2 S}$) can be bounded by the sum of the complexity of calling \sap\xspace to learn a set of exploration policies, and the complexity of playing these policies to collect samples. By \Cref{thm:partitioning_works}, we can bound the complexity of calling \sap\xspace by
\begin{align*}
C_K(\tfrac{\deltil}{H \iotaepstil \numepochs}, \delsamp,\iotaepstil) \frac{256H^2 S}{\epstil}
\end{align*}
where $\delsamp = \tfrac{\deltil}{H \iotaepstil \numepochs} \cdot \tfrac{1}{\iotaepstil \max_j n_{ij}^\ell} \le \frac{\deltil  \epstil_\ell^2}{2^{17} H^3 \iotadeltil \iotaepstil^2 \numepochs}$. As shown in \Cref{sec:learn2explore}, $C_K(\tfrac{\deltil}{H \iotaepstil \numepochs}, \delsamp,\iotaepstil)$ is $\poly(S,A,H,\log 1/\epstil, \log 1/\deltil)$, so this entire term is $\frac{\poly(S,A,H,\log 1/\epstil, \log 1/\deltil)}{\epstil}$.

Since rerunning the policies in $\Pi_{hij}^\ell$ yields at least $N_{hij}^\ell/2$ samples from each $(s,a)$ in $X_{hij}^\ell$, if we desire $n_{ij}^\ell$ samples from each $(s,a)$, the complexity of running the policies returned by \sap\xspace in order to collect the desired samples is clearly given by 
\begin{align*}
\sum_{j=1}^{\iotaepstil} | \Pi_{hij}^\ell |  \lceil 2 n_{ij}^\ell/N_{hij}^\ell \rceil.
\end{align*}
By the construction of $\Pi_{hij}^\ell$ and definition of $N_{hij}^\ell$ given in \sap\xspace, we have that
\begin{align*}
| \Pi_{hij}^\ell | = 2^j C_K(\tfrac{\deltil}{H \iotaepstil \numepochs}, \delsamp,j), \quad N_{hij}^\ell = \frac{| \Pi_{hij}^\ell |}{4 M_{hij}^\ell 2^j}.
\end{align*}
where $M_{hij}^\ell = \sum_{j'=j}^{\iotaepstil+1} | \cX_{hij'}^\ell |$ and $\cX_{hi(\iotaepstil+1)}^\ell = \Xgood_{hi}^\ell \backslash \cup_{j=1}^{\iotaepstil}  \cX_{hij}^\ell$. As we are on $\Eexplore$,  $ \Xgood_{hi}^\ell = \cup_{j=1}^{\iotaepstil}  \cX_{hij}^\ell$, so $| \cX_{hi(\iotaepstil+1)}^\ell| = 0$. It follows that the complexity can be upper bounded as
\begin{align*}
\sum_{j=1}^{\iotaepstil} | \Pi_{hij}^\ell |  \lceil 2 n_{ij}^\ell/N_{hij}^\ell \rceil & \le 8 \sum_{j=1}^{\iotaepstil} 2^j M_{hij}^\ell n_{ij}^\ell + \sum_{j=1}^{\iotaepstil} 2^j C_K(\tfrac{\deltil}{H \iotaepstil \numepochs},\delsamp,j) \\
& \le 8 \sum_{j=1}^{\iotaepstil} 2^j M_{hij}^\ell n_{ij}^\ell +   2^{\iotaepstil+1} C_K(\tfrac{\deltil}{H \iotaepstil \numepochs}, \delsamp,\iotaepstil) \\
& = 8 \frac{2^{17} H^2 \iotadeltil}{2^{2i} \epstil_\ell^2} \sum_{j=1}^{\iotaepstil} 2^j M_{hij}^\ell  +  2^{\iotaepstil+1} C_K(\tfrac{\deltil}{H \iotaepstil \numepochs}, \delsamp,\iotaepstil)
\end{align*}
The term $2^{\iotaepstil+1} C_K(\tfrac{\deltil}{H \iotaepstil \numepochs}, \delsamp,\iotaepstil)$ is $\frac{\poly(S,A,H,\log 1/\epstil, \log 1/\deltil)}{\epstil}$ by definition of $\iotaepstil$ and $C_K$. Furthermore,
\begin{align*}
\sum_{j=1}^{\iotaepstil} 2^j M_{hij}^\ell = \sum_{j=1}^{\iotaepstil} 2^j \sum_{j'=j}^{\iotaepstil} | \cX_{hij'}^\ell | \le \iotaepstil \sum_{j=1}^{\iotaepstil} 2^j | \cX_{hij}^\ell | .
\end{align*}
 We can therefore bound
\begin{align*}
 \frac{2^{17} H^2 \iotadeltil}{2^{2i} \epstil_\ell^2} \sum_{j=1}^{\iotaepstil} 2^j M_{hij}^\ell  & \le \frac{c H^2 \iotadeltil \iotaepstil}{\epstil_\ell^2} \sum_{j=1}^{\iotaepstil} 2^{j-2i} | \cX_{hij}^\ell | .
\end{align*}
Finally, using that on $\Eexplore$ $\Pst_h(s) \ge \Phat_h(s)$, and that all $(s,a) \in \cX_{hij}^\ell$ have a value of $\Phat_h(s)$ within a factor of 2 of every other, we can upper bound $2^{-i} \le 4 \Pst_h(s)$ for any $(s,a) \in \cX_{hij}^\ell$. This completes the proof of the first claim. 

The second claim follows similarly. By the same argument as above, we can upper bound the sample complexity of calling \collectsamp($\Xgood_h^{\numepochs+1}, \{\nlast_j \}_{j=1}^{\iotaepstil},h,\pihat,\tfrac{\deltil}{H},\tfrac{\epsexp}{32}$) as
\begin{align*}
\sum_{j=1}^{\iotaepstil} | \Pi_{hj}^{\numepochs+1}| & \lceil 2 \nlast_j / N_{hj}^{\numepochs+1} \rceil + \frac{\poly(S,A,H,\log 1/\epstil, \log 1/\delta)}{\epstil} \\
& \le 8 \sum_{j=1}^{\iotaepstil} 2^j M_{hj}^{\numepochs+1}  \nlast_j  + \frac{\poly(S,A,H,\log 1/\epstil, \log 1/\delta)}{\epstil} \\
& \overset{(a)}{=} \tfrac{c H^4 \iotadelmoca \iotaepsmoca^2 }{\epsmoca^2} \sum_{j=1}^{\iotaepstil} 2^{-j}  M_{hj}^{\numepochs+1}    + \frac{\poly(S,A,H,\log 1/\epstil, \log 1/\delta)}{\epstil} \\
& \overset{(b)}{\le} \tfrac{c H^4 \iotadelmoca \iotaepsmoca^2 }{\epsmoca^2} | \Xgood_h^{\numepochs+1}|   + \frac{\poly(S,A,H,\log 1/\epstil, \log 1/\delta)}{\epstil}
\end{align*}
where $(a)$ follows by our setting of $\nlast_j$ and $(b)$ follows since $M_{hj}^{\numepochs+1} \le  | \Xgood_h^{\numepochs+1}|$. The second conclusion follows.
\end{proof}

Using this, we show our main sample complexity lemma. 
\begin{lem}[Formal Statement of \Cref{lem:complexity1}] \label{lem:collect_bound2}
On the event $\Eest \cap \Eexplore$, for a given $h$ and $i$, the loop over $\ell$ on \Cref{line:ell_loop} of \mcae\xspace will take at most
\begin{align*}
c H^2 \iotadeltil \iotaepstil^2 \numepochs \inf_\pi   \max_{s \in \Xgood_{hi}} \max_a \min \left \{ \frac{1}{\wpi_h(s,a) \Deltil_h(s,a)^2}, \frac{\Pst_h(s)^2}{\wpi_h(s,a) \epstil^2} \right \}
\end{align*}
episodes. Furthermore, the total complexity of calling \mcae\xspace with \finalround\xspace = \false \xspace is bounded by:
\begin{align*}
H^2 c \iotadeltil \iotaepstil^3 \numepochs \cdot \sum_{h=1}^H \inf_{\pi} \max_{s,a} \min \left \{  \frac{1}{\wpi_h(s,a) \Deltil_h(s,a)^2}, \frac{\Pst_h(s)^2}{\wpi_h(s,a) \epstil^2} \right \} + \frac{\poly(S,A,H,\log 1/\epstil, \log 1/\deltil)}{\epstil}
\end{align*}
for a universal constant $c$. 
\end{lem}
\begin{proof}
With \finalround\xspace = \false \xspace, the complexity of \mcae\xspace is given by the complexity incurred calling \sap\xspace on \Cref{line:get_phat} and calling \collectsamp\xspace on \Cref{line:collectsamp1}. By \Cref{thm:partitioning_works} and since we call \sap\xspace at most $SH$ times, we can bound the complexity of calling \sap\xspace by
\begin{align*}
\frac{\poly(S,A,H,\log 1/\epstil, \log 1/\deltil)}{\epstil}.
\end{align*}
Next, we turn to upper bounding the sample complexity of  \sap\xspace. We can lower bound
\begin{align*}
| \cX_{hij}^\ell | \sup_\pi \min_{(s,a) \in \cX_{hij}^\ell} \wpi_h(s,a)  \le \sup_\pi \sum_{(s,a) \in \cX_{hij}^\ell} \wpi_h(s,a).
\end{align*}
so, on $\Eexplore$, $2^{j} \le 2 (| \cX_{hij}^\ell | \sup_\pi \min_{(s,a) \in \cX_{hij}^\ell} \wpi_h(s,a))^{-1} $. Plugging this into the bound given in \Cref{lem:collectdata_complexity}, we can bound the leading term in the sample complexity of a single call to \collectsamp\xspace as
\begin{align*}
\frac{c H^2 \iotadeltil \iotaepstil}{\epstil_\ell^2} \sum_{j=1}^{\iotaepstil} 2^{j} \sum_{(s,a) \in \cX_{hij}^\ell } \Pst_h(s)^2 & \le \frac{c H^2 \iotadeltil \iotaepstil}{\epstil_\ell^2} \sum_{j=1}^{\iotaepstil} \frac{1}{| \cX_{hij}^\ell | \sup_\pi \min_{(s,a) \in \cX_{hij}^\ell} \wpi_h(s,a)} \sum_{(s,a) \in \cX_{hij}^\ell } \Pst_h(s)^2 \\
& \overset{(a)}{\le} \frac{c H^2 \iotadeltil \iotaepstil}{\epstil_\ell^2} \sum_{j=1}^{\iotaepstil} \inf_\pi \max_{(s,a) \in \cX_{hij}^\ell} \frac{\Pst_h(s)^2}{\wpi_h(s,a)} \\
& \le \frac{c H^2 \iotadeltil \iotaepstil^2}{\epstil_\ell^2} \inf_\pi  \max_{j \in \{ 1,\ldots,\iotaepstil \}} \max_{(s,a) \in \cX_{hij}^\ell} \frac{\Pst_h(s)^2}{\wpi_h(s,a)} 
\end{align*}
where $(a)$ holds since all $s \in \cX_{hij}^\ell$ have values of $\Phat_h(s)$ within a constant factor of each other, and since on $\Eexplore$ $\Phat_h(s) \le \Pst_h(s) \le 32 \Phat_h(s)$, which together imply that
\begin{align*}
\max_{s \in \cX_{hij}^\ell} \Pst_h(s) \le c \min_{s \in \cX_{hij}^\ell} \Pst_h(s).
\end{align*}
If $(s,a) \in \cX_{hij}^\ell$, then we must have that $(s,a) \in \Xgood_{hi}^\ell$ since $\cX_{hij}^\ell \subseteq \Xgood_{hi}^\ell$, and, by the definition of $\Xgood_{hi}^\ell$, $a \in \frakA_h^{\ell-1}(s)$ and $| \frakA_h^{\ell-1}(s)| > 1$. \Cref{lem:action_subopt} gives that any $a \in \frakA_h^{\ell-1}(s)$ satisfies $\Delta_h(s,a) \le 3\epstil_{\ell-1}/(2 \Pst_h(s))$. Since $| \frakA_h^{\ell-1}(s)| > 1$, it follows there exists $a,a'$, $a \neq a'$, such that
\begin{align*}
\Delta_h(s,a) \le 3\epstil_{\ell-1}/(2 \Pst_h(s)) \quad \text{and} \quad \Delta_h(s,a') \le 3\epstil_{\ell-1}/(2 \Pst_h(s)).
\end{align*}
Thus, if $(s,a) \in \cX_{hij}^\ell$, $\frac{1}{4 \epstil_\ell^2} = \frac{1}{\epstil_{\ell-1}^2} \le \frac{9}{4 \Pst_h(s)^2 \Delta_h(s,a)^2}$ and $\frac{1}{4 \epstil_\ell^2} = \frac{1}{\epstil_{\ell-1}^2} \le \frac{9}{4 \Pst_h(s)^2 \Delta_h(s,a')^2}$, which implies $\frac{1}{4 \epstil_\ell^2} \le \frac{9}{4 \Pst_h(s)^2 \max \{\Delta_h(s,a)^2, \Delta_h(s,a')^2\}}$. Note that $\max \{\Delta_h(s,a)^2, \Delta_h(s,a')^2\} \ge \Deltil_h(s,a)^2$ since if $\Delta_h(s,a) = 0$, we will have $\max \{\Delta_h(s,a)^2, \Delta_h(s,a')^2\} = \Delta_h(s,a')^2$, so either $a$ is the unique optimal action at $(s,h)$, in which case $\Delta_h(s,a') \ge \Delmin(s,h) = \Deltil_h(s,a)$, or there are multiple optimal actions, in which case $\Delta_h(s,a') \ge 0 = \Deltil_h(s,a)$.
Thus,
\begin{align*}
\frac{c H^2 \iotadeltil \iotaepstil^2}{\epstil_\ell^2}  \inf_\pi &  \max_{j \in \{ 1,\ldots,\iotaepstil \}} \max_{(s,a) \in \cX_{hij}^\ell} \frac{\Pst_h(s)^2}{\wpi_h(s,a)} \\
& \le c H^2 \iotadeltil \iotaepstil^2  \inf_\pi  \max_{j \in \{ 1,\ldots,\iotaepstil \}} \max_{(s,a) \in \cX_{hij}^\ell} \min \left \{ \frac{1}{\wpi_h(s,a) \Deltil_h(s,a)^2}, \frac{\Pst_h(s)^2}{\wpi_h(s,a) \epsilon_\ell^2} \right \}  \\
& \le c H^2 \iotadeltil \iotaepstil^2 \inf_\pi   \max_{(s,a) \in \Xgood_{hi}^\ell} \min \left \{  \frac{1}{\wpi_h(s,a) \Deltil_h(s,a)^2} , \frac{\Pst_h(s)^2}{\wpi_h(s,a) \epstil_\ell^2}\right \}.
\end{align*}
Summing over $\ell$ and using that for all $(s,a) \in \Xgood_{hi}^\ell$, $s \in \Xgood_{hi}$, proves the first conclusion.
Summing over $i$, and $h$ gives
\begin{align*}
& \sum_{h=1}^H \sum_{i=1}^{\iotaepstil}  c H^2 \iotadeltil \iotaepstil^2 \numepochs \inf_\pi   \max_{s \in \Xgood_{hi}, a} \min \left \{  \frac{1}{\wpi_h(s,a) \Deltil_h(s,a)^2}, \frac{\Pst_h(s)^2}{\wpi_h(s,a) \epstil^2} \right \} \\
& \qquad \le c H^2 \iotadeltil \iotaepstil^3  \numepochs \sum_{h=1}^H \inf_{\pi} \max_{s,a} \min \left \{  \frac{1}{\wpi_h(s,a) \Deltil_h(s,a)^2}, \frac{\Pst_h(s)^2}{\wpi_h(s,a) \epstil^2} \right \}.
\end{align*}
This proves the result.
\end{proof}

Finally, we bound the complexity of calling \mcae\xspace with \finalround\xspace = \true \xspace.
\begin{lem}[Formal Statement of \Cref{lem:complexity2}] \label{lem:collect_bound3}
On the event $\Eest \cap \Eexplore$, if \mcae\xspace is called with \finalround\xspace = \true,  the procedure within the if statement on \Cref{line:fr_true} will terminate after collecting at most
\begin{align*}
\frac{c H^4 \iotadeltil \iotaepstil^2}{\epstil^2} |\Xgood_h^{\numepochs+1}| + \frac{\poly(S,A,H, \log 1/\deltil, \log 1/\epstil)}{\epstil}
\end{align*}
episodes. Furthermore, the total complexity of calling \mcae\xspace with \finalround\xspace = \true \xspace is bounded by:
\begin{align*}
H^2  c \iotadeltil \iotaepstil^3 \numepochs \cdot \Compb(\cM,\epstil)  + \frac{\poly(S,A,H,\log 1/\epstil, \log 1/\deltil)}{\epstil}
\end{align*}
for a universal constant $c$.
\end{lem}
\begin{proof}
The only additional samples taken when running \mcae\xspace with \finalround\xspace = \true \xspace as compared to running it with \finalround\xspace = \false \xspace is incurred by calling \collectsamp\xspace on \Cref{line:sap_final} of \mcae. Thus, the total complexity can be bounded by adding the complexity bound from \Cref{lem:collect_bound2} to this additional cost.  

In particular, by \Cref{lem:collectdata_complexity}, this additional call of \collectsamp\xspace will require at most
\begin{align*}
 \frac{c H^4 \iotadeltil \iotaepstil^2}{\epstil^2}| \Xgood_h^{\numepochs+1}| + \frac{\poly(S,A,H, \log 1/\deltil, \log 1/\epstil)}{\epstil}
\end{align*}
episodes to terminate, from which the first conclusion follows.
We can repeat the argument from the proof of \Cref{lem:collect_bound2} to get that $\Xgood_h^{\numepochs+1} \subseteq \cW_h^{\numepochs+1}$, where we define $\cW_h^{\numepochs} := \{ (s,a) \ : \ s \in \Xgood_h, \exists a'\neq a, \max \{ \Delta_h(s,a), \Delta_h(s,a') \} \le 3 \epstil_{\numepochs -1} / (2 \Pst_h(s)) \}$. However, note that $\epstil_{\numepochs-1} \le 2\epstil$, and the condition $\exists a'\neq a, \max \{ \Delta_h(s,a), \Delta_h(s,a') \} \le 3 \epstil_{\numepochs -1} / (2 \Pst_h(s))$ implies $\Deltil_h(s,a) \le 3 \epstil_{\numepochs -1} / (2 \Pst_h(s))$. It follows that 
\begin{align*}
\cW_h^{\numepochs+1} \subseteq \Big \{ (s,a) \ : \ \Deltil_h(s,a) \le 3 \epstil /  \Pst_h(s) \Big \} =: \opt(\epstil,h)
\end{align*}
Summing over $h$ gives the result. 

\end{proof}

\subsection{Proof of \Cref{thm:complexity}}
We are finally ready to complete the proof of \Cref{thm:complexity}.

\begin{proof}[Proof of \Cref{thm:complexity}]
Note that $\Pr[\Eest \cap \Eexplore] \ge 1-\delta$ by \Cref{lem:good_events_hold}. We will assume for the remainder of the proof that this event holds. 

\paragraph{Case 1: $\epsout \ge  \min \{ \min_{s,a,h} \Pst_h(s) \Deltil_h(s,a)/3, 2 H^2 S \min_{s,h} \Pst_h(s) \}$.} 
In this case, that the policy returned is $\epsout$-optimal is guaranteed by \Cref{lem:correct} since the final call to \mcae\xspace is run with \finalround\xspace = \texttt{true}. To bound the sample complexity, we can then simply combine \Cref{lem:collect_bound2} and \Cref{lem:collect_bound3}, which gives that the total sample complexity is bounded as (using that $\epsoutm \ge \epsout$ and that $\deloutm \ge \delout / (36 \lceil \log H/\epsout \rceil^2) =: \delta' $):
\begin{align*}
& \sum_{m=1}^{\lceil \log H/\epsout \rceil - 1} H^2 c \iota_{\deloutm} \iota_{\epsoutm}^3 \ell_{\epsoutm} \cdot \sum_{h=1}^H \inf_{\pi} \max_{s,a} \min \left \{  \frac{1}{\wpi_h(s,a) \Deltil_h(s,a)^2}, \frac{\Pst_h(s)^2}{\wpi_h(s,a) \epsoutm^2} \right \}  \\
& \qquad \qquad + H^2 c \iota_{\deloutm} \iota_{\epsout}^3 \ell_{\epsout} \cdot \sum_{h=1}^H \inf_{\pi} \max_{s,a} \min \left \{  \frac{1}{\wpi_h(s,a) \Deltil_h(s,a)^2}, \frac{\Pst_h(s)^2}{\wpi_h(s,a) \epsout^2} \right \}  \\
& \qquad \qquad + \frac{c H^4 \iota_{\delout} \iota_{\epsout}^2 | \Xbareps|}{\epsout^2} + \frac{\lceil \log H/\epsout \rceil \cdot \poly(S,A,H,\log 1/\epsout, \log 1/\delout)}{\epsout} .
\end{align*}
This can be upper bounded as
\begin{align*}
& \lceil \log H/\epsout \rceil \cdot H^2 c \iota_{\delta'} \iota_{\epsout}^3 \ell_{\epsout} \cdot \sum_{h=1}^H \inf_{\pi} \max_{s,a} \min \left \{  \frac{1}{\wpi_h(s,a) \Deltil_h(s,a)^2}, \frac{\Pst_h(s)^2}{\wpi_h(s,a) \epsout^2} \right \}  \\
& \qquad + \frac{c H^4 \iota_{\delta'} \iota_{\epsout}^2 | \Xbareps|}{\epsout^2}  + \frac{\poly(S,A,H,\log 1/\epsout, \log 1/\delout)}{\epsout}.
\end{align*}
This and the definition of $\Compb(\cM,\epsilon)$ gives the first conclusion of \Cref{thm:complexity}.

\newcommand{\mstop}{\bar{m}}
\paragraph{Case 2: $\epsout < \min \{ \min_{s,a,h} \Pst_h(s) \Deltil_h(s,a)/3, 2 H^2 S \min_{s,h} \Pst_h(s) \}$.} 
As we showed in the proof of \Cref{lem:collect_bound3}, we will have that $\Xgood_h^{\numepochs+1} \subseteq \opt(\epsilon,h)$. Therefore, if for all $(s,a)$, $\Deltil_h(s,a) > 3 \epstil /  \Pst_h(s)$, we will have that $|\Xgood_h^{\numepochs+1}| = 0$, which implies that for every $s \in \Xgood_h$, $|\frakA_h^{\numepochs}(s)| = 1$. Furthermore, on $\Eexplore$, we will have that $\Xgood_h = \cS \times \cA$ if $\frac{\epstil}{2 H^2 S} <  \min_{s} \Pst_h(s)$. If each of these conditions hold for all $h$, then the returned sets $\frakA_h^{\numepochs+1}(s)$ will satisfy $|\frakA_h^{\numepochs+1}(s)|$ for all $s$ and $h$. 

It follows then that if $\epsout < \min \{ \min_{s,a,h} \Pst_h(s) \Deltil_h(s,a)/3, 2 H^2 S \min_{s,h} \Pst_h(s) \}$, either $\epsoutm < \min \{ \min_{s,a,h} \Pst_h(s) \Deltil_h(s,a)/3, 2 H^2 S \min_{s,h} \Pst_h(s) \}$ for some $m$, in which case the above condition will be met, and the termination criteria on \Cref{line:meta_early_term} of \algname will be satisfied, or 
$$\epsoutm \ge  \min \{ \min_{s,a,h} \Pst_h(s) \Deltil_h(s,a)/3, 2 H^2 S \min_{s,h} \Pst_h(s) \},$$
and \algname will reach the final call of \mcae\xspace with \finalround\xspace = \true. In the former case, letting $\mstop$ denote the value of $m$ at which \algname terminates, the total sample complexity will be bounded as, using the same argument as in Case 1,
\begin{align*}
& \sum_{m=1}^{\mstop} H^2 c \iota_{\deloutmstop} \iota_{\epsoutmstop}^3 \ell_{\epsoutmstop} \cdot  \sum_{h=1}^H \inf_{\pi} \max_{s,a} \min \left \{  \frac{1}{\wpi_h(s,a) \Deltil_h(s,a)^2}, \frac{\Pst_h(s)^2}{\wpi_h(s,a) \epsoutm^2} \right \} \\
& \qquad \qquad + \frac{\mstop \cdot \poly(S,A,H,\log 1/\epsoutmstop, \log 1/\deloutmstop)}{\epsoutmstop} \\
& \qquad \le \mstop H^2 c \iota_{\deloutmstop} \iota_{\epsoutmstop}^3 \ell_{\epsoutmstop}  \cdot  \sum_{h=1}^H \inf_{\pi} \max_{s,a} \min \left \{  \frac{1}{\wpi_h(s,a) \Deltil_h(s,a)^2}, \frac{\Pst_h(s)^2}{\wpi_h(s,a) \epsoutmstop^2} \right \} \\
& \qquad \qquad + \frac{ \poly(S,A,H,\log 1/\epsoutmstop, \log 1/\deloutmstop)}{\epsoutmstop}
\end{align*}
and note that $\epsoutmstopb \ge  \min \{ \min_{s,a,h} \Pst_h(s) \Deltil_h(s,a)/3, 2 H^2 S \min_{s,h} \Pst_h(s) \}$, since we did not terminate at round $\mstop - 1$, implying that $\epsoutmstop \ge 2 \min \{ \min_{s,a,h} \Pst_h(s) \Deltil_h(s,a)/3, 2 H^2 S \min_{s,h} \Pst_h(s) \}$. Note also that $\deloutmstop = \frac{\delta}{36 \log^2 \epsoutmstop}$, so we can also bound 
$$\log 1/\deloutmstop \le \cO(\log 1/\delout + \log \log (2 \min \{ \min_{s,a,h} \Pst_h(s) \Deltil_h(s,a)/3, 2 H^2 S \min_{s,h} \Pst_h(s) \})).$$ 
Together these give the bound stated in \Cref{thm:complexity}.

In the latter case, when we do not terminate early at \Cref{line:meta_early_term}, the same sample complexity bound applies but with $\epsoutmstop$ replaced by $\epsout$, since if $|\Xgood_h^{\numepochseps+1}| = 0$, the final call to \collectsamp\xspace in \Cref{line:sap_final} of \mcae\xspace will not collect any samples. As before, in this case we can lower bound 
$$\epsout \ge 2 \min \{ \min_{s,a,h} \Pst_h(s) \Deltil_h(s,a)/3, 2 H^2 S \min_{s,h} \Pst_h(s) \}$$ 
from which the bound follows.

It remains to show that $\pihat = \pist$. This follows inductively from \Cref{lem:action_subopt} since if $| \frakA_H^\ell(s) | = 1$, this implies that for $a \in \frakA_H^\ell(s)$, $a = \pist_H(s)$. Then if we assume that $\pihat_{h'}(s) = \pist_{h'}(s)$ for all $s$ and $h' > h$, if $| \frakA_h^\ell(s) | = 1$ this implies that for $a \in \frakA_h^\ell(s)$, $a = \pist_h(s)$ since, by \Cref{lem:action_subopt}, in this case
\begin{align*}
\max_{a'} \Qpihat_h(s,a') - \Qpihat_h(s,a) = 0
\end{align*}
but $\Qpihat_h(s,a'') = \Qst_h(s,a'')$. Thus, it follows that $\pihat = \pist$, which completes the proof.

\end{proof}

\subsection{Proofs of Additional Lemmas and Claims}\label{sec:good_event_holds}

\begin{proof}[Proof of \Cref{lem:good_events_hold}]
\textbf{$\Eest$ holds.} That $\Eest$ holds with probability $1-\delout/2$ follows directly from Hoeffding's inequality and a union bound, since $\qcheckpith(s_h^t,a_h^t) \le H$ almost surely. In particular, note that for any given call to \mcae\xspace, we will form at most $SAH \iotaepstil (\numepochs + 1)$ estimates of $\Qpihat_h(s,a)$. By Hoeffding's inequality and our choice of $\iotadeltil$, that each of these estimates concentrates as given on $\Eest$ then holds with probability
\begin{align*}
1 - SAH \iotaepstil (\numepochs+1) \cdot \frac{\deltil}{SAH \iotaepstil (\numepochs + 1) } = 1 - \deltil.
\end{align*}
With our choice of $\deloutm = \frac{\delout}{36m^2}$, union bounding over this holding for each call to \mcae, we then have that $\Eest$ holds with probability at least
\begin{align*}
1 - \sum_{m=1}^{\lceil \log H/\epsilon \rceil} \frac{\delout}{36 m^2} \ge 1 - \frac{\delout}{2},
\end{align*}
which is the desired result.

\paragraph{$\Eexplore$ holds.}
We show that the desired events hold for a single call of \mcae, then union bound over all calls to \mcae\xspace to get the final result. Let $\Eexplore^m$ denote the event on which all conditions of $\Eexplore$ hold for the $m$th call to \mcae.

Assume that we run \mcae\xspace with tolerance $\epsoutm$ and confidence $\deloutm$. Let $\Esap^{sh}$ denote the success event of calling \sap\xspace on \Cref{line:get_phat}, $\Esap^{h i \ell}$ denote the success event of calling \sap\xspace in the call to \collectsamp\xspace at iteration $(h,i,\ell)$ on \Cref{line:collectsamp1}, and $\Esap^h$ the success event of calling \sap\xspace in the call to \collectsamp\xspace on \Cref{line:sap_final}. By \Cref{thm:partitioning_works} and the confidence with which we call \sap, we have that $\Pr[\Esap^{sh}] \ge 1 - \deloutm/SH$, $\Pr[\Esap^{hi\ell}] \ge 1 - \deloutm/(H\iotaepstil \numepochs)$, and $\Pr[\Esap^h] \ge 1 - \deloutm/H$. Union bounding over these events, and using that there are at most $H \iotaepstil \numepochs$ indices $(h,i,\ell)$, we get that the event
\begin{align*}
(\cap_{s,h} \Esap^{sh}) \cap ( \cap_{h=1}^H \cap_{i=1}^{\iotaepstil} \cap_{\ell=1}^{\numepochs} \Esap^{hi\ell}) \cap ( \cap_{h=1}^H \Esap^h)
\end{align*}
holds with probability at least $1-3\deloutm$. 

That
\begin{align*}
\sup_\pi \sum_{(s,a) \in \cX_{hij}^\ell} \wpi_h(s,a) \le 2^{-j+1}
\end{align*}
for $j \in [\iotaepstil]$, is a direct consequence of $\Esap^{hi\ell}$ holding, and similarly that
\begin{align*}
\sup_\pi \sum_{(s,a) \in \Xlast_{hj}} \wpi_h(s,a) \le 2^{-j+1}
\end{align*}
holds for $j \in [\iotaepstil]$, is a direct consequence of $\Esap^h$. In addition, that 
\begin{align*}
\sup_\pi  \max_{s \in \Xgood_h^c} \wpi_h(s) \le \frac{\epsilon}{2H^2S}
\end{align*}
holds for all $h$ is immediate on $\cap_{s,h} \Esap^{sh}$.

On the event $\Esap^{h i \ell}$, if we run the policies returned by \sap\xspace for some $j \in \{ 1,\ldots, \iotaepstil \}$, $\Pi_{hij}^\ell$, \Cref{thm:partitioning_works} and our choice of $\delsamp$ gives that we will collect at least $\frac{1}{2} N_{hij}^\ell$ samples from each $(s,a) \in \cX_{hij}^\ell$ with probability at least $1 - \deloutm/(H \iotaepstil^2 \numepochs n_{i1}^\ell)$. As \collectsamp\xspace runs each policy $\lceil 2 n_{i1}^\ell/N_{hij}^\ell \rceil$ times, it follows that we will collect at least $\lceil 2 n_{i1}^\ell/N_{hij}^\ell \rceil \cdot \frac{1}{2} N_{hij}^\ell \ge n_{i1}^\ell$ samples from each $(s,a) \in \cX_{hij}^\ell$ with probability at least $1 - \deloutm/(H \iotaepstil^2 \numepochs n_{i1}^\ell) \cdot \lceil 2 n_{1i}^\ell/N_{hij}^\ell \rceil \ge 1 - 3 \deloutm/(H \iotaepstil^2 \numepochs)$. Union bounding over this for each $h,i,\ell$ and $j \in [\iotaepstil]$ gives that with probability at least $1-3\deloutm$, we collect at least $n_{i1}^\ell$ samples from each $(s,a) \in \cX_{hij}^\ell$. The same argument gives that with probability at least $1-3\deloutm$ we collect at least $\nlast_j$ samples from each $(s,a) \in \Xlast_{hj}$, $j = 1,\ldots,\iotaepstil, h \in [H]$.

\paragraph{Relating $\Phatst_h(s)$ to $\Pst_h(s)$.} 
It remains to show that $\Phatst_h(s) \le \Pst_h(s) \le 32 \Phatst_h(s)$ for all $s \in \Xgood_h$, $\cup_{j=1}^{\iotaepstil} \cX_{hij}^\ell = \Xgood_{hi}^\ell$, and $\cup_{j=1}^{\iotaepstil} \Xlast_{hj} = \Xgood_h^{\numepochs+1}$.

We first show $\Phatst_h(s) \le \Pst_h(s) \le 32 \Phatst_h(s)$. Consider running \sap\xspace with $\cX = \{ (s,a) \}$ for arbitrary $a$ and assume that $\cX_j^{sh}$ is the returned partition containing $(s,a)$. By \Cref{thm:partitioning_works}, on $\Esap^{sh}$ we will have that
\begin{align*}
\Pst_h(s)  \le 2^{-j+1}
\end{align*}
and, furthermore, that with probability at least $1/2$, if we rerun all policies in $\Pi_j^{sh}$ returned by \sap\xspace, we will obtain at least $N_j^{sh}/2 =  |\Pi_j^{sh}|/(8 |\cX| 2^j) =   |\Pi_j^{sh}|/(8 \cdot 2^j)$ samples from $(s,a,h)$. 

Let $X$ be a random variable which is the count of total samples collected in $(s,a,h)$ when running $\pi_k \in \Pi_j^{sh}$. Then Markov's inequality and the above property of $\Pi_j^{sh}$ gives
\begin{align*}
\frac{1}{2} \le \Pr[ X \ge N_j^{sh}/2] \le \frac{2 \Exp[X]}{N_j^{sh}} = \frac{2}{N_j^{sh}} \sum_{\pi \in \Pi_j^{sh}} \wpi_h(s,a) \le \frac{2 |\Pi_j^{sh}|}{N_j^{sh}} \Pst_h(s) = 8 \cdot 2^j \Pst_h(s).
\end{align*}
Rearranging this and recalling that we set $\Phatst_h(s) = \frac{1}{16 \cdot 2^j}$, we have that $\Phatst_h(s) \le \Pst_h(s)$. However, we also have
\begin{align*}
\Pst_h(s) \le 2^{-j+1} = 32 \Phatst_h(s).
\end{align*}
This proves that $\Phatst_h(s) \le \Pst_h(s) \le 32 \Phatst_h(s)$. 

Now note that any $s \in \Xgood_h$ has $\Phatst_h(s) \ge \frac{\epsoutm}{32 H^2 S}$, which, combined with the above, implies that $\Pst_h(s) \ge \frac{\epsoutm}{32 H^2 S}$. Fix $(h,i,\ell)$, and note that the call to \sap\xspace in the call to \collectsamp\xspace for index $(h,i,\ell)$ uses input tolerance $\frac{\epsoutm}{64H^2S}$. \Cref{thm:partitioning_works} then gives that, on $\Esap^{hi\ell}$, we will have
\begin{align*}
\sup_\pi \sum_{(s,a) \in \Xgood_{hi}^\ell \backslash (\cup_{j=1}^{\iotaepstil}  \cX_{hij}^\ell)} \wpi_h(s,a) \le \frac{\epsoutm}{64 H^2 S}.
\end{align*}
However, as $\Pst_h(s') \le \sup_\pi \sum_{(s,a) \in \Xgood_{hi}^\ell \backslash (\cup_{j=1}^{\iotaepstil}  \cX_{hij}^\ell)} \wpi_h(s,a)$ for any $(s',a)  \in  \Xgood_{hi}^\ell \backslash (\cup_{j=1}^{\iotaepstil}  \cX_{hij}^\ell)$, we will have that any $(s,a) \in \Xgood_{hi}^\ell \backslash (\cup_{j=1}^{\iotaepstil}  \cX_{hij}^\ell)$ has $\Pst_h(s) \le \frac{\epsoutm}{64 H^2 S}$. This is a contradiction since we know $\Pst_h(s) \ge \frac{\epsoutm}{32 H^2 S}$ for any $(s,a) \in \Xgood_{hi}^\ell$. Thus, we must have that $ \Xgood_{hi}^\ell \backslash (\cup_{j=1}^{\iotaepstil}  \cX_{hij}^\ell) = \emptyset$ so $\cup_{j=1}^{\iotaepstil} \cX_{hij}^\ell = \Xgood_{hi}^\ell$. The same argument shows that $\cup_{j=1}^{\iotaepstil} \Xlast_{hj} = \Xgood_h^{\numepochs+1}$.

\paragraph{Completing the proof.}
We have therefore shown that $\Pr[\Eexplore^m] \ge 1 - 9 \deloutm$. Union bounding over all $m$, by our choice of $\deloutm = \frac{\delout}{36 m^2}$, we have that 
$$\Pr[\Eexplore] = \Pr[\cap_{m=1}^{\lceil \log H/\epsilon \rceil} \Eexplore^m] \ge 1 - \sum_{m=1}^{\lceil \log H/\epsilon \rceil } 9 \frac{\delout}{36 m^2} \ge 1 - \delout/2. $$ 
Union bounding over $\Eexplore$ and $\Eest$ then gives the result.
\end{proof}

\begin{proof}[Proof of \Cref{lem:optactionsin}]
We proceed by induction. Consider some $s \in \Xgood_{hi}$. The base case is trivial as $\frakA_h^0(s) = \cA$. Fix some $\ell \le \numepochs$ and assume that $\ahatst_h(s) \in \frakA_h^{\ell-1}(s)$ and $|\frakA_h^{\ell-1}(s)| > 1$. Then, on $\Eexplore$, we can guarantee that we will collect at least $\frac{2^{18} H^2 \iotadeltil}{2^{2i} \epsilon_{\ell}^2}$ samples from $(s,a)$ for each $a \in \frakA_h^{\ell-1}$. On the event $\Eest$, it then follows that for each $a \in \frakA_h^{\ell-1}(s)$,
\begin{align*}
| \Qhatpihat_{h,\ell}(s,a) - \Qpihat_h(s,a)|  \le 2^{i}\epsilon_{\ell}/2^9.
\end{align*}
Thus, since by assumption $\ahatst_h(s) \in \frakA_h^{\ell-1}(s)$,
\begin{align*}
\max_{a \in \frakA_h^{\ell-1}(s)} \Qhatpihat_{h,\ell}(s,a) - \Qhatpihat_{h,\ell}(s,\ahatst_h(s)) & \le \max_{a \in \frakA_h^{\ell-1}(s)} \Qpihat_{h}(s,a) - \Qpihat_{h}(s,\ahatst_h(s)) + 2 \cdot 2^{i}\epsilon_{\ell}/2^9 \\
& \le 2 \cdot 2^{i}\epsilon_{\ell}/2^9 \\
& = \gamma_{ij}^\ell
\end{align*} 
for any $j$, so the exit condition on \Cref{line:collect_Aell} of \collectdata\xspace is not met for $\ahatst_h(s)$, and thus $\ahatst_h(s) \in \frakA_h^{\ell}(s)$. The result follows analogously if $\ell = \numepochs+1$, in which case we simply use the different values of $n$ and $\gamma$.

Now if $(s,a) \not\in \Xgood_{hi}^\ell$ for all $a$, that means we will never remove arms from $\frakA_h^\ell(s)$ again. However, by the above inductive argument, if $\ell'$ is the last round such that $(s,a) \in \Xgood_{hi}^{\ell'}$ for some $a$, we will have that $\ahatst_h(s) \in \frakA_h^{\ell'}(s)$, so it follows that $s \in \frakA_h^\ell(s)$. 

Finally, if $s \not\in \Xgood_{h}$, then we will never remove an arm from $\frakA_h^{0}(s)$, and since $\frakA_h^{0}(s) = \cA$, the conclusion follows trivially. 
\end{proof}

\begin{proof}[Proof of \Cref{lem:gap_lower_bound}]
In \Cref{lem:correct}, we showed that the local suboptimality bounds of $\pihat$, $\epstil_h(s)$, satisfy
\begin{align*}
\sum_{h=1}^H \sup_\pi \sum_s \wpi_h(s) \epstil_h(s) \le \epsilon.
\end{align*}
By \Cref{lem:local_to_global_subopt2}, it follows that for any $\pi'$ and any $h$,
\begin{align*}
\sum_s w^{\pi'}_h(s) (\Vst_h(s) - \Vpihat_h(s)) \le \sum_{h'=h}^H \sup_\pi \sum_{s} \wpi_{h'}(s) \epstil_{h'}(s) \le \epsilon.
\end{align*}
The result then follows from \Cref{lem:gap_diff}.
\end{proof}

\begin{lem}\label{lem:local_to_global_subopt2}
Assume that for each $h$ and $s$, $\pihat$ plays an action which satisfies 
\begin{align}\label{eq:local_subopt}
\max_a Q_{h}^{\pihat}(s,a) - \Qpihat_{h}(s,\pihat_{h}(s)) \le \epsilon_{h}(s).
\end{align}
Then for any $h$ and $\pi'$,
\begin{align*}
\sum_s w_{h}^{\pi'}(s) ( \Vst_{h}(s) - \Vpihat_{h}(s)) \le \sum_{h'=h}^H \sup_\pi \sum_s \wpi_{h'}(s) \epsilon_{h'}(s).
\end{align*}
\end{lem}
\begin{proof}
We proceed by backwards induction. The base case, $h=H$, is trivial. Assume that at level $h$, for any $\pi$,
\begin{align*}
\sum_s \wpi_h(s)( \Vst_h(s) - \Vpihat_h(s)) \le  \sum_{h'=h}^H \sup_{\pi'} \sum_{s'} w_{h'}^{\pi'}(s') \epsilon_{h'}(s')
\end{align*}
and that at level $h-1$, for each $s$ \eqref{eq:local_subopt} holds. By definition,
\begin{align*}
\Vst_{h-1}(s) - \Vpihat_{h-1}(s)  & = \Qst_{h-1}(s,\pist_{h-1}(s)) - \Qpihat_{h-1}(s,\pihat_{h-1}(s)) \\
& = \Qst_{h-1}(s,\pist_{h-1}(s)) - \Qpihat_{h-1}(s,\pist_{h-1}(s)) + \Qpihat_{h-1}(s,\pist(s)) - \max_a \Qpihat_{h-1}(s,a) \\
& \qquad \qquad + \max_a \Qpihat_{h-1}(s,a) - \Qpihat_{h-1}(s,\pihat_{h-1}(s)).
\end{align*}
Clearly, $\Qpihat_{h-1}(s,\pist(s)) - \max_a \Qpihat_{h-1}(s,a) \le 0$ and by assumption $\max_a \Qpihat_{h-1}(s,a) - \Qpihat_{h-1}(s,\pihat_{h-1}(s)) \le \epsilon_{h-1}(s)$. Furthermore,
\begin{align*}
\Qst_{h-1}(s,\pist_{h-1}(s)) - \Qpihat_{h-1}(s,\pist_{h-1}(s)) & = \sum_{s'} P_{h-1}(s'|s,\pist_{h-1}(s)) (\Vst_h(s') - \Vpihat_h(s')) .
\end{align*}
Then, for any $\pi$,
\begin{align*}
\sum_s \wpi_{h-1}(s)( \Vst_{h-1}(s) - \Vpihat_{h-1}(s)) & \le \sum_s \wpi_{h-1}(s) \epsilon_{h-1}(s) + \sum_s \sum_{s'} \wpi_{h-1}(s) P_{h-1}(s'|s,\pist_{h-1}(s)) (\Vst_h(s') - \Vpihat_h(s')) \\
& =  \sum_s \wpi_{h-1}(s) \epsilon_{h-1}(s) + \sum_s  w^{\pi'}_{h}(s) (\Vst_h(s) - \Vpihat_h(s)) \\
& \le  \sum_{h'=h-1}^H \sup_{\pi'} \sum_{s'} w_{h'}^{\pi'}(s') \epsilon_{h'}(s')
\end{align*}
where the last inequality follows by the inductive hypothesis and we have used that
\begin{align*}
\sum_s \wpi_{h-1}(s) P_{h-1}(s'|s,\pist_{h-1}(s)) = w_h^{\pi'}(s').
\end{align*}
where $\pi_{h'}'(s) = \pi_{h'}(s)$ for all $h' \leq h-2$ and $\pi_{h'}'(s) = \pi_{h'}^\star(s)$ for $h' \geq h-1$.
The conclusion then follows.
\end{proof}


\section{Learning to Explore}\label{sec:learn2explore}

Define the following value: 
\begin{align}
K_i(\delta,\delsamp) & = \Bigg \lceil 2^i  \max \Bigg \{  288 \ceuler^{2} S^{2} A^{2} H (i + 3) \log ( 576 \ceuler S A H(i+3)) ,  288 \ceuler^2 S^2 A^2 H \log \frac{2SAH }{\delta}, \nonumber \\
& \qquad 2048 S^2 A^2 \log \frac{4SAH}{\delsamp},  256 \ceuler S^3 A^2 H^4 (i+9)^3 \log^3 \left ( 512 \ceuler S A H (i+9) \right ),\label{eq:ki_defn} \\
& \qquad 128 \ceuler S^3 A^2 H^4 \log^3 \frac{2SAH}{\delta} +  8 H  \log \frac{4}{\delta} \Bigg \} \Bigg \rceil  \nonumber \\
& =: 2^i C_K(\delta,\delsamp,i) \nonumber
\end{align}
and note that $C_K(\delta,\delsamp,i) = \poly(S,A,H, \log 1/\delta,\log 1/\delsamp, i)$.

\begin{rem}
The exploration procedure of \ies\xspace is potentially quite wasteful as we restart \euler every time the desired number of samples for a given state is collected. This could likely be improved on by instead running a regret-minimization algorithm that is able to handle time-varying rewards, such as the algorithm presented in \cite{zhang2020nearly}. As the focus of this work is not in optimizing the lower-order terms, we chose to instead simply use \euler.
\end{rem}

\begin{thm}[Formal Statement of \Cref{thm:partitioning_works_informal}]\label{thm:partitioning_works}
Consider running \sap\xspace with tolerance $\epsltoe \leftarrow \epsilon$ and confidence $\delta$ and obtaining a partition $\cX_i \subseteq \cS \times \cA  $ and policies $\Pi_i$, $i \in \{1,2,\ldots,\lceil \log(1/\epsilon) \rceil\}$. Let $\Esap$ be the event on which, for all $i$ simultaneously:
\begin{enumerate}
\item Sets $\cX_i$ satisfy:
\begin{align*}
\sup_\pi \sum_{(s,a) \in \cX_{i}} \wpi_h(s,a) \le 2^{-(i-1)}
\end{align*}
\item For any $i$, if all policies in $\Pi_i$ are each rerun once, we will collect $\frac{1}{2} N_i$ samples from each $(s,a) \in \cX_i$ with probability $1-\delsamp$, where we recall $N_i = K_i(\delta/\lceil \log (1/\epsilon) \rceil,\delsamp)/(4 \cdot 2^{i} | \cX \backslash \cup_{i'=1}^{i-1} \cX_{i'}|)$.
\item The remaining states, $\cX \backslash (\cup_{i=1}^{\lceil \log(1/\epsilon) \rceil} \cX_i)$ satisfy,
\begin{align*}
\sup_\pi \sum_{(s,a) \in (\cX \backslash (\cup_{i=1}^{\lceil \log(1/\epsilon) \rceil} \cX_i))} \wpi_h(s,a) \le \epsilon.
\end{align*}
\end{enumerate}
Then $\Pr[\Esap] \ge 1 - \delta$. Furthermore, \Cref{alg:partition} takes at most
\begin{align*}
\CK \left ( \frac{\delta}{\lceil \log 1/\epsilon \rceil}, \delsamp, \lceil \log 1/\epsilon \rceil \right )\frac{4}{\epsilon}
\end{align*}
episodes to terminate.
\end{thm}
\begin{proof}
This directly follows by induction and \Cref{lem:ies_guarantee}. For $i=1$, it will clearly be the case that
\begin{align*}
\sup_\pi  \sum_{(s,a) \in \cX} \wpi_h(s,a) \le 2^{-(i-1)} = 1
\end{align*}
since $\sum_{s,a} \wpi_h(s,a) = 1$ for any $\pi$ and $h$. Now consider an epoch $i$ and assume that
\begin{align*}
\sup_\pi  \sum_{(s,a) \in \cX} \wpi_h(s,a) \le 2^{-(i-1)}.
\end{align*}
By \Cref{lem:ies_guarantee}, running \ies\space will produce a set $\cX_i$ and policies $\Pi_i$ such that
\begin{align*}
\sup_\pi  \sum_{(s,a) \in \cX_i} \wpi_h(s,a) \le 2^{-(i-1)}, \quad \sup_\pi \sum_{(s,a) \in \cX \backslash \cX_i} \wpi_h(s,a) \le 2^{-i}
\end{align*}
and rerunning every policy in $\Pi_i$ at once will allow us to collect at least $\frac{1}{2} N_i$ samples from each $(s,a) \in \cX_i$. As $\cX \leftarrow \cX \backslash \cX_i$, the hypothesis will then be met at the next epoch, $i+1$. Union bounding over epochs completes the first part of the proof. That
\begin{align*}
\sup_\pi \sum_{(s,a) \in (\cX \backslash (\cup_{i=1}^{\lceil \log(1/\epsilon) \rceil} \cX_i))} \wpi_h(s,a) \le \epsilon
\end{align*}
follows on this same event by \Cref{lem:ies_guarantee} and since we run until $i = \lceil \log(1/\epsilon) \rceil$ which implies $2^{-\lceil \log(1/\epsilon) \rceil} \le  \epsilon$. Union bounding over each $i$ gives the result.

The sample complexity bound follows by bounding 
\begin{align*}
\sum_{i=1}^{\lceil \log(1/\epsilon) \rceil} K_i(\delta/\lceil \log 1/\epsilon \rceil,\delsamp) & \le \CK \left ( \frac{\delta}{\lceil \log 1/\epsilon \rceil}, \delsamp, \lceil \log 1/\epsilon \rceil \right ) \sum_{i=1}^{\lceil \log(1/\epsilon) \rceil} 2^i \\
& \le  \CK \left ( \frac{\delta}{\lceil \log 1/\epsilon \rceil}, \delsamp, \lceil \log 1/\epsilon \rceil \right ) \frac{4}{\epsilon}.
\end{align*}
\end{proof}

\begin{lem}\label{lem:ies_guarantee}
Assume that  $\cX$ satisfies
\begin{align*}
\sup_\pi \sum_{(s,a) \in \cX} \wpi_h(s,a) \le 2^{-(i-1)}.
\end{align*}
Then, if \ies($\cX,h,\delta,K_i,N_i$) returns partition $\cX_i$ and policies $\Pi_i$, with probability $1-\delta$ the returned partition $\cX_i$ will satisfy
\begin{align*}
\sup_\pi \sum_{(s,a) \in \cX_i} \wpi_h(s,a) \le 2^{-(i-1)}, \quad \sup_\pi \sum_{(s,a) \in \cX \backslash \cX_i} \wpi_h(s,a) \le 2^{-i}.
\end{align*}
Furthermore, if all policies in $\Pi_i$ are each rerun once, we will collect $\frac{1}{2} N_i$ samples from each $(s,a,h) \in \cX_i$ with probability $1-\delsamp$. 
\end{lem}
\begin{proof}
The structure of this proof takes inspiration from the proof presented in \cite{zhang2020nearly}. The first conclusion is trivial since $\cX_i \subseteq \cX$ and by our assumption on $\cX$. 

We will simply denote $K_i := K_i(\delta,\delsamp)$ throughout the proof. In addition, we will let $K_{ij}$ denote the total number of epochs taken for fixed $j$, and will let $m_i$ denote the total number of times $j$ is incremented. Therefore,
\begin{align*}
K_i =  \sum_{j=1}^{m_i} K_{ij}.
\end{align*}
Let $V_0^{\star,ij}$ denote the optimal value function on the reward function $r_h^j$ at stage $j$ of epoch $i$. By our assumption on $\cX$ and the definition of our reward function we can bound
\begin{align}\label{eq:exp_vst1_bound}
\begin{split}
V_0^{\star,ij} & \le \sup_{\pi} \Exp_\pi [ \I \{ (s_h,a_h) \in \cX \} ]  = \sup_\pi \sum_{(s,a) \in \cX} \wpi_h(s,a)  \le 2^{-(i-1)}.
\end{split}
\end{align}
As \ies\xspace runs \euler, by \Cref{lem:euler} we will have, with probability at least $1-\delta$, for any fixed $K$ and $j$, 
\begin{align}\label{eq:eu_reg}
 \Big ( \sum_{k=1}^{K} V_0^{\star,ij} -  \sum_{k=1}^{K} V_0^{k,ij} \Big ) | \cF_{j-1} \le \ceuler \sqrt{SAH V_0^{\star,i1} K \log \frac{SAHK}{\delta} } + \ceuler S^2 A H^4 \log^3 \frac{SAHK}{\delta} 
\end{align}
where $\cF_{j-1}$ denotes the filtration of up to iteration $j$, and we have used that $V_0^{\star,ij} \le V_0^{\star,i1}$ for all $j$ since the reward function can only decrease as $j$ increases. \ies\xspace terminates and restarts \euler if the condition on \Cref{line:restart_euler} is met, but this is a \emph{random} stopping condition. As such, to guarantee that \eqref{eq:eu_reg} holds for any possible value of this stopping time, we union bound over all values. Since \ies\xspace runs for at most $K_i$ epochs, it suffices to union bound over $K_i$ stopping times. We then have that
\begin{align*}
 \Big ( \sum_{k=1}^{K} V_0^{\star,ij} -  \sum_{k=1}^{K} V_0^{k,ij} \Big ) | \cF_{j-1} \le 2\ceuler \sqrt{SAH V_0^{\star,i1}  K \log \frac{2SAHK_i}{\delta} } + 8\ceuler S^2 A H^4 \log^3 \frac{2SAHK_i}{\delta} 
\end{align*}
with probability at least $1 - \frac{\delta}{2 SA}$ for all $K \in [1,K_i]$ simultaneously. Since $m_i \le SA$, union bounding over all $j$ we then have that, with probability at least $1-\delta/2$,
\begin{align*}
\sum_{j=1}^{m_i} \Big ( \sum_{k=1}^{K_{ij}} V_0^{\star,ij} -  \sum_{k=1}^{K_{ij}} V_0^{k,ij} \Big ) & \le \sum_{j=1}^{m_i} 2\ceuler \sqrt{SAH V^{\star,i1}_0 K_{ij} \log \frac{2SAHK_{i}}{\delta} } + 8\ceuler S^3 A^2 H^4 \log^3 \frac{2SAHK_i}{\delta} \\
& \le 2\ceuler \sqrt{S^2A^2H V^{\star,i1}_0 K_{i} \log \frac{2SAHK_{i}}{\delta} } + 8\ceuler S^3 A^2 H^4 \log^3 \frac{2SAHK_i}{\delta} 
\end{align*}
where the final inequality follows from Jensen's inequality. Using the same calculation as in the proof of \Cref{lem:euler}, we can bound
\begin{align*}
\Exp_{\pi_k}[(\sum_{h=1}^H R_h^j(s_h,a_h) - V_0^{k,ij})^2] \le 4  V_0^{k,ij}
\end{align*}
By \eqref{eq:exp_vst1_bound}, $4  V_0^{k,ij} \le 4/2^{i-1}$, so we can apply \Cref{lem:rewards_concentrate} with $\sigma_V^2 = 4/2^{i-1}$, to get that, with probability at least $1-\delta/2$,
\begin{align*}
\left | \sum_{j=1}^{m_i} \sum_{k=1}^{K_{ij}} \sum_{h=1}^H R_h^j(s_h^{j,k},a_h^{j,k}) - \sum_{j=1}^{m_i} \sum_{k=1}^{K_{ij}} V_0^{k,ij} \right | \le \sqrt{ 32  K_i 2^{-i}  \log \frac{4}{\delta} } + 2 H  \log \frac{4}{\delta}.
\end{align*}
Putting this together and union bounding over these events, we have that with probability at least $1-\delta$,
\begin{align*}
\sum_{j=1}^{m_i} \sum_{k=1}^{K_{ij}} \sum_{h=1}^H R_h^j(s_h^{j,k},a_h^{j,k}) \ge \sum_{j=1}^{m_i} \sum_{k=1}^{K_{ij}} V_0^{\star,ij} - \sqrt{ 64  K_i 2^{-i}  \log \frac{4}{\delta} } -  2\ceuler \sqrt{S^2A^2H V^{\star,i1}_0 K_{i} \log \frac{2SAHK_{i}}{\delta} } - \Creg
\end{align*}
where we denote
\begin{align*}
\Creg := 8\ceuler S^3 A^2 H^4 \log^3 \frac{2SAHK_i}{\delta} +  2 H \log \frac{4}{\delta}.
\end{align*}
Assume that $V_0^{\star,i m_i} > 2^{-i}$. Using that the reward decreases monotonically so $V_0^{\star,i m_i} \le V_0^{\star,ij}$ for any $j \le m_i$, we can lower bound the above as
\begin{align*}
&  \ge 2^{-i} K_i - \sqrt{ 64  K_i 2^{-i}  \log \frac{4}{\delta} }- 2\ceuler \sqrt{ S^2 A^2 H V_0^{\star,i1} K_i  \log \frac{2SAHK_{i}}{\delta}} -  \Creg \\
& \qquad \ge 2^{-i} K_i - 3 \ceuler \sqrt{ S^2 A^2 H 2^{-i} K_i  \log \frac{2SAHK_{i}}{\delta}} -  \Creg
\end{align*}
where the second inequality follows by \eqref{eq:exp_vst1_bound} and since $\sqrt{ 64 K_i 2^{-i}  \log \frac{4}{\delta} }$ will then be dominated by the regret term, $\ceuler \sqrt{ S^2 A^2 H V_0^{\star,i1} K_i  \log \frac{2SAHK_{i}}{\delta}}$. \Cref{lem:Ki_large} gives
\begin{align*}
K_i \ge 2^i \max \left \{ 4 \Creg, 144 \ceuler^2 S^2 A^2 H \log \frac{2SAHK_{i}}{\delta} \right \}
\end{align*}
which implies
\begin{align*}
\frac{1}{4} 2^{-i} K_i - \Creg \ge 0
\end{align*}
and
\begin{align*}
\frac{1}{4} 2^{-i} K_i & - 3 \ceuler \sqrt{ S^2 A^2 H 2^{-i} K_i  \log \frac{2SAHK_{i}}{\delta}} \\
& \ge \frac{ 2^{i} \cdot 144 \ceuler^2 S^2A^2 H  \log \frac{2SAHK_{i}}{\delta}}{4 \cdot 2^i} - 3 \ceuler \sqrt{S^2A^2H 2^{-i}  \log \frac{2SAHK_{i}}{\delta} \cdot  2^{i} 144 \ceuler^2 S^2A^2 H  \log \frac{2SAHK_{i}}{\delta}} \\
& = 0.
\end{align*}
Thus, we can lower bound the above as
\begin{align*}
2^{-i} K_i - 3 \ceuler \sqrt{ S^2 A^2 H 2^{-i} K_i  \log \frac{2SAHK_{i}}{\delta}} -  \Creg  \ge \frac{1}{2 } 2^{-i}  K_{i} .
\end{align*}
Note that we can collect a total reward of at most $|\cX| N_i$. However, by our choice of $N_i = K_i/(4 |\cX| \cdot 2^i)$, we have that 
\begin{align*}
|\cX| N_i = \frac{1}{4 \cdot 2^i} K_i < \frac{1}{2 \cdot 2^i}  K_i.
\end{align*}
This is a contradiction. Thus, we must have that $V_0^{\star,i m_i} \le 1/2^i$. The second conclusion follows from this by definition of $V_0^{\star,i m_i} $.

For the third conclusion, we can apply \Cref{lem:enough_samples_collected}. By construction, we will only add some $(s,a,h)$ to $\cX_i$ if we visit $N_i$ times. It follows by \Cref{lem:enough_samples_collected} that, with probability $1-\delsamp/(SAH )$, if we rerun all policies, we will collect at least
\begin{align*}
N_i - \sqrt{8 K_i \max_k w_h^{\pi_k}(s,a) \log \frac{4SAH}{\delsamp} } - \frac{4}{3} \log \frac{4SAH}{\delsamp}
\end{align*}
samples from $(s,a,h)$. Note that $\max_k w_h^{\pi_k}(s,a) \le 2^{-i}$ by our assumption on $\cX$. Given our choice of $N_i$, we can then guarantee that we will collect at least
\begin{align*}
\frac{K_i}{4|\cX| 2^i} - \sqrt{\frac{8K_i H }{2^i} \log \frac{4SAH}{\delsamp} } - \frac{4}{3} \log \frac{4SAH}{\delsamp}
\end{align*}
samples. Since $K_i \ge 2048 S^2 A^2 \log \frac{4SAH}{\delsamp}$, and $|\cX| \le SA$, we will have that
\begin{align*}
\frac{K_i}{4|\cX| 2^i} - \sqrt{\frac{8K_i H}{2^i}  \log \frac{4SAH}{\delsamp} } - \frac{4}{3} \log \frac{4SAH}{\delsamp} \ge \frac{K_i}{8|\cX| 2^i} = \frac{1}{2} N_i
\end{align*}
The third conclusion follows by union bounding over every $(s,a,h) \in \cX_i$.
\end{proof}

\begin{rem}[Improving lower order term to $\log 1/\delta \cdot \log\log 1/\delta$]\label{rem:logterm}
In \iftoggle{arxiv}{\Cref{rem:log_delta}}{\Cref{sec:results}} we noted that relying on \textsc{StrongEuler} instead of \euler in the exploration phase would allow us to reduce the lower order term from $\log^3 1/\delta$ to $\log 1/\delta \cdot \log\log 1/\delta$. We briefly sketch out that argument here.

As shown in \cite{simchowitz2019non}, the lower order term in \textsc{StrongEuler} scales as
$H^4 SA ( S \vee H) \log \frac{SAHK}{\delta} \cdot \min \{\log  \frac{SAHK}{\delta}, \log \frac{SAH}{\Delmin} \}$. This already achieves the correct scaling in $\log 1/\delta$ but unfortunately relies on an instance-dependent quantity, $\Delmin$, which is unknown (indeed, note that since we are running this on the MDP with reward function set to induce exploration, $\Delmin$ here is different than the minimum gap on the original reward function). As such, since \sap\xspace relies on knowing the regret bound of the algorithm it is running, this bound cannot be applied directly. 

Fundamentally, the lower order term arises from summing over the lower order term in the Bernstein-style bonuses which scale as $\cO(\frac{\log 1/\delta}{N_h(s,a)})$, where $N_h(s,a)$ is the visitation count of $(s,a,h)$. Intuitively, by summing this bonus over all $s,a,h$ and episodes $K$, we can obtain a term scaling as $\poly(S,A,H) \log (1/\delta) \log K$. Indeed, we see that the original proof of \textsc{StrongEuler} in \cite{simchowitz2019non} relies on an integration lemma which does just this (Lemma B.9). However, by modifying the proof of this lemma slightly, we obtain a scaling in the lower-order term of $\log^2 K + \log K \cdot \log 1/\delta$. We then apply the observation from \Cref{lem:Ki_large} that $x \ge C^i (i+3j)^j \log^j(C(i+3j))$ implies $x \ge C^i \log^j x$ to get that we need only
\begin{align*}
K \gtrsim C \log(1/\delta) \log (C \log (1/\delta)), \quad K \gtrsim C \log^2( C)
\end{align*}
to ensure that $K \gtrsim C(\log^2 K + \log K \cdot \log 1/\delta)$. It follows that using the lower order term of \textsc{StrongEuler} in the definition of $\Creg$ in \Cref{lem:ies_guarantee}, we can guarantee that $K_i \ge 2^i \Creg$ while only requiring that $K_i \gtrsim \log(1/\delta) \log ( \log (1/\delta))$. This allows us to reduce the $\log 1/\delta$ dependence in the definition of $C_K$, which allows us to then reduce the dependence on $\log 1/\delta$ in the lower-order term of \Cref{thm:complexity}. 
\end{rem}

\subsection{Technical Lemmas}

\begin{lem}\label{lem:Ki_large}
We will have that
\begin{align*}
K_i(\delta,\delsamp) \ge 2^i \max \Bigg \{ & 32 \ceuler S^3 A^2 H^4 \log^3 \frac{2SAHK_i(\delta,\delsamp)}{\delta} +  8 H \log \frac{4}{\delta}, \\
& 144 \ceuler^2 S^2 A^2 H \log \frac{2SAHK_{i}(\delta)}{\delta} \Bigg \}.
\end{align*}
\end{lem}
\begin{proof}
Note that for any $i,j > 0 $ and $C > 0$, if $x \ge C^i (i+3j)^j \log^j(C(i+3j))$, then $x \ge C^i \log^j x$ since
\begin{align*}
C^i \log^j x = C^i \log^j [ C^i (i+3j)^j \log^j(C(i+3j)) ] & \le C^i \log^j [ C^{i+j} (i+3j)^{2j} ] \\
& \le C^i (i+3j)^j \log[C (i+3j)] \\
& = x
\end{align*}
and, furthermore, $\frac{d}{dy} y |_{y = C^{i+j} (\max\{i+j,2j\})^{2j}} = 1$, while
\begin{align*}
\frac{d}{dy} C^i \log^j y |_{y = C^{i+j} (\max\{i+j,2j\})^{2j}} = \frac{C^i \log^{j-1} y}{y}|_{y = C^{i+j} (\max\{i+j,2j\})^{2j}} \le 1
\end{align*}
and since the derivative of $\poly\log$ functions decreases monotonically.

It follows that
\begin{align*}
K_i(\delta,\delsamp) \ge 2^i \cdot 256 \ceuler S^3 A^2 H^4 \log^3 K_i(\delta,\delsamp)
\end{align*}
as long as
\begin{align*}
K_i(\delta,\delsamp) \ge 2^i \cdot 256 \ceuler S^3 A^2 H^4 (i+9)^3 \log^3 \left ( 512 \ceuler S A H (i+9) \right )
\end{align*}
So 
\begin{align*}
K_i(\delta,\delsamp) & \ge 2 \cdot 2^i \max \{  128 \ceuler S^3 A^2 H^4 \log^3 K_i(\delta,\delsamp),  128 \ceuler S^3 A^2 H^4 \log^3 \frac{2SAH }{\delta} +  8 H \log \frac{4}{\delta} \} \\
& \ge 2^i ( 32 \ceuler S^3 A^2 H^4 \log^3 \frac{2SAHK_i(\delta,\delsamp)}{\delta} +  8 H \log \frac{4}{\delta})
\end{align*}
if 
\begin{align*}
K_i(\delta,\delsamp) \ge \max \{ 2^i \cdot 256 \ceuler S^3 A^2 H^4 (i+9)^3 \log^3 \left ( 512 \ceuler S A H (i+9) \right ), 128 \ceuler S^3 A^2 H^4 \log^3 \frac{2SAH}{\delta} +  8 H \log \frac{4}{\delta} \}.
\end{align*}
Similarly, 
\begin{align*}
K_i(\delta,\delsamp) \ge 2^i \cdot 144 \ceuler^2 S^2 A^2 H \log \frac{2SAHK_{i}}{\delta}
\end{align*}
if
\begin{align*}
K_i(\delta,\delsamp) \ge \max \{ 2^{i} \cdot 288 \ceuler^{2} S^{2} A^{2} H (i + 3) \log ( 576 \ceuler S A H(i+3)), 288 \ceuler^2 S^2 A^2 H \log \frac{2SAH }{\delta} \}.
\end{align*}
The result then follows recalling the definition of $K_i(\delta,\delsamp)$ given in \eqref{eq:ki_defn}.
\end{proof}

\begin{lem}\label{lem:enough_samples_collected}
Consider a set of policies $\{ \pi_k \}_{k=1}^K$. Assume that running each of these policies once, we collect at least $N$ samples from some $(s,a,h)$. Then, if we rerun each of these policies once, we will collect, with probability $1-\delta$, at least
\begin{align*}
N - \sqrt{8 K \max_k w_h^{\pi_k}(s,a) \log 4/\delta} - 4/3 \log 4/\delta
\end{align*}
samples from $(s,a,h)$. 
\end{lem}
\begin{proof}
Note that when running $\pi_k$, the expected number of visits to $(s,a,h)$ is $w_h^{\pi_k}(s,a)$. By Bernstein's inequality, and using that $\I \{ (s_h^k,a_h^k) = (s,a) \} \sim \text{Bernoulli}(w_h^{\pi_k}(s,a))$, we then have that, with probability at least $1-\delta$,
\begin{align*}
\left | \sum_{k=1}^K w_h^{\pi_k}(s,a) - \sum_{k=1}^K \I \{ (s_h^k,a_h^k) = (s,a) \} \right | \le \sqrt{2 K \max_k w_h^{\pi_k}(s,a) \log 2/\delta} + 2/3 \log 2/\delta
\end{align*}
As our first draw from the policies yielded a value of at least $N$, we can apply \Cref{prop:empirical_confidence}, which gives that, with probability at least $1-2\delta$,
\begin{align*}
\sum_{k=1}^K \I \{ (s_h^k,a_h^k) = (s,a) \} \ge N - 2\sqrt{2 K \max_k w_h^{\pi_k}(s,a) \log 2/\delta} - 4/3 \log 2/\delta
\end{align*}
The result follows.
\end{proof}

\begin{lem}[Lemma 3.4 of \cite{jin2020reward}]\label{lem:euler}
If $r_h^k$ is non-zero for at most one $h$ per episode, the regret of \euler \citep{zanette2019tighter} will be bounded, with probability at least $1-\delta$, as
\begin{align*}
\sum_{k=1}^K \Vst_0 - \sum_{k=1}^K V_0^{\pi_k} \le \ceuler \sqrt{SAH \Vst_0 K \log \frac{SAHK}{\delta} } + \ceuler S^2 A H^4 \log^3 \frac{SAHK}{\delta} 
\end{align*}
for some absolute constant $\ceuler$. 
\end{lem}
\begin{proof}
The proof of this is identical to the proof of Lemma 3.4 in \cite{jin2020reward} but we include it for completeness. We therefore repeat their analysis, using an alternative upper bound for equation (156) in \cite{zanette2019tighter}:
\begin{align*}
\frac{1}{KH} \sum_{k=1}^K \Exp_{\pi_k} \left [ ( \sum_{h=1}^H r_h^k - V_0^{\pi_k})^2 \right ] &  \le \frac{2}{KH} \sum_{k=1}^K \Exp_{\pi_k} \left [ ( \sum_{h=1}^H r_h^k)^2 + (V_0^{\pi_k})^2 \right ] \\
& \overset{(a)}{\le} \frac{2}{KH} \sum_{k=1}^K \Exp_{\pi_k} \left [  \sum_{h=1}^H (r_h^k)^2 + V_0^{\pi_k} \right ] \\
& \overset{(b)}{\le} \frac{2}{KH} \sum_{k=1}^K \Exp_{\pi_k} \left [  \sum_{h=1}^H r_h^k + V_0^{\pi_k} \right ] \\
& = \frac{4}{KH} \sum_{k=1}^K V_0^{\pi_k} \\
& \le 4 \Vst_0 / H
\end{align*}
where $(a)$ follows since $r_h^k$ is nonzero for at most one $h$ and $(b)$ follows since $r_h^k \le 1$. Thus, we can replace $\cG^2$ in Theorem 1 of \cite{zanette2019tighter} with $4 \Vst_0$. As \cite{zanette2019tighter} assume a stationary MDP while ours is non-stationary, we must replace $S$ in their bound with $SH$. This gives the result.

\end{proof}

\begin{lem}\label{lem:rewards_concentrate}
Consider some set of policies $\{ \pi_k \}_{k=1}^K$ where $\pi_k$ is $\cF_{k-1}$ measurable. Let $\sum_{h=1}^H R_h^k$ denote the (random) reward obtained running $\pi_k$ on the MDP $\cM_k$, and let $V_0^{k}$ denote the value function of running $\pi_k$ on $\cM_k$. Assume that
\begin{align*}
\Exp_{\pi_k}[(\sum_{h=1}^H R_h^k - V_0^k)^2 | \cF_{k-1} ] \le \sigma_V^2
\end{align*}
for all $k$ and constant $\sigma_V^2$ which is $\cF_0$-measurable. Then, with probability at least $1-\delta$,
\begin{align*}
\left | \sum_{k=1}^K \sum_{h=1}^H R_h^k - \sum_{k=1}^K V_0^k \right | \le \sqrt{ 8  K \sigma_V^2 \log \frac{2}{\delta} } + 2 H \log \frac{2}{\delta}.
\end{align*}
\end{lem}
\begin{proof}
By definition, $V_0^k = \Exp[\sum_{h=1}^H R_h^k | \cF_{k-1} ]$ and $|\sum_{h=1}^H R_h^k - V_0^k| \le H$ almost surely. The result then follows directly from Freedman's Inequality \citep{freedman1975tail}.

\end{proof}

\begin{prop}\label{prop:empirical_confidence}
Consider some distribution $\P$ and assume that $\Pr_{x \sim \P}[ x \in [\mu - c, \mu+c]] \ge 1-\delta$. Then $\Pr_{x,x' \iidsim \P}[ x \ge x' - 2c] \ge 1- 2\delta$. 
\end{prop}
\begin{proof}
\begin{align*}
\Pr_{x,x' \iidsim \P}[ x \ge x' -2 c] & = \Pr_{x,x' \iidsim \P}[ x' - \mu + \mu - x \le 2c] \\
&\ge \Pr_{x,x' \iidsim \P}[ |x' - \mu| + |\mu - x| \le 2c] \\
& \ge  \Pr_{x \iidsim \P}[  |\mu - x| \le c] \Pr_{x' \iidsim \P}[ |x' - \mu| \le c] \\
& \ge (1-\delta)^2 \\
& \ge 1 - 2 \delta.
\end{align*}
\end{proof}


\section{Proof that Low-Regret is Suboptimal for PAC}\label{sec:ex}

\subsection{Proof of \Cref{prop:two_state_ex}}

\begin{inst}\label{ex:two_state} Given gap parameters $\Delta_1, \Delta_2 > 0$ and transition probability $p \in (0,1/2)$, consider an MDP with $H = S = A = 2$ which always starts in state $s_0$ and has rewards and transitions defined as (where we drop the horizon subscript for simplicity):
\begin{align*}
& P(s_1| s_0,a_1) = 1-p, \quad P(s_2|s_0,a_1) = p, \quad P(s_1|s_0,a_2) = 0, \quad P(s_2|s_0,a_2) = 1 \\
& R(s_0,a_1) \sim \text{Bernoulli}(1), \quad R(s_0,a_2) \sim \text{Bernoulli}(0) \\
& R(s_i,a_1) \sim \text{Bernoulli}(0.5+\Delta_i), \quad R(s_i,a_2) \sim \text{Bernoulli}(0.5), i \in \{1,2\}
\end{align*}
At $h = 2$, we can then think of each state as simply a two-armed bandit with gap $\Delta_i$. We assume that $p < 1/2$, so that $1-p$ can be thought of as a constant. This instance is illustrated in \Cref{fig:mdp_ex1}.

\end{inst}

\begin{prop}[Formal Statement of \Cref{prop:two_state_ex}]\label{prop:two_state_ex2} Given any MDP in \Cref{ex:two_state}, any learner executing \Cref{prot:lr_to_pac} which computes an optimal policy with probability at least $1-\delta$ must collect at least
\begin{align*}
K \ge \Omega \left ( \frac{\log 1/\delta}{ \Delta_1^2} +  \frac{\log 1/\delta}{p \Delta_2^2}  \right )
\end{align*}
episodes, as long as $\frac{\log 1/\delta}{\Delta_2^2} \ge c \max \{ C_2, C_1^{\frac{1}{1-\alpha}} p^{\frac{-\alpha}{1-\alpha}} \}$, for a universal constant $c$. However, on this instance,
\begin{align*}
\Compbsolve(\cM) \le \cO \left (  \frac{1}{\Delta_1^2} + \frac{1}{ \Delta_2^2} \right )
\end{align*}
and so, with probability $1-\delta$, \algname will terminate in at most $K \le \cOtil(\Compbsolve(\cM) \cdot \log 1/\delta)$ episodes and return the optimal policy.
\end{prop}

\begin{proof}[Proof of \Cref{prop:two_state_ex}]
To get the complexity bound of \algname, we apply \Cref{thm:complexity} and \Cref{prop:relate_complexities}. The stated complexity follows since $\Pst_2(s_1) = 1-p \ge 1/2$ and $\Pst_2(s_2) = 1$, from which the stated complexity follows directly.

\paragraph{Complexity of Low-Regret Algorithms.} 
The expected regret of any algorithm is given by
\begin{align*}
N_1(s_0,a_2) + \Delta_1 N_2(s_1,a_2) + \Delta_2 N_2(s_2,a_2)
\end{align*}
where we let $N_h(s_i,a_j)$ denote the expected number of times action $a_j$ is taken in state $s_i$ at timestep $h$. Our assumption on the regret implies that $N_1(s_1,a_2) \le C_1 K^\alpha + C_2$.

From standard lower bounds on bandits (Theorem 4 of \cite{kaufmann2016complexity}), and using that for small $\Delta$ KL$(\text{Bernoulli}(0.5) || \text{Bernoulli}(0.5+\Delta)) = \Theta(\Delta^2)$, to solve the bandit in $s_1$ with probability at least $1-\delta$, we must have that $N_2(s_1) \ge c \frac{\log 1/\delta}{\Delta_1^2}$, and similarly, to solve the bandit in $s_2$, we must have that $N_2(s_2) \ge c \frac{\log 1/\delta}{\Delta_2^2}$, for an absolute constant $c$. 

Note that $N_2(s_1) = (1-p) N_1(s_0,a_1)$ and $N_2(s_2) = p N_1(s_0,a_1) + N_1(s_0,a_2)$, and that the total number of episodes run is $N_1(s_0,a_1) + N_1(s_0,a_2)$. This implies that we must have
\begin{align*}
N_1(s_0,a_1) \ge \frac{c \log 1/\delta}{(1-p) \Delta_1^2}, \quad p N_1(s_0,a_1) + N_1(s_0,a_2) \ge \frac{c \log 1/\delta}{\Delta_2^2}.
\end{align*}
However, since $N_1(s_0,a_2) \le C_1 K^\alpha + C_2$, $N_1(s_0,a_1)$ must at least satisfy
\begin{align*}
p N_1(s_0,a_1) + C_1K^\alpha + C_2 \ge \frac{\log 1/\delta}{\Delta_2^2} \implies N_1(s_0,a_1) \ge \frac{1}{p} \left ( \frac{c \log 1/\delta}{\Delta_2^2} - C_1K^\alpha - C_2 \right ) .
\end{align*}
Thus, we need
\begin{align*}
K = N_1(s_0,a_1) + N_1(s_0,a_2) \ge N_1(s_0,a_1) \ge \max \left \{  \frac{c \log 1/\delta}{(1-p) \Delta_1^2}, \frac{1}{p} \left ( \frac{c \log 1/\delta}{\Delta_2^2} - C_1 K^\alpha - C_2 \right )  \right \}.
\end{align*}
Assume that $\frac{c \log 1/\delta}{\Delta_2^2}  \ge 2 C_2$, then
\begin{align*}
K \ge \frac{1}{p} \left ( \frac{c \log 1/\delta}{\Delta_2^2} - C_1 K^\alpha - C_2 \right )
\end{align*}
implies
\begin{align*}
K \ge \frac{1}{p} \left ( \frac{c \log 1/\delta}{2 \Delta_2^2} - C_1 K^\alpha  \right ) \implies 2 \max \{ p K, C_1 K^\alpha \} \ge \frac{c \log 1/\delta}{2 \Delta_2^2}.
\end{align*}
The second expression is equivalent to
\begin{align*}
K \ge \min \left \{ \frac{c \log 1/\delta}{4p \Delta_2^2}, (\frac{c \log 1/\delta}{4 C_1 \Delta_2^2})^{1/\alpha} \right \}
\end{align*}
and we will have that the minimizer of this is $\frac{c \log 1/\delta}{4p \Delta_2^2}$ as long as $\frac{\log 1/\delta}{\Delta_2^2} \ge c' C_1^{\frac{1}{1-\alpha}} p^{\frac{-\alpha}{1-\alpha}} $. The result follows.
\end{proof}

\subsection{Proof for \Cref{ex:epsilon}}

\begin{inst}[Formal Definition of \Cref{ex:epsilon}]\label{ex:epsilon2} Given a number of states $S \in \N$, consider MDP with horizon $H= 2$, $S$ states, and $S+1$ actions. We assume we always start in state $s_0$ and define our transition kernel and reward function as follows:
\begin{align*}
& P(s_i|s_0,\ast) = \frac{2^{-i}}{1 - 2^{-S}}, \quad  P(s_i|s_0,a_i) = 1, i \in [S] 
\end{align*}
\vspace{-2em}
\begin{align*}
& R(s_0,\ast) \sim \text{Bernoulli}(1) , \quad R(s_0,a_i) \sim \text{Bernoulli}(0), i \in [S] \\
& \forall i: \quad R(s_i,\ast) \sim \text{Bernoulli}(0.9), \quad  R(s_i,a_j) \sim \text{Bernoulli}(0.1), j \in [S].
\end{align*}
Note that $\ast$ is the optimal action in every state.
\end{inst}

\begin{prop}[Formal Statement of \Cref{prop:epsilon_ex}]\label{prop:epsilon_ex2} For the MDP in \Cref{ex:epsilon2} with $S$ states, and any
\begin{align*} \epsilon \in [2^{-S}, c \min \{ C_1^{-1/\alpha} (S \log 1/\delta)^{\frac{1-\alpha}{\alpha}}, C_2^{-1} S \log 1/\delta \} ]
\end{align*}
where $c$ is an absolute constant,  to find an $\epsilon$-optimal policy with probability $1-\delta$ any learner executing \Cref{prot:lr_to_pac} with a low-regret algorithm satisfying \Cref{def:low_regret} must collect at least
$$\Omega \left (\frac{S \log 1/\delta}{\epsilon} \right )$$ 
episodes. In contrast, on this example $\Compbsolve(\cM) = \cO(S^2)$ and $\epssolved = 1/3$, so,  for $ \epsilon \in [2^{-S},1/3]$, with probability $1-\delta$, \algname will terminate and output $\pist$ in $\cOtil(\Clow(1/3))$ episodes.
\end{prop}

\paragraph{Randomized to deterministic policies.}
Assume we are given some randomized policy $\pi$ which for every $(s,h)$ choose action $a$ with probability $\pi_h(a|s)$. Then we define the deterministic policy $\pitil$ given this randomized policy as
\begin{align*}
\pitil_h(s) = \argmax_a \pi_h(a|s).
\end{align*}
We will use this mapping in our lower bound.

\newcommand{\ieps}{i_\epsilon}

\begin{proof}[Proof of \Cref{prop:epsilon_ex2}]
The complexity for \Cref{alg:mcae3_meta} follows directly from \Cref{thm:complexity} and \Cref{prop:relate_complexities} and since in this example we will have $\Pst_2(s) = 1$ for each $s$ and so $\epssolved = 1/3$. Furthermore, $\Compbsolve(\cM) = \cO(S^2)$. The stated complexity follows.

\paragraph{Complexity of Low-Regret Algorithms.}
Let $\Delkl := \text{KL}(\text{Bernoulli}(0.1) || \text{Bernoulli}(0.9)) \approx 1.76$ denote the KL divergence between the reward distributions of the optimal and suboptimal actions at any state for $h=2$, and $\Delta := 0.9 - 0.1$ the suboptimality gap.

Assume that a policy $\pi$ takes action $\ast$ in $s_0$. Then, the total suboptimality of the policy is given by 
\begin{align*}
\sum_{i=1}^S \frac{2^{-i}}{1-2^{-S}} \epsilon_2(s_i,\pi)
\end{align*}
where $\epsilon_2(s_i,\pi)$ denotes the suboptimality of policy $\pi$ in $s_i, i \in [S]$. In particular, for any $i_\epsilon$, to guarantee our policy is $\epsilon$-good we need
\begin{align*}
\frac{2^{-\ieps}}{1-2^{-S}} \epsilon_2(s_{\ieps},\pi) \le \epsilon.
\end{align*}
By the structure of the reward in any state $s_{\ieps}$, the total suboptimality in this state will be
\begin{align*}
\epsilon_2(s_{\ieps},\pi) = (1 - \sum_{j=1}^S \pi_2(a_j| s_{\ieps})) \Delta 
\end{align*}
It follows that if $\epsilon_2(s_{\ieps},\pi) < \Delta/4$, then we will have that $\pitil_2(s_{\ieps}) = \ast$, where $\pitil$ is the deterministic policy derived from $\pi$. Choose $\ieps = \lfloor -\log_2(2 \epsilon(1-2^{-S})/\Delta) - 1 \rfloor$. Then it follows that,
\begin{align*}
\frac{2^{-\ieps}}{1-2^{-S}} (1 - \sum_{j=1}^S \pi_2(a_j| s_{\ieps})) \Delta \ge 4 \epsilon (1 - \sum_{j=1}^S \pi_2(a_j| s_{\ieps}))
\end{align*}
and thus, for the policy to be $\epsilon$-optimal, we must have that $(1 - \sum_{j=1}^S \pi_2(a_j| s_{\ieps})) \le 1/4$. This implies that $\pitil_2(s_{\ieps}) = \ast$, so we have therefore derived a deterministic policy from our stochastic one that is optimal in $(s_{\ieps},2)$. 
By Theorem 4 of \cite{kaufmann2016complexity}, to identify the optimal action in state $s_{\ieps}$ with probability $1-\delta$ we must have that
\begin{align*}
N_2(s_{\ieps}) \ge \frac{(S+1)}{\Delkl} \log \frac{1}{2.4 \delta}
\end{align*}
where $N_2(s_{\ieps})$ is the expected number of samples collected in $s_{\ieps}$ at $h = 2$. As we have deterministically derived $\pitil$ from $\pi$, and since $\pitil$ will play the optimal action in $s_{\ieps}$ for any $\epsilon$-optimal $\pi$, it follows that this lower bound on $N_2(s_{\ieps}) $ applies here.

If our low-regret algorithm has regret bounded as $C_1 K^\alpha + C_2$, then we must have that
\begin{align*}
\sum_{i=1}^S N_1(s_1,a_i) \le C_1K^\alpha + C_2
\end{align*}
since every time action $a_i \neq \ast$ is taken we will incur a loss of 1. This implies that 
\begin{align*}
N_2(s_{\ieps}) \le C_1K^\alpha + C_2 + \frac{2^{-\ieps}}{1-2^{-S}} K
\end{align*}
since if action $\ast$ is taken in state $s_1$, we will only reach state $s_{\ieps}$ with probability $\frac{2^{-\ieps}}{1-2^{-S}}$. Combining these, to ensure that the optimal action is learned in $s_{\ieps}$, we will need that
\begin{align*}
 \frac{(S+1)}{\Delkl} \log \frac{1}{2.4 \delta} \le C_1 K^\alpha + C_2 + \frac{2^{-\ieps}}{1-2^{-S}} K \le C_1 K^\alpha + C_2 + \frac{4  \epsilon}{\Delta} K
\end{align*}
where the second inequality follows by our choice of $\ieps$. It follows that we need
\begin{align*}
K  \ge \frac{\Delta}{4\epsilon} & \left (  \frac{(S+1)}{\Delkl} \log \frac{1}{2.4 \delta} - C_1 K^\alpha - C_2 \right ) \ge  \frac{(S + 1) \log 1/2.4\delta}{12 \epsilon} - C_1 K^\alpha - C_2 \\
& \quad \ge \frac{(S + 1) \log 1/2.4\delta}{24 \epsilon} - C_1 K^\alpha
\end{align*}
where the final inequality holds as long as $\frac{(S + 1) \log 1/2.4\delta}{12 \epsilon} \ge 2 C_2$. This implies
\begin{align*}
2 \max \{K, C_1 K^\alpha \} \ge   \frac{(S + 1) \log 1/2.4\delta}{24 \epsilon} 
\end{align*}
which is equivalent to 
\begin{align*}
K \ge \min \left \{ \frac{(S + 1) \log 1/2.4\delta}{48 \epsilon},  \left (  \frac{(S + 1) \log 1/2.4\delta}{48 C_1 \epsilon} \right )^{1/\alpha} \right \}.
\end{align*}
For 
\begin{align*}
\epsilon \le \cO \left ( C_1^{-1/\alpha} (S \log 1/\delta)^{\frac{1-\alpha}{\alpha}} \right )
\end{align*}
we will have that the minimizer is the first term, and
\begin{align*}
K \ge \Omega \left ( \frac{S \log 1/\delta}{\epsilon} \right ).
\end{align*}
\end{proof}

\section{Lower Bounds on Best Policy Identification}

\begin{lem}\label{lem:mdp_change_measure}
Consider MDPs $\cM$ and $\cM'$ with the same state space $\cS$, actions space $\cA$, horizon $H$, and initial state distribution $P_0$. Fix some $(s,h) \in \cS \times [H]$, and for any $a \in \cA$ let $\nu_h(s,a)$ denote the law of the joint distribution of $(s',R)$ where $s' \sim P_\cM(\cdot|s,a)$ and $R \sim R_\cM(s,a)$. Define the law $\nu_h'(s,a)$ analogously with respect to $\cM'$. For any almost-sure stopping time $\tau$ with respect to $(\cF_k)$, 
\begin{align*}
\sum_{s,a,h} \Exp_{\cM}[N_{h}^\tau(s,a)] \KL(\nu_h(s,a), \nu_h'(s,a)) \ge \sup_{\cE \in \cF_\tau} d(\Pr_{\cM}(\cE),\Pr_{\cM'}(\cE))
\end{align*}
where $d(x,y) = x \log \frac{x}{y} + (1 - x) \log \frac{1-x}{1-y}$ and $N_h^\tau(s,a)$ denotes the number of visits to $(s,a,h)$ in the $\tau$ episodes. 
\end{lem}
\begin{proof}
This is the MDP analogue of Lemma 1 of \cite{kaufmann2016complexity} and its proof follows identically. 
\end{proof}

\begin{defn}
We say an algorithm is $\delta$-correct if, for any MDP $\cM \in \mathfrak{M}$, we have that $\cM$ terminates at some (possibly random) episode $K_\delta$ and outputs $\pist$, with probability at least $1-\delta$. 
\end{defn}

\subsection{Proof of \Cref{prop:bpi_lb1}}
\paragraph{MDP Construction.}
Fix some $\hbar \in [H]$, gaps $\{ \gap(s,a) \}_{s \in [S], a \in [A-1]} \subseteq (0,1/2)^{SA}$, and arbitrary transition kernels $\{ P_h \}_{h=1}^{\hbar - 1}$. For each $s$, fix a single $a$ and set  $\gap(s,a) = 0$. Let $\cM$ denote the MDP with transitions $\{ P_h \}_{h=1}^{\hbar - 1}$, and for $h \ge \hbar$ define
\begin{align*}
P_{h}(s|s,a) = 1, \quad \forall a \in \cA.
\end{align*}
Then let the rewards be defined as follows. For all $h > \hbar$ and all $s$, choose any $a'$ and set $R_h(s,a') = 1$, and $R_h(s,a) = 0$ for all $a \neq a'$. For $h = \hbar$, set
\begin{align*}
R_h(s,a) \sim \bern(3/4 - \gap(s,a)) .
\end{align*}
For $h < \hbar$, let
\begin{align*}
\pist_h(s) = \argmax_a \sum_{s'} P_h(s'|s,a) \Vst_{h+1}(s')
\end{align*}
where $\Vst_{h+1}(s')$ is the optimal value function at step $h+1$ (note that the MDP is now fully specified for $h' > h$ so this is well-defined). Then set $R_h(s,\pist_h(s)) = 1$ and $R_h(s,a) = 0$ for $a \neq \pist_h(s)$ (if $\pist_h(s)$ is not unique, simply choose some $\pist$ out of all $\pist_h(s)$ arbitrarily, set $R_h(s,\pist) = 1$, and all other $R_h(s,a) = 0$).

Note that we could have just as easily encoded the gaps in the transition function and set the rewards to be, for example, deterministic at level $\hbar$.

\begin{lem}\label{lem:bpi_lb_gaps}
The MDP constructed above has gaps which satisfy
\begin{align*}
& \Delta_{\hbar}(s,a) = \gap(s,a), \quad \forall s \in \cS, a \in \cA, a \neq \pist_h(s) \\
& \Delta_{h}(s,a) \ge 1, \quad \forall s \in \cS, a \in \cA, h \neq \hbar
\end{align*}
Furthermore, for each $s$ and $h > \hbar$, we have $\Pst_h(s) = \Pst_{\hbar}(s)$. 
\end{lem}
\begin{proof}
We begin with level $\hbar$. Since the action take at $(s,\hbar)$ does not effect the outgoing transition, we have that, for $a \neq \pist_h(s)$,
\begin{align*}
\Delta_{\hbar}(s,a) = \max_{a'} \Qst_{\hbar}(s,a') - \Qst_{\hbar}(s,a) = 3/4 - (3/4 - \gap(s,a)) = \gap(s,a).
\end{align*}
For $h > \hbar$, we again have that the outgoing transition is not effected by the action taken, so it follows that the gap depends exclusively on the reward function at this state. Since the reward is set to 1 for a single action and 0 otherwise, it follows that the gaps are all 1.

For $h \le \hbar$, we will have that
\begin{align*}
\Delta_h(s,a) & = \max_{a'} \Qst_h(s,a') - \Qst_h(s,a) \\
& = 1 + \max_{a'} \sum_{s'} P_h(s'|s,a') \Vst_{h+1}(s') - \sum_{s'} P_h(s'|s,a) \Vst_{h+1}(s') \\
& \ge 1.
\end{align*}

Finally, that $\Pst_h(s) = \Pst_{\hbar}(s)$ for all $s$ and $h > \hbar$ follows since for all steps after $\hbar$, state $s$ transitions to state $s$ with probability 1.
\end{proof}

\begin{lem}\label{lem:bpi_cst_bound}
On this example, 
\begin{align*}
\Compbsolve(\cM) \le \inf_\pi \max_{s,a} \frac{1}{\wpi_{\hbar}(s,a) \Delta_h(s,a)^2} +  \max_{s,h} \frac{SAH}{\Pst_h(s)} .
\end{align*}
\end{lem}
\begin{proof}
By definition,
\begin{align*}
\Compbsolve(\cM) = \sum_{h=1}^H \inf_\pi \max_{s,a} \frac{1}{\wpi_h(s,a) \Delta_h(s,a)^2} .
\end{align*}
By \Cref{lem:bpi_lb_gaps}, we can bound
\begin{align*}
\sum_{h \neq \hbar} \inf_\pi \max_{s,a} \frac{1}{\wpi_h(s,a) \Delta_h(s,a)^2}  \le \sum_{h \neq \hbar} \inf_\pi \max_{s,a} \frac{1}{\wpi_h(s,a)} .
\end{align*}
Consider the policy $\pi'$ which is the mixture over the policies $\pi^{sh}$ where $w_h^{\pi^{sh}}(s) = \Pst_h(s)$. Then,
\begin{align*}
\sum_{h \neq \hbar} \inf_\pi \max_{s,a} \frac{1}{\wpi_h(s,a)} \le \sum_{h \neq \hbar}  \max_{s,a} \frac{1}{w^{\pi'}_h(s,a)} \le \sum_{h \neq \hbar} \max_{s} \frac{SA}{\Pst_h(s)} \le \max_s \frac{SAH}{\Pst_h(s)} .
\end{align*}
\end{proof}

\begin{lem}
On the MDP constructed above, any $\delta$-correct algorithm will have
\begin{align*}
\Exp_{\cM}[K_\delta] & \ge \inf_\pi \max_{s,a} \frac{1}{6 \wpi_{\hbar}(s,a) \Delta_{\hbar}(s,a)^2} \cdot \log \frac{1}{2.4 \delta} \\
& \gtrsim \Compbsolve(\cM) \cdot \log \frac{1}{2.4 \delta} - \max_{s,h} \frac{SAH}{\Pst_h(s)} .
\end{align*}
\end{lem}
\begin{proof}
We will apply \Cref{lem:mdp_change_measure} on our MDP, $\cM$, and MDP $\cM'$ which is identical to $\cM$ except that, for some $(s,a)$, $a \neq \pist_{\hbar}(s)$, we set $R_{\hbar}(s,a) \sim \bern(3/4 + \alpha)$ for small $\alpha$. Note that in this case we have that the optimal policy on $\cM$ and $\cM'$ differ at $(s,\hbar)$. Since $\cM$ and $\cM'$ are identical at all points but this one, we have 
\begin{align*}
& \sum_{s,a,h} \Exp_{\cM}[N_{h}^\tau(s,a)] \KL(\nu_h(s,a), \nu_h'(s,a)) \\
& \qquad = \Exp_{\cM}[N_{\hbar}^\tau(s,a)] \KL \big (\bern(3/4 - \gap(s,a)),\bern(3/4+\alpha) \big ) .
\end{align*}
Let $\pist(\cM)$ denote the optimal policy on $\cM$, and $\pihat$ denote the policy returned by our algorithm. Let $\cE = \{ \pihat = \pist(\cM) \}$. Since we assume our algorithm is $\delta$-correct, and since the optimal policies on $\cM$ and $\cM'$ differ, we have $\Pr_{\cM}(\cE) \ge 1 - \delta$ and $\Pr_{\cM'}(\cE) \le \delta$. By \cite{kaufmann2016complexity}, we can then lower bound
\begin{align*}
d(\Pr_{\cM}(\cE),\Pr_{\cM'}(\cE)) \ge \log \frac{1}{2.4 \delta}. 
\end{align*}
Thus, by \Cref{lem:mdp_change_measure}, we have shown that, for any $(s,a)$, $a \neq \pist_{\hbar}(s)$, 
\begin{align*}
\Exp_{\cM}[N_{\hbar}^\tau(s,a)] \ge \frac{1}{\KL \big (\bern(3/4 - \gap(s,a)),\bern(3/4+\alpha) \big )} \cdot \log \frac{1}{2.4 \delta} .
\end{align*}
For small $\alpha$, we can bound (see e.g. Lemma 2.7 of \cite{tsybakov2009introduction})
\begin{align*}
\KL \big (\bern(3/4 - \gap(s,a)),\bern(3/4+\alpha) \big ) \le 6 ( \gap(s,a) - \alpha)^2 .
\end{align*}
Taking $\alpha \rightarrow 0$, we have
\begin{align*}
\Exp_{\cM}[N_{\hbar}^\tau(s,a)] \ge \frac{1}{6 \gap(s,a)^2} \cdot \log \frac{1}{2.4 \delta} .
\end{align*}
We can write $\Exp_{\cM}[N_{\hbar}^\tau(s,a)] = \Exp_{\cM}[\sum_{k = 1}^\tau w_{\hbar}^{\pi_k}(s,a)]$ where $\pi_k$ denotes the policy our algorithm played at episode $k$. 
Note that all state-visitation distributions lie in a convex set in $[0,1]^{SA}$ and that for any valid state-visitation distribution, there exists some policy that realizes it, by \Cref{prop:state_act_vis}. By Caratheodory's Theorem, it follows that there exists some set of policies $\Pi$ with $|\Pi| \le SA + 1$ such that, for any $\pi$ and all $s,a$, $\wpi_{\hbar}(s,a) = \sum_{\pi' \in \Pi} \lambda_{\pi'} w^{\pi'}_{\hbar}(s,a)$, for some $\lambda \in \simplex_\Pi$. Letting $\lambda^k$ denote this distribution satisfying the above inequality for $\pi_k$, it follows that
\begin{align*}
\Exp_{\cM}[\sum_{k = 1}^\tau w_{\hbar}^{\pi_k}(s,a)] & = \Exp_{\cM}[\sum_{k=1}^\tau \sum_{\pi \in \Pi} \lambda_\pi^k \wpi_{\hbar}(s,a)] \\
& = \sum_{\pi \in \Pi}\Exp_{\cM}[\sum_{k=1}^\tau \lambda_\pi^k ] \wpi_{\hbar}(s,a) \\
&  = \Exp_{\cM}[\tau] \sum_{\pi \in \Pi} \frac{\Exp_{\cM}[\sum_{k=1}^\tau \lambda_\pi^k ]}{\Exp_{\cM}[\tau]} \wpi_{\hbar}(s,a) .
\end{align*}
Note that $ \sum_{\pi \in \Pi} \Exp_{\cM}[\sum_{k=1}^\tau \lambda_\pi^k ] = \Exp_{\cM} [\sum_{k=1}^\tau \sum_{\pi \in \Pi} \lambda_\pi^k] = \Exp_{\cM}[\tau]$ so it follows that $(\frac{\Exp_{\cM}[\sum_{k=1}^\tau \lambda_\pi^k ]}{\Exp_{\cM}[\tau]})_{\pi \in \Pi} \in \simplex_\Pi$. Thus, a $\delta$-correct algorithm must satisfy, for all $s,a$ and some $\lambda \in \simplex_\Pi$,
\begin{align*}
\Exp_{\cM}[\tau] \ge \frac{1}{6 \gap(s,a)^2 \cdot \sum_{\pi \in \Pi} \lambda_\pi \wpi_{\hbar}(s,a)} \cdot \log \frac{1}{2.4 \delta} .
\end{align*}
Since the set of state visitation distributions is convex, and since for any state-visitation distribution we can find some policy realizing that distribution, for any $\lambda \in \simplex_\Pi$, it follows that there exists some $\pi'$ such that, for all $s,a$, $\sum_{\pi \in \Pi} \lambda_\pi \wpi_{\hbar}(s,a) = w_{\hbar}^{\pi'}(s,a)$. So, we need, for all $s,a$
\begin{align*}
\Exp_{\cM}[\tau] \ge \frac{1}{6 \gap(s,a)^2 \cdot w^{\pi'}_{\hbar}(s,a)} \cdot \log \frac{1}{2.4 \delta} .
\end{align*}
It follows that \emph{every} $\delta$-correct algorithm must satisfy
\begin{align*}
\Exp_{\cM}[\tau] \ge \inf_\pi \max_{s,a} \frac{1}{6 \gap(s,a)^2 \cdot w^{\pi}_{\hbar}(s,a)} \cdot \log \frac{1}{2.4 \delta} ,
\end{align*}
from which the first inequality follows. The second follows from \Cref{lem:bpi_cst_bound}.
\end{proof}

\end{document}